\pdfoutput=1
\documentclass[journal]{IEEEtran}
\IEEEoverridecommandlockouts
\usepackage{amsmath,graphicx,epstopdf,amssymb,amsthm}% ,epstopdf,spconf,amssymb,epsfig,fullpage}%, ,graphicx,psfrag,url,amsmath,amsthm,comment} 
\usepackage{cite}
\usepackage{float}
\usepackage{hyperref}
\usepackage{epstopdf}
\usepackage{amsmath,bm}
\usepackage{verbatim}
\usepackage[utf8]{inputenc}
\usepackage[english]{babel}
\usepackage{caption}
\usepackage{enumitem}
\usepackage[ruled,vlined,linesnumbered]{algorithm2e}
\DeclareCaptionLabelFormat{lc}{\MakeLowercase{#1}~#2}
\usepackage[dvipsnames]{xcolor}
\usepackage{url}
\newenvironment{skproof}{\noindent\textit{Sketch of  Proof:}}{\hfill$\blacksquare$}

\newtheorem{theorem}{Theorem}
\newtheorem{lemma}{Lemma}
\newtheorem{fact}{Fact}
\newtheorem{definition}{Definition}
\newtheorem{proposition}{Proposition}
\newtheorem{corollary}{Corollary}

\newtheorem{remark}{Remark}

\newtheorem{assumption}{Assumption}
\newtheorem{condition}{Condition}

\allowdisplaybreaks
\usepackage[protrusion=true,expansion=true]{microtype}
\pdfoutput=1

\usepackage[font=small]{caption}

\usepackage{mathtools}

\newcommand{\addFL}[1]{\textcolor{blue}{#1}}
\DeclarePairedDelimiter\floor{\lfloor}{\rfloor}

\newcommand{\new}[1]{{\color{blue}#1}}
\newcommand{\nm}[1]{{\color{ForestGreen} {\bf [NM: #1]}}}
\newcommand{\add}[1]{\textcolor{red}{#1}}

\usepackage{cancel}
   
\usepackage{soul,xcolor}
\DeclareMathOperator*{\argmin}{arg\,min}
\usepackage{bbm, dsfont}

\usepackage{setspace}

\def\BibTeX{{\rm B\kern-.05em{\sc i\kern-.025em b}\kern-.08em
    T\kern-.1667em\lower.7ex\hbox{E}\kern-.125emX}}
\begin{document}
\setulcolor{red}
\setul{red}{2pt}
\setstcolor{red}
%\newcommand\semiSmall{\fontsize{23.8}{20.38}\selectfont}
%\title{D2D-assisted Federated Learning: Hybrid Distributed Machine Learning in Two Timescales}

\title{Semi-Decentralized Federated Learning with Cooperative D2D Local Model Aggregations}

% Distributed model training under cooperative D2D communications.. 

% A consensus-driven distributed model training platform via cooperative device-to-device communications

% Some keywords: Device-to-device, peer-to-peer, Non-i.i.d data (to emphasize our new definition...), resource constrained, local descent method (we have multiple local descents),
% D2D, P2P, cluster-based, local cooperation.., locally cooperative devices ....

%  D2D-assisted hybrid federated learning with Aperiodic Consensus

\author{Frank Po-Chen Lin,~\IEEEmembership{Student Member,~IEEE},  Seyyedali~Hosseinalipour,~\IEEEmembership{Member,~IEEE}, Sheikh Shams Azam, \\ Christopher G. Brinton,~\IEEEmembership{Senior~Member,~IEEE}, and Nicol\`o Michelusi, \IEEEmembership{Senior~Member,~IEEE}
\thanks{F. Lin, S. Hosseinalipour, S. Azam, and C. Brinton are with the School of Electrical and Computer Engineering, Purdue University, IN, USA. e-mail: \{lin1183,hosseina,azam1,cgb\}@purdue.edu. Part of Brinton's research has been funded by ONR under grant N00014-21-1-2472 and NSC under grant W15QKN-15-9-1004.}
\thanks{N. Michelusi is with the School of Electrical, Computer and Energy Engineering, Arizona State University, AZ, USA. e-mail: nicolo.michelusi@asu.edu. 
{Michelusi's work was supported in part by the National Science Foundation under grants CNS-1642982 and CNS-2129015.}}
\thanks{An abridged version of this paper has been published in IEEE Globecom 2021~\cite{fedStar}.}
}
\maketitle

% \nm{Once the paper is finalized, you will have to update the response document with the correct quotes and references.}

% \nm{I temporarily ordered the biblio in alphabetic order (to check for repeated entries, this can be done by using IEEEtranS instead of IEEEtran)
% I noticed that [45] and [46] are the same! please fix. Once you have fixed, please change again to IEEEtran.. and also remember to add the reference to our submitted Globecom.}

\begin{abstract}
Federated learning has emerged as a popular technique for distributing machine learning (ML) model training across the wireless edge. In this paper, we propose
\textit{two timescale hybrid federated learning} ({\tt TT-HF}), a semi-decentralized learning architecture that combines the conventional device-to-server communication paradigm for federated learning with device-to-device (D2D) communications for model training. In {\tt TT-HF}, during each global aggregation interval, devices (i) perform multiple stochastic gradient descent iterations on their individual datasets, and (ii) aperiodically engage in consensus procedure
of their model parameters through cooperative, distributed D2D communications within local clusters. With a new general definition of gradient diversity, we formally study the convergence behavior of {\tt TT-HF}, resulting in new convergence bounds for distributed ML. We leverage our convergence bounds to develop an adaptive control algorithm that tunes the step size, D2D communication rounds, and global aggregation period of {\tt TT-HF} over time to target a sublinear convergence rate of $\mathcal{O}(1/t)$ while minimizing network resource utilization. Our subsequent experiments
demonstrate that {\tt TT-HF} significantly outperforms the current art in federated learning in terms of model accuracy and/or network energy consumption in different scenarios where local device datasets exhibit statistical heterogeneity. Finally, our numerical evaluations demonstrate robustness against outages caused by fading channels, as well favorable performance with non-convex loss functions.
\end{abstract}
\begin{IEEEkeywords} Device-to-device (D2D) communications, {peer-to-peer (P2P) learning}, fog learning, cooperative consensus  {formation}, semi-decentralized federated learning. 
\end{IEEEkeywords}

\section{Introduction}
\noindent Machine learning (ML) techniques have exhibited widespread successes in applications ranging from computer vision to natural language processing~\cite{ghosh2019understanding,kang2016object,goldberg2017neural}. Traditionally, ML model training has been conducted in a centralized manner, e.g., at datacenters, where the computational infrastructure and dataset required for training coexist. In many applications, however, the data required for model training is generated at devices which are distributed across the edge of communications networks. As the amount of data on each device increases, uplink {transmission of local datasets} to a main server becomes bandwidth-intensive and time consuming, which is prohibitive in latency-sensitive applications~\cite{Chiang}. Common examples include object detection for autonomous vehicles~\cite{wu2017squeezedet} and keyboard next-word prediction on smartphones~\cite{hard2018federated}, each requiring rapid analysis of data generated from embedded sensors. Also, in many applications, end users may not be willing to share their datasets with a server due to privacy concerns~\cite{azam2020towards}.
 
%keyboard next word prediction in cellular phones~\cite{hard2018federated} and object detection for autonomous vehicles~\cite{wu2017squeezedet}, the data required to train the ML model is generated at the edge devices, e.g., cellular phones and smart cars, using the plethora of their mounted sensing devices and interactions with the users. In these applications, transferring of the raw data to a central location may not be feasible/practical due to the following reasons: (i) the size of the dataset can become prohibitively large upon data acquisition from a large number of users, which is cumbersome to store in a central location; (ii) transferring of the raw data from the end users to a central server can be extremely time and power consuming; (iii) the end devices may not be willing to share their gathered data due to privacy concerns. 

\subsection{Federated Learning at the Wireless Edge}
Federated learning has emerged as a popular distributed ML technique for addressing these bandwidth and privacy challenges~\cite{mcmahan2017communication,konevcny2016federated,yang2019federated}. A schematic of its conventional architecture is given in Fig.~\ref{fig:simpleFL}: in each iteration, each device trains a local model based on its own dataset, often using (stochastic) gradient descent. The devices then upload their local models to the server, which aggregates them into a global model, often using a weighted average, and synchronizes the devices with this new model to initiate the next round of local training. 
 
Although widespread deployment of federated learning is desired~\cite{lim2020federated,bennis2020edge}, its conventional architecture in Fig.~\ref{fig:simpleFL} poses challenges for the wireless edge: the devices comprising the Internet of Things (IoT) may exhibit significant heterogeneity in their computational resources (e.g., a high-powered drone compared to a low-powered smartphone)
% \hl{which may affect the number of samples each can process}\nm{how is this relevant in your paper? I dont think we address this challenge}
\cite{tu2020network}; additionally, the devices may exhibit varying proximity to the server (e.g., varying distances from smartphones to the base station in a cell), which may cause significant energy consumption
for upstream data transmission \cite{yang2019energy}.

{To mitigate the cost of uplink and downlink transmissions, 
local model training coupled with periodic but infrequent global aggregations has been proposed~\cite{wang2019adaptive,frank2020delay,tu2020network}.
}
% \sst{One technique that has been developed to address the variation in communication requirements across wireless devices in federated learning is decreasing the frequency of global aggregations~$cite{wang2019adaptive,tu2020network}$. By increasing the length of the local update period, devices can also perform more gradient iterations to further refine their local models between aggregations, while decreasing the frequency of uplink and downlink transmissions. However,} 
Yet, the local datasets may exhibit significant heterogeneity in their statistical distributions~\cite{hosseinalipour2020federated}, resulting in learned models that may be
% \sst{. In this case, longer local update periods will result in device models that are significantly}
biased towards local datasets, hence degrading the global model accuracy
% \sst{ This can in turn degrade the convergence speed of the global model and the resulting trained model accuracy}
\cite{wang2019adaptive}.
% \sst{ In such settings, methods are needed to mitigate divergence across the local models.}

%Moreover, the local datasets may exhibit significant heterogeneity in their statistical distributions (i.e., varying degrees of non-i.i.d.), which can cause considerable divergence between local models and hinder convergence of the global model \cite{hosseinalipour2020federated}.
%One of the main techniques to implement federated learning in such environments is reducing the number of global aggregations during the period of learning via conducting multiple local gradient descent rounds at the edge devices between two consecutive global aggregations~\cite{wang2019adaptive,tu2020network}. However, upon having extreme non-i.i.d datasets, this method does not perform well since it results in model bias toward the local datasets, which is boosted by increasing the number of local descent updates that can severely affect the performance of the model training~\cite{wang2019adaptive}.

In this setting, motivated by the need to mitigate divergence across the local models, we study the problem of \textit{resource-efficient federated learning across heterogeneous local datasets at the wireless edge}. A key technology that we incorporate into our approach is device-to-device (D2D) communications among edge devices{, which is a localized version of peer-to-peer (P2P) among direct physical connections}. D2D communications is being envisioned in fog computing and IoT systems through 5G wireless~\cite{Chiang,hosseinalipour2020federated,chen2020wireless}; indeed, it is expected that 50\% of all network connections will be machine-to-machine by 2023~\cite{hosseinalipour2020federated}. Through D2D, we
design a consensus mechanism to mitigate model divergence via low-power communications among nearby devices. We call our approach \textit{two timescale hybrid federated learning} ({\tt TT-HF}), since it (i) involves a hybrid between device-to-device and device-to-server communications, and (ii) incorporates two timescales for model training: iterations of stochastic gradient descent at individual devices, and rounds of cooperative D2D communications within clusters.
% \nm{so why not three timescale then? local SGD -- consensus -- global aggregation? And why not semi-dec instead of two TS hybrid, to be consistent with the title?}
By inducing consensus in the local models within a cluster of devices, {\tt TT-HF} promises resource efficiency, as we will show both theoretically and by simulation,
since only one device from the cluster needs to upload the \emph{cluster model} to the server during global aggregation, as opposed to the conventional federated learning architecture, where most of the devices are required to upload their local models~\cite{mcmahan2017communication}.
% \nm{I think this paragraph was not emphasizing sufficiently the consensus part, which is a key component of our architecture. I rephrased it.}
% can design a coordination mechanism to mitigate model divergence via low-power communications among nearby devices.
Specifically, during the local update interval in federated learning, devices can systematically
% \nm{in your scheme tyou dont do it "occasionally", you do it "systematically"}
share their model parameters with others in their neighborhood to form a distributed consensus among each cluster of edge devices. Then, at the end of each local training interval, assuming that each device's model now reflects the consensus of its cluster, the main server can randomly sample just one device from each cluster for the global aggregation. We call our approach \textit{two timescale hybrid federated learning} ({\tt TT-HF}), since it (i) involves a hybrid between device-to-device and device-to-server communications, and (ii) incorporates two timescales for model training: iterations of gradient descent at individual devices, and rounds of cooperative D2D communications within clusters.

\begin{figure}[t]
\includegraphics[width=\linewidth]{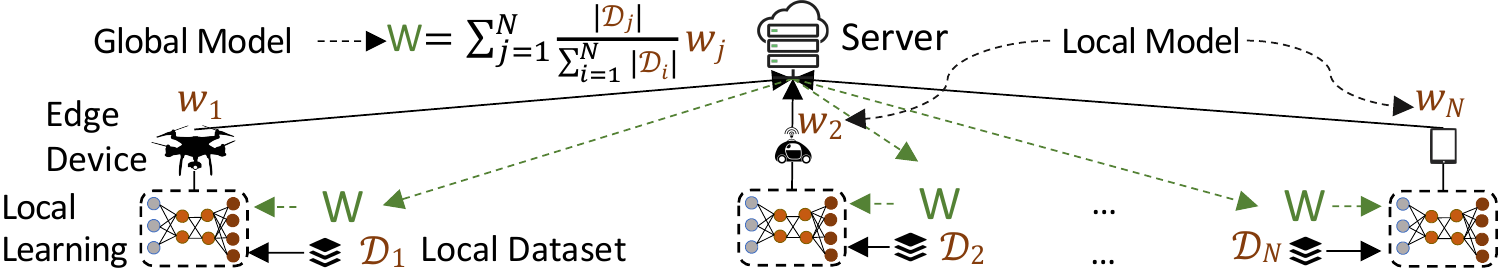}
\centering
\caption{
% \nm{do yiou really need three devices? Why not showing just two and dots .... between them? This figure looks too small!}
Conventional federated learning. In each training round, devices perform local model updates based on local datasets, followed by an aggregation at the main server to compute the global model, which is broadcast to the devices for the next round of local updates.}
\label{fig:simpleFL}
\vspace{-5mm}
\end{figure}
{\tt TT-HF} migrates from the ``star'' topology of conventional federated learning in Fig.~\ref{fig:simpleFL} to a semi-decentralized learning architecture, {shown in Fig.~\ref{fig2},}
that includes local topologies between edge devices, as advocated in the new ``fog learning'' paradigm~\cite{hosseinalipour2020federated}.  %{\color{blue}This forms a semi-decentralized learning architecture for federated learning.}
In doing so, we must carefully consider the relationships between {device-level stochastic gradient updates, cluster-level consensus procedure, and network-level global aggregations.}
% \nm{again, I see three TS here, not two! Just use SD-FL (semi dec fL) insted of TTHF?} 
We quantify these relationships in this work, and use them to tune the lengths of each local update and consensus period. As we will see, the result is a version of federated learning which optimizes the global model convergence characteristics while minimizing the uplink communication requirement in the system.

\subsection{Related Work} \label{sec:related}
A multitude of works on federated learning have emerged in the past few years,
{addressing various aspects, such as} communication and computation constraints of wireless devices~\cite{tran2019federated,chen2019joint,yang2019energy,yang2020federated}, multi-task learning~\cite{smith2017federated,corinzia2019variational,9006060}, and personalized model training~\cite{jiang2019improving,fallah2020personalized}. We refer the reader to e.g.,~\cite{rahman2020survey,li2020federated} for comprehensive surveys of the federated learning literature; in this section, we will focus on the works addressing resource efficiency, statistical data heterogeneity, and cooperative learning, which is the main focus of this paper.

%we refer the reader to \cite{rahman2020survey} for a comprehensive survey. Most existing literature in this area has studied optimizations of federated learning over its classic structure depicted in Fig.~\ref{fig:simpleFL}. In particular, the consideration of computation and communication constraints of devices over wireless networks is an active area of research~\cite{tran2019federated,chen2019joint,yang2019energy,yang2020federated}.
%Also, federated learning has been studied in other contexts, such as multi-task learning and personalized model training~\cite{smith2017federated,corinzia2019variational,9006060,jiang2019improving,fallah2020personalized}, where individual models are tailored for different users. We refer the reader to~\cite{li2020federated} and references therein for other research directions and interesting problems investigated.

In terms of wireless communication efficiency, several works have investigated the impact of performing multiple rounds of local gradient updates in-between consecutive global aggregations~\cite{haddadpour2019convergence,wang2019adaptive}, including optimizing the aggregation period according to a total resource budget~\cite{wang2019adaptive}. To further reduce the demand for global aggregations,~\cite{liu2020client} proposed a hierarchical system model for federated learning where edge servers are utilized for partial global aggregations. 
% Regarding device processing efficiency,~\cite{tu2020network,wang2021device} showed that data offloading in D2D-enabled wireless networks can reduce resource utilization while preserving model accuracy.
% \nm{?? We are claiming that we are the first to do D2D for FL, but then you cite this paper... you need to clarify the difference with your work.. What is lacking in this paper that you address in yours?}
Model quantization~\cite{amiri2020federated} and sparsification~\cite{sattler2019robust} techniques have also been proposed. 
{As compared to above works, we propose a semi-decentralized architecture, where D2D communications are used to exchange model parameters among the nodes in conjunction with global aggregations. We show that our framework can reduce the frequency of global aggregations and result in network resource savings.}

\begin{figure}[t]
  \centering
%   \vspace{-0.4in}
    \includegraphics[height=25mm,width=88mm]{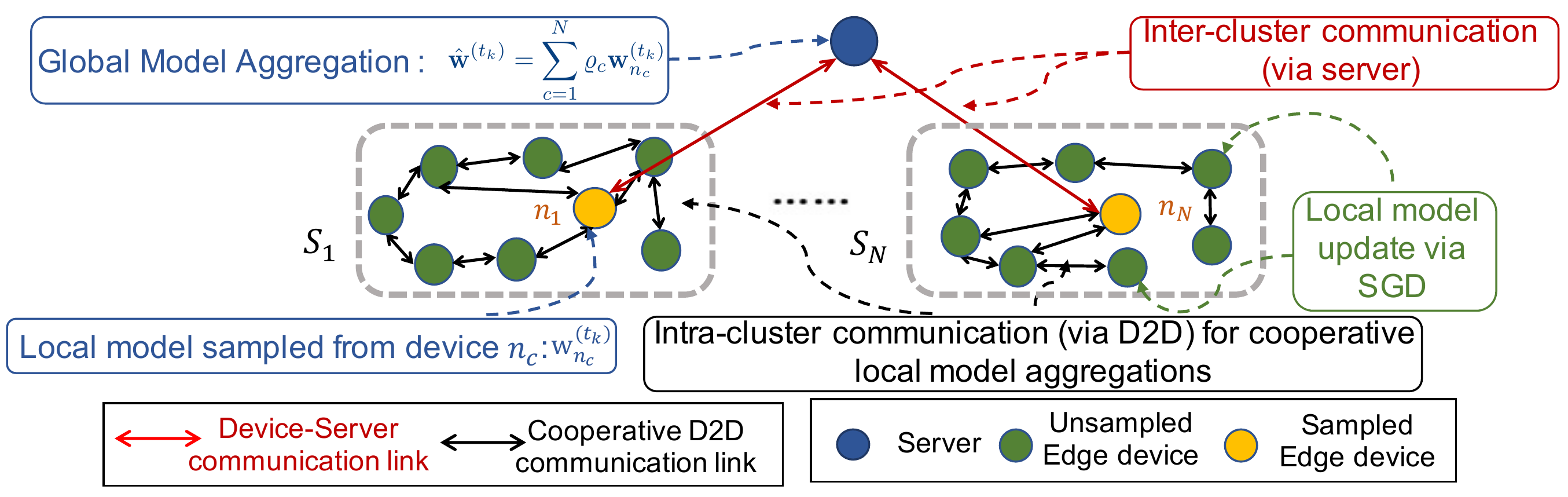}
     \caption{
    %  please make smaller to occupy one col only! This figure should go directly next to Fig 1 for easy comparison, since it helps with the discusssion in the introduction.
     Network architecture of semi-decentralized federated learning. Edge devices form local cluster topologies based on their D2D communication capability. Cooperative local model aggregations among these clusters occur using D2D communications in between global aggregations conducted by the server.} 
    %  \nm{Do you really need 4 clusters? Why not just 2?}
     \label{fig2}
     \vspace{-5mm}
\end{figure}
% Conducting multiple local descents is mostly advocated in terms of energy efficiency since it can reduce the number of global aggregations in a certain period that are more energy consuming.
%However, these approaches are prone to the existence of nodes with non-i.i.d datasets, where multiple local descent updates may lead to the model biasedness toward the local datasets, which in turn could lead to severe deterioration of the convergence speed of the model training~\cite{wang2019adaptive}.

Other works have considered improving model training in the presence of heterogeneous data among the devices via raw data sharing~\cite{9149323,wang2021device,tu2020network,zhao2018federated}. In~\cite{9149323}, the authors propose uploading portions of the local datasets to the server, which is then used to augment global model training. 
The works~\cite{zhao2018federated,wang2021device,tu2020network} mitigate local data heterogeneity by enabling the server to share a portion of its aggregated data among the devices \cite{zhao2018federated},
or by optimizing D2D data offloading~\cite{wang2021device,tu2020network}. However, raw data sharing may suffer from privacy concerns or bandwidth limitations. {In our framework, we exploit D2D communications to exchange model parameters among the devices, which alleviates such concerns.}
% \nm{rephrased:}
% The work in~\cite{zhao2018federated}  can be seen as an additional step, where the server shares a portion of its aggregated data among the devices to
% make their local data distributions less heterogeneous.~\cite{wang2021device} also attempts to promote homogeneity among local datasets, but instead through  D2D data offloading optimization. 
% Raw data sharing, however, may suffer from privacy concerns or bandwidth limitations.
  
Different from the above works, we propose a methodology that addresses the communication efficiency and data heterogeneity challenges simultaneously. To do this, we introduce distributed cooperative learning among devices into the local update process -- {as advocated} recently~\cite{hosseinalipour2020federated} -- resulting in a novel system architecture with D2D-augmented learning. In this regard, the most relevant work is~\cite{hosseinalipour2020multi}, which also studies cluster-based consensus procedure between global aggregations in federated learning. Different from~\cite{hosseinalipour2020multi}, we consider the case where (i) devices may conduct multiple (stochastic) gradient iterations between global aggregations, (ii) the global aggregations are aperiodic, and (iii) consensus procedure among the devices may occur aperiodically during each global aggregation. We further introduce a new metric of gradient diversity that extends the previous existing definition in literature. Doing so leads to a more complex system model, which we analyze to provide improvements to resource efficiency and model convergence. Consequently, the techniques used in the convergence analysis and the bound obtained differ significantly from~\cite{hosseinalipour2020multi}. There is also an emerging set of works on fully decentralized (server-less) federated learning~\cite{9154332,8950073,hu2019decentralized,lalitha2019peer}.
% \nm{why are you not comparing with any of these works by simulation?}
{However, these architectures require a well-connected communication graph among all the devices in the network, which may not be scalable to large-scale systems where devices from various regions/countries are involved in ML model training.}
Our work can be seen as intermediate between the star topology {assumed in conventional federated learning} and fully decentralized {architectures}, and constitutes a {novel} semi-decentralized learning architecture
{that mitigates the cost of resource intensive uplink communications of conventional server-based methods over star topologies, achieved via local low-power D2D communications, while improving scalability over fully decentralized server-less architectures.}
% \nm{rephrased:}
% it enjoys fewer\nm{fewer than what?} resource intensive uplink communications than in conventional server-based methods, achieved via deploying low power D2D communications. At the same time, it is more scalable as compared to server-less architectures since it does not require a well connected communication graph among all the devices in the network, which is more practical in large-scale systems where devices from various regions/countries are involved in ML model training. In particular, we partition the devices into multiple clusters and only assume the existence of communication graph inside each cluster.}

Finally, note that there is a well-developed literature on consensus-based optimization, e.g.,~\cite{yuan2011distributed,shi2014linear,consensus2009quantize,xiao2004fast}. Our work employs the distributed average consensus technique and exploits that in a new semi-decentralized machine learning architecture
% \nm{does this paper assume or mention any wireless comm? We could use that to further motivate the wireless comm model}
and contributes new results on distributed ML to this literature.

\subsection{Outline and Summary of Contributions}

\begin{itemize}[leftmargin=3mm]
\item We propose \textit{two timescale hybrid federated learning} ({\tt TT-HF}), which  augments the conventional federated learning architecture with aperiodic consensus procedure of models within local device clusters and aperiodic global aggregations by the server  (Sec. \ref{sec:tthf}).
% a novel methodology for improving ML model training efficiency at the network edge (Sec. \ref{sec:tthf}). {\tt TT-HF} augments conventional federated learning with aperiodic consensus procedure of models within local device clusters and aperiodic global aggregations by the server. %The combination of device-to-device and device-to-server communications aims to minimize network costs while maximizing trained global model accuracy and convergence speed.

%Based on an organization of edge devices into local clusters, {\tt TT-HF} augments organizes edge devices into clusters and orchestrates the edge devices by partitioning them into multiple clusters in each of which cooperative D2D communications can be performed. Subsequently, we introduce a \textit{hybrid} learning paradigm where the model training consists of both device-server (inter-cluster) and cooperative D2D (intra-cluster) communications.  Also, we constitute a \textit{two timescale} learning paradigm that relies on successive SGD interactions complemented by aperiodic consensus procedure among the devices via heterogeneous number of D2D communications  in conjunction with aperiodic global aggregations.

\item 
We propose a new model of gradient diversity, and theoretically investigate the convergence behavior of {\tt TT-HF} through techniques including coupled dynamic systems (Sec. \ref{sec:convAnalysis}). Our bounds quantify how properties of the ML model, device datasets, consensus process, and global aggregations impact the convergence speed of {\tt TT-HF}. In doing so, we obtain a set of conditions under which {\tt TT-HF} converges at a rate of $\mathcal{O}(1/t)$, similar to centralized {stochastic} gradient descent.

%and obtain an upper bound on its conv under a new general assumption on gradient diversity that accommodates for extreme heterogeneity among the local datasets. Our convergence bound relates topology structure inside different clusters, gradient diversity, number of D2D communications, and interval of consecutive  local and global aggregations to the convergence behavior of {\tt TT-HF}. We then obtain the condition under which the {\tt TT-HF} converges with the rate of $\mathcal{O}(1/t)$.

\item We develop an adaptive control algorithm for {\tt TT-HF} that tunes the global aggregation intervals, the rounds of D2D performed by each cluster, and the learning rate over time to minimize
{a trade-off between}
energy consumption, delay, {and model accuracy} (Sec. \ref{Sec:controlAlg}). This control algorithm obtains the $\mathcal{O}(1/t)$ convergence rate by including our derived conditions as constraints in the optimization.

%\item Using our convergence relationships, we develop an adaptive control algorithm for {\tt TT-HF} that obtain a set of requirements on parameters of the consensus process (D2D rounds per local update instance), local learning algorithm (gradient step size), and global coordination (interval between model aggregations) to guarantee certain convergence behaviors. We then use these conditions to develop an online control algorithm for {\tt TT-HF} that tunes these parameters during the learning period (Sec. \ref{Sec:controlAlg}).

%\item Through studying the convergence characteristic, we obtain a set of new general policies on the required number of D2D rounds at each instance of local aggregation, the choice of step-size, and the interval of local training to guarantee certain convergence behavior. We subsequently utilize our theoretical findings to develop an online control algorithm to tune these design parameters during the learning period in real-time.

\item Our subsequent experiments on popular learning tasks (Sec. \ref{sec:experiments}) verify that {\tt TT-HF} outperforms {federated learning} with infrequent global aggregations, which is commonly used in literature, substantially in terms of resource consumption and/or training time over D2D-enabled wireless devices. They also confirm that the control algorithm is able to address resource limitations and data heterogeneity across devices by adapting the local and global aggregation periods.
\end{itemize}

{We conclude in Sec.~\ref{sec:conclude} with some concluding remarks.}
% {Due to space constraints, all proofs are given as sketches. The full versions of the proofs can be found in our technical report~\cite{techReport}.}

\section{System Model and Learning Methodology}
\label{sec:tthf}
\noindent In this section, we first describe our edge network system model of D2D-enabled clusters (Sec.~\ref{subsec:syst1}) and formalize the ML task for the system (Sec.~\ref{subsec:syst2}). Then, we develop our two timescale hybrid federated learning algorithm, {\tt TT-HF} (Sec.~\ref{subsec:syst3}).
% \addFL{For the following descriptions, we omit the superscript $(k)$ without causing confusion, but keep in mind that XXX may vary with $k$.}

\subsection{Edge Network System Model}
\label{subsec:syst1}
We consider model learning over the network architecture depicted in Fig.~\ref{fig2}. The network consists of an edge server (e.g., at a base station) and $I$ edge devices gathered by the set $\mathcal{I}=\{1,\cdots,I\}$. 
% \nm{$\mathcal{I}$ appears to be unused..}
We consider a \textit{cluster-based representation} of the edge, where the devices are partitioned into $N$ sets $\mathcal{S}_1,\cdots,\mathcal{S}_N$. Cluster $\mathcal{S}_c$ contains $s_c = |\mathcal{S}_c|$ edge devices, where $\sum_{c=1}^{N}s_c=I$.
%For convenience, $\mathcal{S}_c$ will be used to refer to both the cluster itself and to the set of nodes inside it. 
We assume that the clusters are formed based on the ability of devices to conduct low-energy D2D communications, e.g., geographic proximity. Thus, one cluster may be a fleet of drones while another is a collection of local IoT sensors. In general, we do not place any restrictions on the composition of devices within a cluster, as long as they possess a common D2D protocol~\cite{hosseinalipour2020federated} and communicate with a common server.

For edge device $i\in \mathcal S_c$, we let $\mathcal{N}_i\subseteq S_c$ denote the set of its D2D neighbors, determined based on the transmit power of the nodes, the channel conditions between them, and their physical distances (cluster topology is evaluated numerically in Sec.~\ref{sec:experiments} based on a wireless communications model). We assume that D2D communications are bidirectional, i.e., $i \in \mathcal{N}_{i'}$ if and only if ${i'}\in{\mathcal N}_i$, $\forall i,i'\in\mathcal{S}_c$. Based on this, we associate a network graph $G_c=(\mathcal{S}_c, \mathcal E_c)$ to each cluster,
{where $\mathcal E_c$ denotes}
 the set of edges: $(i,{i'})\in \mathcal E_c$ if and only if $i,i'\in\mathcal{S}_c$ and $i \in \mathcal {N}_{i'}$.

The model training is carried out through a sequence of global aggregations indexed by $k=1,2,\cdots$, as will be explained in Sec.~\ref{subsec:syst3}. Between global aggregations, the edge devices $i \in \mathcal{S}_c$ will participate in cooperative consensus procedure with their neighbors $i' \in \mathcal{N}_i$. Due to the mobility of the devices, the topology of each cluster (i.e., the number of nodes and their positions inside the cluster) can change over time, although we will assume this evolution is slow compared to a the time in between two global aggregations.
% \nm{need to be precise on what you mean by time varying: when are clusters allowed to change?} }

%  As described later, our analysis can easily incorporate the case where D2D communications are not performed in some clusters. 
%  For tractability, it is assumed that the graph structure remains fixed during the interval of one global aggregation, while it can change from one global aggregation to another.
% We consider a dynamic network setting, where the set of neighbors of the nodes can evolve over time, e.g., due to device movements. However, for tractable analysis the set of nodes inside the cluster is assumed to be fixed. 

The model learning topology in this paper (Fig.~\ref{fig2}) is a distinguishing feature compared to the conventional federated learning star topology (Fig.~\ref{fig:simpleFL}). Most existing literature is based on Fig.~\ref{fig:simpleFL}, e.g.,~\cite{tran2019federated,chen2019joint,yang2019energy,yang2020federated,smith2017federated,corinzia2019variational,9006060,jiang2019improving}, where devices only communicate with the edge server, while the rest consider fully decentralized (server-less) architectures~\cite{9154332,8950073,hu2019decentralized,lalitha2019peer}.

\subsection{Machine Learning Task Model} \label{subsec:syst2}
% \subsubsection{Training dataset}
Each edge device $i$ owns a dataset $\mathcal{D}_i$ with $D_i=|\mathcal{D}_i|$ data points. Each data point $(\mathbf x,y)\in\mathcal D_i$ consists of an $m$-dimensional feature vector $\mathbf x\in\mathbb R^m$ and a label $y\in \mathbb{R}$.
% \subsubsection{Loss functions and ML objective}
We let $\hat{f}(\mathbf x,y;\mathbf w)$ denote the \textit{loss} associated with the 
data point $(\mathbf x,y)$ based on \textit{learning model parameter vector} $\mathbf w \in \mathbb{R}^M$, where $M$ denotes the dimension of the model. For example, in linear regression, $\hat{f}(\mathbf x,y;\mathbf w) = \frac{1}{2}(y-\mathbf w^\top\mathbf x)^2$. The \textit{local loss function} at device $i$ is defined as
\begin{align}\label{eq:1}
    F_i(\mathbf w)=\frac{1}{D_i}\sum\limits_{(\mathbf x,y)\in\mathcal D_i}
    \hat{f}(\mathbf x,y;\mathbf w).
\end{align}
% \chris{Should point out here that we make no assumptions on $\mathcal{D}_i$ except that they are a subset of the global dataset.} 

We define the \textit{cluster loss function} for $\mathcal{S}_c$ as the average local loss across the cluster,
\begin{align}\label{eq:c}
    \hat F_c(\mathbf w)=
    \sum\limits_{i\in\mathcal{S}_c} \rho_{i,c} F_i(\mathbf w),
\end{align}
% \nm{there is a problem here.. since clusters are time varying, $\hat F_c(\mathbf w)$ is TV as well! Is our analysis able to handle this scenario?}
where $\rho_{i,c}{=1/s_c}$
% \nm{what is the point for presenting a more general model if we dont use it?} 
is the weight associated with edge device $i\in \mathcal{S}_c$ 
{within its cluster.}
% \chris{This definition of $\rho$ is weighting each device the same. Can we generalize to the case where we weight by number of datapoints? It's more realistic.} 
% \chris{I agree with Nicolo, per my previous comment. If we want to motivate with heterogeneity it would be better to handle the case where the number of datapoints are not the same at each device. I suspect that our analysis holds even for a general $\rho_{i,c}^{(k)}$.}
The \textit{global loss function} is then defined as the average loss across the clusters,
\begin{align} \label{eq:2}
    F(\mathbf w)=\sum\limits_{c=1}^{N} \varrho_c \hat F_c(\mathbf w),
\end{align}
{weighted by the relative cluster size $\varrho_c= s_c \big(\sum_{c'=1}^N s_{c'} \big)^{-1}$.}
% \nm{dont use subscript i, used already for devices indexing! Use d?}}
% \sst{where $\varrho_c > 0$ is the weight associated with cluster $\mathcal{S}_c$ relative to the network. In this work, we will set $\rho_{i,c}= 1/s_c$, i.e., weighting each cluster's nodes equally, and $\varrho_c= s_c \big(\sum_{i=1}^N s_i \big)^{-1}$, i.e., weighting each cluster according to its size.} 
The weight of each edge node $i \in\mathcal{S}_c$ relative to the network can thus be obtained as $\rho_i=\varrho_c\cdot\rho_{i,c}= 1 / I$, meaning each node
{contributes equally to the global loss function.}
% \sst{has an equal weight in $F$.}
% \chris{Following on the previous comment, I think we can write $\rho_{i,c}^{(k)} = D_i (\sum_{j \in \mathcal{S}_c^{(k)}} D_j)^{-1}$ and $\varrho_c^{(k)} = (\sum_{j \in \mathcal{S}_c^{(k)}} D_j) (\sum_{i \in \mathcal{I}} D_i)^{-1}$.}
% \nm{why is $\rho_i$ not a function of k?}
The goal of the ML model training is to find the optimal model parameters $\mathbf w^*\in \mathbb{R}^M$ for $F$: % \nm{you have already defined this as $m$ before!}
\begin{align}
    \mathbf w^* = \mathop{\argmin_{\mathbf w \in \mathbb{R}^M} }F(\mathbf w).
\end{align}
% \nm{you also need to discuss the existence of $w^*$: its existence is implied by strong convexity of F.}

\begin{remark}
An alternative way of defining~\eqref{eq:2} is as an average performance over the datapoints, i.e., $F(\mathbf{w})=\sum_{i=1}^{I} \frac{D_i F_i(\mathbf w)}{\sum_j D_j}$~\cite{tu2020network,wang2019adaptive}. Both approaches can be justified: our formulation promotes equal performance across the devices, at the expense of giving devices with lower numbers of datapoints the same priority in the global model. Our analysis can be readily extended to this other formulation too, in which case the distributed consensus algorithms introduced in Sec.~\ref{subsec:syst3} would take a weighted form instead.
\end{remark}

%In this work, for simplicity of analysis we focus on minimizing the risk/loss over each device, where the devices are weighted equally in the global loss function~\eqref{eq:2}. One alternative way is to define the global loss as an average performance over the data points, i.e., $F(\mathbf{ w})=\sum_{i=1}^{I} \frac{D_i F_i(\mathbf w)}{D}$, where $D$ denotes the total number of data points across the devices,
%in which case instead of distributed consensus introduced later, the weighted distributed consensus algorithm~\cite{hernandes2018improved} should be used. In real scenarios, either of these cases can be used. Our formulation of global loss function achieves a global model that works well for all the devices and avoids having devices with very low performance. This comes with the drawback of giving the devices with very low number of data points the same priority as those with very large number of data points. Also, the aforementioned alternative way of formulation of loss function would bias the global model to the data distribution at the devices with higher number of data points, in which case the devices with lower number of data points may suffer from poor performance.}

In the following, we make some standard assumptions~\cite{haddadpour2019convergence,wang2019adaptive,chen2019joint,friedlander2012hybrid,tran2019federated,8851249,9261995,9337227,yang2019energy,9148815} on the ML loss function that also imply the existence and uniqueness of $\mathbf w^*$. Then, we define a new generic metric to measure the statistical heterogeneity/degree of non-i.i.d. across the local datasets:
\begin{assumption}\label{Assump:SmoothStrong}
\label{beta}
The following assumptions are made throughout the paper:
\begin{itemize}[leftmargin=3mm]
\item  \textbf{Strong convexity}:
% \nm{define $\mu\beta$ as the strong convexity param so that $\mu\in[0,1]$..}
 $F$ is $\mu$-strongly convex, i.e.,\footnote{Convex ML loss functions, e.g., squared SVM and linear regression, are implemented with a regularization term in practice to improve convergence and avoid model overfitting, which makes them strongly convex~\cite{friedlander2012hybrid}.} 
 $\forall { \mathbf w_1,\mathbf w_2}$,
%  \nm{please dont change the fontsize of equations! (fixed), Most times you can use ${+}$ vs $+$ to cut some spacing and the command $\!$}
\begin{equation} \label{eq:11} 
\!\!\!F(\mathbf w_1){\geq}F(\mathbf w_2){+}\nabla F(\mathbf w_2)^\top(\mathbf w_1{-}\mathbf w_2)
    % \nonumber \\&
    % ~~~~
    {+}\frac{\mu}{2}\Big\Vert\mathbf w_1{-}\mathbf w_2\Big\Vert^2.
\end{equation}
    \item  \textbf{Smoothness:} $F_i$ is $\beta$-smooth $\forall i$, i.e., 
    \begin{align}
\Big\Vert \nabla F_i(\mathbf w_1)-\nabla F_i(\mathbf w_2)\Big\Vert \leq & \beta\Big\Vert \mathbf w_1-\mathbf w_2 \Big\Vert,~\forall i, \mathbf w_1, \mathbf w_2,
 \end{align}
where $\beta>\mu$. This implies $\beta$-smoothness of $F$ {and $\hat F_c$} as well.\footnote{Throughout, $\Vert \cdot \Vert$ is always used to denote $\ell_2$ norm, unless otherwise stated.}
%  , i.e., $\Vert\nabla F(\mathbf w)\Vert^2 \leq 2\beta[F(\mathbf w)-F(\mathbf w^*)]$, $\forall \mathbf w$.\nm{why do you call this inequality "$\beta$-smoothness of the global function,"?}
\end{itemize}
% \begin{description}
% \item[$\beta$-smoothness:] 

% \nm{are you assuming that $F_i$ are convex? You need to state that..}
% \frank{I am not assuming $F_i$ to be convex}
% \end{description}
\end{assumption}
% We further define $\vartheta = \mu/\beta < 1$.\nm{no dont add this, keep notation simple..with my def above, $\mu$ is already the ratio} 
% \nm{dont you need convexity (not strong) of the local functions as well? You are not stating it!}
While we leverage these assumptions in our theoretical development, our experiments in Appendix~\ref{app:experiments} demonstrate that our resulting methodology is still effective in the case of non-convex loss functions (in particular, for neural networks). We  also remark that strong-convexity of the \emph{global} loss function entailed by Assumption \ref{Assump:SmoothStrong} is a much looser requirement than strong-convexity enforced on each device's local function, which we do not assume in this paper.
% \nm{thsi should go in the response as well to Rev 1}

% \begin{assumption} \label{PL}
% $F(\cdot)$ satisfies the Polyak-≈Åojasiewicz (PL) condition with constant $\mu$.
% \begin{align} \label{eq:11}
%     \frac{1}{2}\Big\Vert\nabla F(\mathbf w)\Big\Vert^2 \geq \mu(F(\mathbf w)-F(\mathbf w^*)),\ \forall \mathbf w,
% \end{align}
% \end{assumption} 

% \begin{assumption} \label{gradVar}
%     $$\mathbb E[\Vert(\widehat{\mathbf g}_j(\mathbf w_j(t))-\nabla F_j(\mathbf w_j(t))\Vert_2^2]\leq\sigma^2$$
% \end{assumption}

\begin{definition}[Gradient Diversity]\label{gradDiv}
{There exist $\delta\geq 0$ and $\zeta\in[0,2\beta]$ such that the cluster and global gradients satisfy}
% \sst{The gradient diversity across the device clusters is measured via two non-negative constants $\delta,\zeta$ that satisfy }
\begin{align} \label{eq:11}
    \left\Vert\nabla\hat F_c(\mathbf w)-\nabla F(\mathbf w)\right\Vert
    \leq \delta+ \zeta \Vert\mathbf w-\mathbf w^*\Vert,~\forall c, \mathbf w.
\end{align}
\end{definition} 
The conventional definition of gradient diversity used in literature, e.g., as in \cite{wang2019adaptive}, is a special case of~\eqref{eq:11} with $\zeta=0$. However, we observe that solely using $\delta$ on the right hand side of~\eqref{eq:11} may be troublesome since it can be shown to be not applicable to quadratic functions (such as linear regression problems),
% \nm{reivewer comment!} 
and since $\delta$ may be prohibitively large,\footnote{This is especially true {at initialization, where the initial model may be far off the optimal model $\mathbf w^*$.}
% \sst{ in the initial iterations of the model training, where $\Vert\mathbf w-\mathbf w^*\Vert$ may be arbitrarily large.}
}
{leading to overly pessimistic convergence bounds.}
% \sst{reducing the usefulness of convergence bounds derived based on $\delta$ by making them too loose.} 
Indeed, for all functions satisfying Assumption \ref{Assump:SmoothStrong}, Definition \ref{gradDiv} holds. To see this, note that we can upper bound the gradient diversity using the triangle inequality as
\begin{align}\label{eq:motivNewDiv}
    &\Vert\nabla\hat F_c(\mathbf w)-\nabla F(\mathbf w)\Vert
    \nonumber \\&
    = \Vert\nabla\hat F_c(\mathbf w)-\nabla\hat F_c(\mathbf w^*)+\nabla\hat F_c(\mathbf w^*)-\underbrace{\nabla F(\mathbf w^*)}_{=0}-\nabla F(\mathbf w)\Vert
    \nonumber \\&
    \leq \Vert\nabla\hat F_c(\mathbf w)-\nabla\hat F_c(\mathbf w^*)\Vert+\Vert\nabla\hat F_c(\mathbf w^*)\Vert 
    \nonumber \\&
    ~~~~+ \Vert\nabla F(\mathbf w)-\nabla F(\mathbf w^*)\Vert
    \leq \delta+2\beta\Vert\mathbf w-\mathbf w^*\Vert,
\end{align}
where in the last step above we used the smoothness condition and upper bounded the cluster gradients at the optimal model as $\Vert\nabla\hat F_c(\mathbf w^*)\Vert \leq \delta$. 
% \sst{This is a more practical assumption than $\Vert\nabla\hat F_c(\mathbf w)-\nabla F(\mathbf w)\Vert \leq \delta$ that must hold $\forall\mathbf w$.}\nm{already said before, no need to repeat!}
    % From \eqref{eq:motivNewDiv}, it also follows that $\zeta\leq 2\beta$. 
We then introduce a ratio $\omega=\frac{\zeta}{2\beta}$, where $\omega\leq 1$ according to~\eqref{eq:motivNewDiv}. Considering $\zeta$ in~\eqref{eq:11} changes the dynamics of the convergence analysis and requires new techniques to obtain the convergence bounds, which are part of our contributions in this work.

\begin{figure*}[t]
\includegraphics[width=\textwidth]{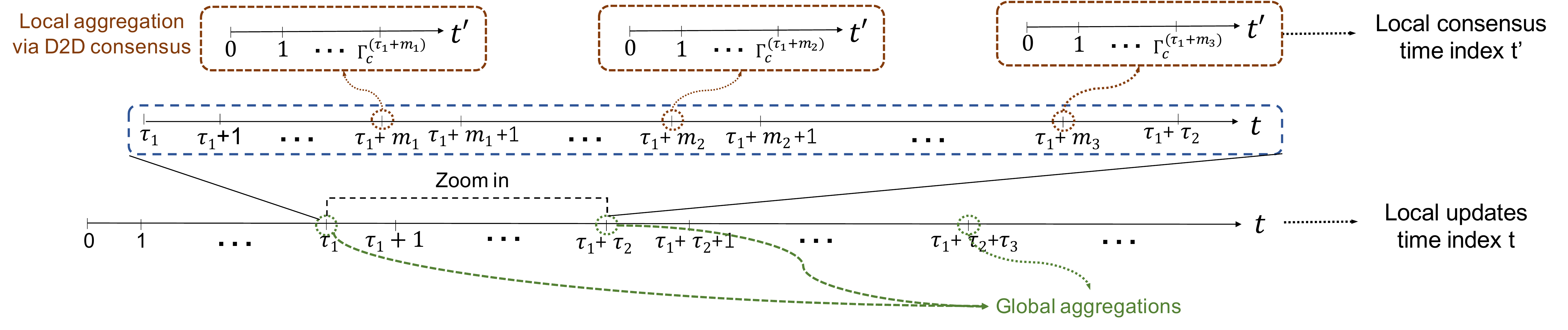}
\centering
\caption{Depiction of two timescales in {\tt TT-HF}. Time index $t$ captures the local descent iterations and global aggregations. In each local training interval, the nodes will aperiodically engage in consensus procedure. Time index $t'$ captures the rounds of these local aggregations.}
\label{fig:twoTimeScale}
\vspace{-5mm}
\end{figure*}

\subsection{{\tt TT-HF}: Two Timescale Hybrid Federated Learning} \label{subsec:syst3}
\subsubsection{Overview and rationale}
{\tt TT-HF} is comprised of a sequence of local model training intervals in-between aperiodic global aggregations. During each interval, the devices conduct local stochastic gradient descent (SGD) iterations and aperiodically synchronize their model parameters through local consensus procedure within their clusters.

There are three main practical reasons for incorporating the local consensus procedure into the learning paradigm. First, local consensus can help further suppress any bias of device models to their local datasets, which is one of the main challenges faced in federated learning in environments where data may be non-i.i.d. across the network~\cite{wang2019adaptive}. Second, 
{local D2D communications during the consensus procedure, typically performed over short ranges~\cite{hmila2019energy,dominic2020joint},
 are} expected to incur much lower device power consumption compared with the global aggregations, which require uplink transmissions to potentially far-away aggregation points (e.g., from smartphone to base station).
 %while D2D communications are performed over short ranges~\cite{hmila2019energy,dominic2020joint}.
 Third, D2D is becoming a prevalent feature of 5G-and-beyond wireless networks~\cite{7994918,9048621}.
% This leads to less dependency on energy-intensive global aggregations (due to the requirement of uplink transmissions) for parameter synchronization, where the global aggregations can occur less frequently. 
% \nm{I got lost with the explanations in this paragraph}

\subsubsection{{\tt TT-HF} procedure}
% \nm{it is essential that this is explained clearly. I think it needs to be improved. The consensus is explained too late in the paragraph. I think you should follow a more logical order: SGD --> consensus --> synchronization --> repeat}
We index time as a set of discrete time indices $\mathcal{T} = \{1, 2, ...\}$. Global aggregation $k$ occurs at time $t_k \in \mathcal{T}$ (with $t_0=0$), so that $\mathcal{T}_k = \{t_{k-1} + 1,...,t_k\}$  denotes
the $k$th \textit{local model training interval}
between aggregations $k-1$ and $k$, of duration $\tau_k =t_k-t_{k-1}$. %we have $t_k = \sum_{\ell=1}^{k}\tau_\ell$ $\forall k$ and $t_0 = 0$.
Since global aggregations are aperiodic, in general $\tau_{k} \neq \tau_{k'}$ for $k \neq k'$.

%As in federated learning, there will be $K$ \textit{global aggregations}, which we will index as $k = 1,...,K$. 

The model computed by the server at the $k$th global aggregation is denoted as 
% The $k$th global aggregation of model parameters computed at the server will be denoted as 
$\hat{\mathbf w}^{(t_k)} \in \mathbb{R}^M$, which will be defined in \eqref{15}. The model training procedure starts with the server broadcasting $\hat {\mathbf w}^{(0)}$ to initialize the devices' local models.

\textbf{Local SGD iterations}: Each device $i \in \mathcal{I}$ has its own local model, denoted $\mathbf w_i^{(t-1)} \in \mathbb{R}^M$ at time $t-1$. Device $i$  performs successive local SGD iterations on its model over time. Specifically, at time $t \in \mathcal{T}_k$, device $i$ randomly samples a mini-batch $\xi_i^{(t-1)}$ of fixed size from its own local dataset $\mathcal D_i$, and calculates the \textit{local gradient estimate}
\begin{align}\label{eq:SGD}
    \widehat{\mathbf g}_{i}^{(t-1)}=\frac{1}{\vert\xi_i^{(t-1)}\vert}\sum\limits_{(\mathbf x,y)\in\xi_i^{(t-1)}}
    \nabla\hat{f}(\mathbf x,y;\mathbf w_i^{(t-1)}).
\end{align}
% which is an unbiased estimate of the local gradient, i.e., $\mathbb E_{t}[\widehat{\mathbf g}_{i}^{(t)}]=\nabla F_i(\mathbf w_i^{(t)})$\nm{what about its variance?}
It then computes its \textit{intermediate updated local model} as
\begin{align} \label{8}
    {\widetilde{\mathbf{w}}}_i^{(t)} = 
           \mathbf w_i^{(t-1)}-\eta_{t-1} \widehat{\mathbf g}_{i}^{(t-1)},~t\in\mathcal T_k,
\end{align}
where $\eta_{t-1} > 0$ denotes the step size.
{The local model $\mathbf{w}_i^{(t)}$ is then updated according to the following consensus-based procedure.}
% \sst{ 
% Based on ${\widetilde{\mathbf{w}}}_i^{(t)}$, the updated local model $\mathbf{w}_i^{(t)}$ is computed either through setting it to ${\widetilde{\mathbf{w}}}_i^{(t)}$ or through consensus, as will be explained next.}
% \nm{w= tilde w follows as a special case with 0 concensus steps, no need to say this, it is confusing at this point.}

\textbf{Local model update}: 
%\nm{simplify notation: you can say that the discussion refers to a reference round k, so you drop the dependence on k but should kkep in mind that v may vary with k..}
At each time $t \in \mathcal{T}_k$, each cluster may engage in local consensus procedure for model updating. The decision of whether to engage in this consensus process at time $t$ -- and if so, how many iterations of this process to run -- will be developed in Sec.~\ref{Sec:controlAlg} based on a performance-efficiency trade-off optimization. If the devices do not execute consensus procedure, we have the conventional model update rule ${{\mathbf{w}}}_i^{(t)} = {\widetilde{\mathbf{w}}}_i^{(t)}$ from \eqref{8}. Otherwise, multiple \textit{rounds} of D2D communication take place, where in each round parameter transfers occur between neighboring devices. In particular, 
assuming $\Gamma^{(t)}_{{c}}>0$ rounds for cluster $c$ at time $t$, and
letting $t'=0,\dots,\Gamma^{(t)}_{{c}}-1$ index the rounds, each node $i \in \mathcal{S}_c$ carries out the following for $t'=0,\dots,\Gamma^{(t)}_{{c}}-1$:
%During each local model training interval, the devices aperiodically form consensus through engaging in cooperative D2D communications, determining the time of occurrence of which and the number of D2D communications performed are parts of design elaborated in Sec.~\ref{Sec:controlAlg}. If at time $t\in\mathcal T_k$ the devices do not form consensus, we have the conventional model update rule ${{\mathbf{w}}}_i^{(t)} = {\widetilde{\mathbf{w}}}_i^{(t)}=
%           \mathbf w_i^{(t-1)}-\eta^{(k)} \widehat{\mathbf g}_{i}^{(t-1)}$. Upon forming consensus at time $t$, the devices perform multiple \textit{rounds} of D2D,
        %   \nm{does this happen within t? Why are the consensus updates not counted at the same timescale as local SGD? How do you make sure that the clusters are synced?} 
%        each of which consists of parameter transfers between neighboring devices. In particular, each node $i\in\mathcal{S}_c$ conducts the following iteration for $t'=0,\cdots,\Gamma^{(t)}_{{c}}-1$, where $\Gamma^{(t)}_{{c}}$ denotes the rounds of D2D in the respective cluster:
    \begin{equation}\label{eq:ConsCenter}
       \textbf{z}_{i}^{(t'+1)}= v_{i,i} \textbf{z}_{i}^{(t')}+ \sum_{j\in \mathcal{N}_i} v_{i,j}\textbf{z}_{j}^{(t')},
  \end{equation}

\vspace{-0.1in}
\noindent 
where $\textbf{z}_{i}^{(0)}=\widetilde{\mathbf{w}}_i^{(t)}$ is the node's intermediate local model from \eqref{8}, and $v_{i,j}\geq 0$, $\forall i,j$ is the consensus weight that node $i$ applies to the vector received from $j$. At the end of this process, node $i$ takes $\mathbf w_i^{(t)} = \textbf{z}_{i}^{(\Gamma^{(t)}_{{c}})}$ as its updated local model.
%\nm{you need to make an effor to simplify notation.. these nested superscripts are really unappealing..}

The index $t'$ corresponds to the second timescale in {\tt TT-HF}, referring to the consensus process, as opposed to the index $t$ which captures the time elapsed by the local gradient iterations. Fig.~\ref{fig:twoTimeScale} illustrates these two timescales, where at certain local iterations $t$ the consensus process $t'$ is run.
% \nm{again, why not three timescales?}

To analyze this update process, we will find it convenient to express the consensus procedure in matrix form. Let $\widetilde{\mathbf{W}}^{(t)}_{{c}} \in \mathbb{R}^{s_c \times M}$ denote the matrix of intermediate updated local models of the $s_c$ nodes in cluster $\mathcal{S}_c$,
% \sst{ prior to consensus at time $t\in\mathcal{T}_k$}
 where the $i$-th row of $\widetilde{\mathbf{W}}^{(t)}_{{c}}$ corresponds to device $i$'s intermediate local model $\widetilde{{\mathbf{w}}}_i^{(t)}$. Then, the
matrix of updated device parameters after the consensus stage, ${\mathbf{W}}^{(t)}_{{c}}$, can be written as
%\nm{notation!}
  \begin{equation}\label{eq:consensus}
      \mathbf{W}^{(t)}_{{c}}= \left(\mathbf{V}_{{c}}\right)^{\Gamma^{(t)}_{{c}}} \widetilde{\mathbf{W}}^{(t)}_{{c}}, ~t\in\mathcal T_k,
  \end{equation}
where $\Gamma^{(t)}_{{c}}$ denotes the rounds of D2D consensus in the cluster, and $\mathbf{V}_{{c}}=[v_{i,j}]_{1\leq i,j\leq s_c} \in \mathbb{R}^{s_c \times s_c}$ denotes the \textit{consensus matrix}, which we characterize further below. The $i$-th row of ${\mathbf{W}}^{(t)}_{{c}}$ corresponds to device $i$'s local update ${{\mathbf{w}}}_i^{(t)}$, which is then used in~\eqref{eq:SGD} to calculate the gradient estimate for the next local update. For the times $t \in \mathcal T_k$ where consensus is not used, we set $\Gamma^{(t)}_{{c}}=0$, implying $\mathbf{W}^{(t)}_{{c}}= \widetilde{\mathbf{W}}^{(t)}_{{c}}$ so that devices use their individual gradient updates.
%\nm{we should make a rule that we should NEVER have nested subscript/superscripts. It is really bad to see. I can bet that the reviewers will complain a lot about the notation..}

% \nm{you need to add another remark here: signal is quantized, encoded transmitted and decoded, as opposed to analog tx! This is one of the critical comments of the reviewer}
\begin{remark} \label{rem:2}
     Note that the graph $G_c$ may
    % \sst{
    %  parameters $\mathcal S_c$, $s_c$, $\mathcal N_i$, and $G_c(\mathcal S_c,\mathcal E_c)$ for each cluster $c$ can in general} 
     change {over time $t$}. In this paper, we only require that
     {the set of devices in each cluster} remain fixed during each global aggregation period $k$. We drop the dependency on $t$ for simplicity of presentation, although
    %  \nm{I dont think though is formal. used although instead (I found a few instances)}
     the analysis implicitly accommodates it. We similarly do so in notations for node and cluster weights $\rho_{i,c}$, $\varrho_c$ introduced in Sec.~\ref{subsec:syst2} and consensus parameters $v_{i,j}$, $\mathbf V_c$, $\lambda_c$ in Sec.~\ref{subsec:syst3}. 
    Assuming a fixed vertex set during each global aggregation period is a practical assumption, especially when the devices move slowly and do not leave the cluster during each local training interval. 
    {Moreover, although in the analysis we assume that transmissions are outage- and error-free, in Sec.~\ref{sec:experiments} we will perform a numerical evaluation to evaluate the impact of fast fading and limited channel state information (CSI), resulting in outages and time-varying link configurations.}
    % \nm{this is not true! Our analysis and model DO NOT incorporate fast fading!}
    % \sst{
    % Nevertheless, the edge configuration among the devices can vary across the local SGD iterations due to fast fading in the channels that results in time varying link configurations (see Sec.~$ref{sec:experiments}$ for the implementation details).}
\end{remark}

\textbf{Consensus characteristics}: The consensus matrix $\mathbf{V}_{{c}}$ can be constructed in several ways based on the cluster topology $G_c$.
% \nm{paragprah moved after assumption 2}
In this paper, we make the following standard assumption~\cite{xiao2004fast}:

\begin{assumption}\label{assump:cons}
The consensus matrix $\mathbf{V}_{{c}}$ satisfies the following conditions: (i) $\left(\mathbf{V}_{{c}}\right)_{m,n}=0~~\textrm{if}~~ \left({m},{n}\right) \notin \mathcal{E}_{{c}}$, i.e., nodes only receive from their neighbors; (ii) $\mathbf{V}_{{c}}\textbf{1} = \textbf{1}$, i.e., row stochasticity; (iii) $\mathbf{V}_{{c}} = {\mathbf{V}_{{c}}}^\top$, i.e., symmetry; and (iv)~$\rho \big(\mathbf{V}_{{c}}-\frac{\textbf{1} \textbf{1}^\top}{s_c} \big) < 1$, i.e., the largest eigenvalue of $\mathbf{V}_{{c}}-\frac{\textbf{1} \textbf{1}^\top}{s_c}$ has magnitude $<1$. %where $\textbf{1}$ is the vector of 1s and \underline{$\rho(\mathbf{A})$ is the spectral radius of matrix $\mathbf{A}$}.
\end{assumption}
{For example,}
from the distributed consensus literature~\cite{xiao2004fast}, one common choice
{that satisfies this property is $v_{i,j}=d_c,\forall j\in\mathcal N_i$ and $v_{i,i}=1-d_c|\mathcal N_i|$,}
% \sst{$eqref{eq:ConsCenter}$ is $\textbf{z}_{i}^{(t'+1)} = \textbf{z}_{i}^{(t')}+d_{{c}}\sum_{j\in \mathcal{N}_i} (\textbf{z}_{j}^{(t')}-\textbf{z}_{i}^{(t')})$,}
where $0 < d_{{c}} < 1 / D_{{c}}$ and $D_{{c}}$ denotes the maximum degree among the nodes in $G_{{c}}$.
% \sst{ This results in a consensus matrix for the cluster with $v_{i,i}=1-|\mathcal{N}_i|d_{{c}}$ and $v_{i,j}=d_{{c}}, i \neq j$. }
% Our theoretical results in Sec. \ref{sec:convAnalysis} will show how properties of $\mathbf{V}_{{c}}$, particularly the spectral radius $\rho(\cdot)$ of the matrix $\mathbf{V}_{{c}}-\frac{\textbf{1} \textbf{1}^\top}{s_c}$, affect the convergence rate of {\tt TT-HF}. 

The consensus procedure process can be viewed as an imperfect aggregation of the models in each cluster. Specifically, we can write the local parameter at device $i\in \mathcal{S}_c$ as
\begin{align} \label{eq14}
    \mathbf w_i^{(t)} = \bar{\mathbf w}_c^{(t)} + \mathbf e_i^{(t)},
\end{align}
where $\bar{\mathbf w}_c^{(t)} = \sum_{i\in\mathcal{S}_c} \rho_{i,c} \tilde{\mathbf w}_i^{(t)}$ is the average
% \nm{what do you mean by the average to be perfect? ITs just the average} 
of the local models in the cluster and $\mathbf e_i^{(t)} \in \mathbb{R}^M$ denotes the \textit{consensus error} caused by limited D2D rounds (i.e., $\Gamma_c^{(t)} < \infty$) among the devices, {which can be bounded as in the following lemma}.

% \sst{ We next introduce a definition of the divergence across intermediate updated local models, which we can use to derive an upper bound on the consensus error:}
% \nm{since you are "hiding" the consensus steps within the time index $t$, one may argue that you are not really capturing these time constraints..}

\begin{algorithm}[t]
\small 
\SetAlgoLined
\caption{Two timescale hybrid federated learning {\tt TT-HF} with set control parameters.} \label{TT-HF}
% \KwResult{Write here the result }
\KwIn{Length of training $T$, number of global aggregations $K$, D2D rounds $\{\Gamma_c^{(t)}\}_{t=1}^{T},~\forall c$, length of local model training intervals $\tau_k,~ k=1,...,K$} 
\KwOut{Final global model $\hat{\mathbf w}^{(T)}$}
 // Initialization by the server \\
 Initialize $\hat{\mathbf w}^{(0)}$ and broadcast it among the devices along with the indices $n_c$ of the sampled devices for the first global aggregation.\\
 \For{$k=1:K$}{
     \For{$t=t_{k-1}+1:t_k$}{
        \For{$c=1:N$}
        {
         // Procedure at the clusters \\
         Each device $i\in\mathcal{S}_c$ performs local SGD update based on~\eqref{eq:SGD} and~\eqref{8} using $\mathbf w_i^{(t-1)}$ to obtain~$\widetilde{\mathbf{w}}_i^{(t)}$.\\
        %  with:\\
        %   $\widetilde{\mathbf w}_i^{(t)} = 
        %   \mathbf w_i^{(t-1)}-\eta_{t-1} \widehat{\mathbf g}_{i}^{(t-1)}$\; 
        Devices inside the cluster conduct $\Gamma^{(t)}_{{c}}$ rounds of consensus procedure based on~\eqref{eq:ConsCenter}, initializing  $\textbf{z}_{i}^{(0)}=\widetilde{\mathbf{w}}_i^{(t)}$ and setting $\mathbf w_i^{(t)} = \textbf{z}_{i}^{(\Gamma^{(t)}_{{c}})}$.
        % on $\widetilde{\mathbf w}_i^{(t)}$ as:\\
        % $\textbf{z}_{i}^{(0)}=\widetilde{\mathbf{w}}_i^{(t)}$\;
        % \For{$t'=0:\Gamma^{(t)}_{{c}}-1$}{
        %     $\textbf{z}_{i}^{(t'+1)}= v^{(k)}_{i,i} \textbf{z}_{i}^{(t')}+\hspace{-2mm} \sum_{j\in \mathcal{N}^{(k)}_i} v^{(k)}_{i,j}\textbf{z}_{j}^{(t')}$
        % }
        % $\mathbf w_i^{(t)} = \textbf{z}_{i}^{(\Gamma^{(t)}_{{C}})}$\; 
      }
      \If{$t=t_k$}{
      // Procedure at the clusters \\
      Each sampled device $n_c$ sends $\mathbf w_{n_c}^{(t_k)}$ to the server.\\
      // Procedure at the server \\
    %   Estimate $\delta$\;
     Compute $\hat{\mathbf w}(t)$ using \eqref{15}, and
      broadcast it among the devices along with the indices $n_c$ chosen for the next global aggregation.
      }
 }
}
\end{algorithm}

% \nm{Im going to remove this.. it is very confusing, it makes it seem like you are IMPOSING an upper bound on this divergence to make the analysis tractable, but this is not the case! Let me rephrase}
% \begin{definition}\label{def:clustDiv}
% \nm{REMOVE:}
% The divergence of intermediate updated local model parameters in cluster $\mathcal{S}_c$ at time $t\in\mathcal{T}_k$, denoted by $\Upsilon^{(t)}_{{c}}$, is defined as follows: \begin{equation}\label{eq:Updef}
%   \big\lvert (\widetilde{\mathbf{w}}^{(t)}_{{i}})_{z} -(\widetilde{\mathbf{w}}^{(t)}_{{i'}})_z \big\rvert \leq \Upsilon^{(t)}_{{c}} ,~ \forall {i},{i'}\in\mathcal{S}_c, 1\leq z\leq M,
% \end{equation} 
% where $(.)_z$ denotes the $z$-th element of the indexed vector.
% \end{definition}

\begin{lemma}\label{lemma:cons}
After performing $\Gamma^{(t)}_{{c}}$ rounds of consensus in cluster $\mathcal{S}_c$ with the consensus matrix $\mathbf{V}_{{c}}$, the consensus error $\mathbf e_i^{(t)}$ satisfies
% \nm{I think there was a factor M missing! I got something different:}
% \begin{equation} \label{eq:cons}
%   \hspace{-3mm} \Vert \mathbf e_i^{(t)} \Vert  \hspace{-.5mm}  \leq \hspace{-.5mm} 
%   (\lambda_{{c}})^{\Gamma^{(t)}_{{c}}}\hspace{-1mm} 
% {
% \sqrt{s_cM}\hspace{-.5mm} 
% \underbrace{\max_{j,j'\in\mathcal S_c}\Vert \widetilde{\mathbf{w}}^{(t)}_{{j}} -\widetilde{\mathbf{w}}^{(t)}_{{j'}}\Vert_\infty}_{\triangleq \Upsilon^{(t)}_{{c}}}},
% % \mst{   \left(\lambda_{{c}}\right)^{\Gamma^{(t)}_{{c}}} s_c{\Upsilon^{(t)}_{{c}}} M},
% ~ \forall i\in \mathcal{S}_c, \hspace{-3mm} 
% \end{equation}
% \nm{I dont thim we need the infinity norm! In fact, we are always using L2 in our proofs: so replace with following}
% \add{
\begin{equation} \label{eq:cons}
  \hspace{-3mm} \Vert \mathbf e_i^{(t)} \Vert  \hspace{-.5mm}  \leq (\lambda_{{c}})^{\Gamma^{(t)}_{{c}}}
\sqrt{s_c}\underbrace{\max_{j,j'\in\mathcal S_c}\Vert\tilde{\mathbf w}_j^{(t)}-\tilde{\mathbf w}_{j'}^{(t)}\Vert}_{\triangleq \Upsilon^{(t)}_{{c}}},
% \mst{   \left(\lambda_{{c}}\right)^{\Gamma^{(t)}_{{c}}} s_c{\Upsilon^{(t)}_{{c}}} M},
~ \forall i\in \mathcal{S}_c. \hspace{-3mm} 
\end{equation}
% }
where $\lambda_{{c}}=\rho \big(\mathbf{V}_{{c}}-\frac{\textbf{1} \textbf{1}^\top}{s_c} \big)$. 
% and $\Vert\mathbf a\Vert_\infty=\max_{z}|\mathbf a_z|$ denotes the $\ell_\infty$ norm.}
% \sst{each $\lambda_{{c}}$ is a constant such that $1 > \lambda_{{c}} \geq \rho \big(\mathbf{V}_{{c}}-\frac{\textbf{1} \textbf{1}^\top}{s_c} \big)$.}
\end{lemma}
% \nm{all sketches should go in an appendix! Having them here breaks the flow of the discussion}
% \begin{skproof}
% \add{See Appendix \nm{xx}.}
% \end{skproof}
\begin{skproof}
  Let
  {$
\overline{\mathbf{W}}^{(t)}_{{c}}=\frac{1}{s_c}\mathbf 1 \mathbf 1^\top\widetilde{\mathbf{W}}^{(t)}_{{c}}
  $ be the matrix with rows given by the average model parameters across the cluster, and let
   \begin{align*}
      \mathbf E^{(t)}_{{c}}= {\mathbf{W}}^{(t)}_{{c}}-\overline{\mathbf{W}}^{(t)}_{{c}}
      =[\left(\mathbf{V}_{{c}}\right)^{\Gamma^{(t)}_{{c}}}-\mathbf 1\mathbf 1^\top/s_c] [\widetilde{\mathbf{W}}^{(t)}_{{c}}-\overline{\mathbf{W}}^{(t)}_{{c}}]
      ,
  \end{align*}
  so that $[\mathbf E^{(t)}_{{c}}]_{i,:}$ ($i$th row of $\mathbf E^{(t)}_{{c}}$) $=\mathbf e_{i}^{(t)}$, where in the second step we used \eqref{eq:consensus} and the fact that $\mathbf 1^\top \mathbf E^{(t)}_{{c}}=\mathbf 0$ (hence
  $\mathbf E^{(t)}_{{c}}=[\mathbf I-\mathbf 1\mathbf 1^\top/s_c]\mathbf E^{(t)}_{{c}}$). Therefore, using Assumption \ref{assump:cons}, we can bound the consensus error as
  \begin{align} \label{eq:consensus-bound-1}
      &\Vert\mathbf e_{i}^{(t)}\Vert^2\leq
      \mathrm{trace}((\mathbf E^{(t)}_{{c}})^{\top}\mathbf E^{(t)}_{{c}})
      \\&\nonumber
      =
      \mathrm{trace}\Big(
      [\widetilde{\mathbf{W}}^{(t)}_{{c}}{-}\overline{\mathbf{W}}^{(t)}_{{c}}]^\top
      [\left(\mathbf{V}_{{c}}\right)^{\Gamma^{(t)}_{{c}}}{-}\mathbf 1\mathbf 1^\top/s_c]^2 [\widetilde{\mathbf{W}}^{(t)}_{{c}}{-}\overline{\mathbf{W}}^{(t)}_{{c}}]
\Big)
\\&
\leq (\lambda_{{c}})^{2\Gamma^{(t)}_{{c}}}
\sum_{j=1}^{s_c}\Vert\tilde{\mathbf w}_j^{(t)}-\bar{\mathbf w}_c^{(t)}\Vert^2
\nonumber
\\&
\leq (\lambda_{{c}})^{2\Gamma^{(t)}_{{c}}}
\frac{1}{s_c}\sum_{j,j'=1}^{s_c}\Vert\tilde{\mathbf w}_j^{(t)}-\tilde{\mathbf w}_{j'}^{(t)}\Vert^2
\nonumber
\\&
\nonumber
\leq (\lambda_{{c}})^{2\Gamma^{(t)}_{{c}}}
s_c\max_{j,j'\in\mathcal S_c}\Vert\tilde{\mathbf w}_j^{(t)}-\tilde{\mathbf w}_{j'}^{(t)}\Vert^2,
% \mst{(\lambda_{{c}})^{2\Gamma^{(t)}_{{c}}}
% s_cM\max_{j,j'\in\mathcal S_c}\Vert\tilde{\mathbf w}_j^{(t)}-\bar{\mathbf w}_c^{(t)}\Vert_\infty^2,}\nm{dont need this!}
  \end{align}
  so that the result directly follows.  For the complete proof, refer to Appendix \ref{app:consensus-error}.
  }
\end{skproof}

Note that $\Upsilon^{(t)}_{{c}}$ defined in~\eqref{eq:cons} captures the divergence of intermediate updated local model parameters in cluster $\mathcal{S}_c$ at time $t\in\mathcal{T}_k$ (before consensus is performed).
Intuitively, according to~\eqref{eq:cons}, to make the consensus error smaller, more rounds of consensus need to be performed. However, this may be impractical due to energy and delay considerations,
hence a trade-off arises between the consensus error and the energy/delay cost. This trade-off will be optimized by tuning $\Gamma_c^{(t)}$, via our adaptive control algorithm developed in Sec. \ref{Sec:controlAlg}.
% \nm{I removed the rest.. lets keep the discussion simple}
%Intuitively, $\mathbf e_i^{(t)} = \mathbf 0$ in two cases: (i) All the devices inside the cluster have the same intermediate local model parameter (i.e., $\widetilde{\mathbf w}_i^{(t)} = \widetilde{\mathbf w}_{i'}^{(t)}$ for $i, i' \in \mathcal{S}_c$), in which case $\Upsilon^{(t)}_{{c}} = 0$. (ii) Infinite rounds of D2D are performed (i.e., $\Gamma_c^{(t)} \rightarrow \infty$), which drives $\left(\lambda_{{c}}\right)^{\Gamma^{(t)}_{{c}}}$ to $0$. This second case would be impractical, however, due to energy and delay considerations, which emphasizes the tradeoffs our adaptive control algorithm in Sec. \ref{Sec:controlAlg} must optimize over in tuning $\Gamma_c^{(t)}$.

\textbf{Global aggregation}: At the end of each local model training interval $\mathcal{T}_k$, the global model $\mathbf{w}$ will be updated based on the trained local model updates. Referring to Fig.~\ref{fig2}, the main server will \textit{sample} one device from each cluster $c$ uniformly at random, and request these devices to upload their local models, so that the new global model is updated as
\begin{align} \label{15}
    \hat{\mathbf w}^{(t)} &= 
          \sum\limits_{c=1}^N \varrho_c \mathbf w_{n_c}^{(t)}, \;\; t=t_k, k=1,2,...
\end{align}
{where $n_c$ is the node sampled from cluster $c$ at time $t$.}
This sampling technique is introduced to reduce the uplink communication cost by a factor of the cluster sizes,
{and is enabled by the consensus procedure, which mimics a local aggregation procedure within a cluster (albeit imperfectly due to consensus errors, see \eqref{eq14})} \cite{hosseinalipour2020federated}. 
        %   . Formally, at time index $t = t_k$, the main server samples device $n_c^{(k)} \in \mathcal{S}_c^{(k)}$ for each cluster $c$, and updates the global model as follows:
The global model is then broadcast by the main server to all of the edge devices, which override their local models at time $t_k$: $\mathbf w_i^{(t_{k})} = \hat{\mathbf{w}}^{(t_{k})}$, $\forall i$.
% \nm{there is no delay to wait for these new models?} \frank{the devices are directly sync with the new models}
The process then repeats for $\mathcal{T}_{k+1}$.
% \chris{For the sake of proofs we define the global model update at any time instance as ...., although it is only realized at the instance of global aggregations at the main server.}

A summary of the {\tt TT-HF} algorithm developed in this section (for set control parameters) is given in Algorithm~\ref{TT-HF}.

% \nm{No! This remark is too strong and requires a rigorous proof... please revise. Again, we CANNOT claim that our model is robust agasinst fading, outages, just because we model "consensus error". IT is not the same! We can only make these claims empirically..}
\begin{remark} \label{rem:3}
 Note that we consider digital transmission (in both D2D and uplink/downlink communications) where using state-of-the-art techniques in encoding/decoding, e.g., low density parity check (LDPC) codes, the bit error rate (BER) is reasonably small and negligible~\cite{8316763}. Moreover, the effect of quantized model transmissions can be readily incorporated using techniques developed in~\cite{consensus2009quantize},
and precoding techniques may be used to mitigate the effect of signal outage due to fading~\cite{precoding2021}.
  Therefore, in this analysis, we assume that the model parameters transmitted by the devices to their neighbors (during consensus) and then to the server (during global aggregation) are received with no errors at the respective receivers. The impact of outages due to fast fading and lack of CSI will be studied numerically in Sec. \ref{sec:experiments}.
 %the model parameters devices' transmitted model parameters can be recovered at the server almost perfectly unless outage occurs, which can be avoided using precoding~\cite{precoding2021} or taken into account to obtain proper rounds of D2D communications via the method of~\cite{consensus2009wireless}. 
%  In this regime, upon using any D2D protocol for conducting local model aggregations, w $i$ can be written as~\eqref{eq14} for a general $\mathbf{e}_i$, caused by finite D2D communication rounds, our analysis and the subsequent results \hl{(Proposition 1, Theorem 1 and Theorem 2)}\nm{please use latex ref} readily hold. 
\end{remark}

\section{Convergence Analysis of {\tt TT-HF}} \label{sec:convAnalysis}
\noindent In this section, we theoretically analyze the convergence behavior of {\tt TT-HF}. Our main results are presented in Sec.~\ref{ssec:convAvg} and Sec.~\ref{ssec:sublinear}. Before then, in Sec.~\ref{ssec:definitions}, we introduce some additional definitions and a key proposition for the analysis.

\subsection{Definitions and Bounding Model Dispersion}
\label{ssec:definitions}
We first introduce a standard assumption on the noise of gradient estimation, and then define an upper bound on the average of consensus error for the clusters:
\begin{assumption} \label{assump:SGD_noise}
    Let ${\mathbf n}_{i}^{(t)}{=}\widehat{\mathbf g}_{i}^{(t)}{-}\nabla F_i(\mathbf w_{i}^{(t)})$ $\forall i,t$ denote the noise of the estimated gradient through the SGD process for device $i$. We assume that it is unbiased with bounded variance, i.e.
     $\mathbb{E}[{\mathbf n}_{i}^{(t)}|\mathbf w_{i}^{(t)}]{=}0$ and $\exists \sigma{>}0{:}~\mathbb{E}[\Vert{\mathbf n}_{i}^{(t)}\Vert^2|\mathbf w_{i}^{(t)}]{\leq}\sigma^2$, $\forall i,t$.
\end{assumption}
% \noindent  In this section, we study the convergence of
% {\tt TT-HF} in terms of the optimality gap
%  $F(\hat{\mathbf w}^{(K)})-F(\mathbf w^*)$ between the
% global objective function at the algorithm output $\hat{\mathbf w}^{(K)}$ and at the globally optimal parameter vector $\mathbf w^*$.   

{Moreover, the following condition bounds the consensus error within each cluster.}

\begin{condition} \label{paraDiv}
Let $\epsilon_c^{(t)}$ be an upper bound on the average of the consensus error inside cluster $c$ at time $t$, i.e.,
%for $t \in \mathcal{T}_k$, $\forall k$:
    \begin{align}
        \frac{1}{s_c}\sum\limits_{i\in \mathcal S_c}\Vert \mathbf{e}_i^{(t)}\Vert^2 \leq (\epsilon_c^{(t)})^2.
    \end{align}
    We further define $(\epsilon^{(t)})^2=\sum\limits_{c=1}^{N}\rho_c(\epsilon_c^{(t)})^2$ as the average of these upper bounds over the network at time $t$.
    % \chris{Is $\epsilon(k)$ another definition? It sounds weird to say ``where'' since it does not appear in (13).}
\end{condition}
{
In fact,
using Lemma~\ref{lemma:cons}, this condition can be satisfied by tuning the number of consensus steps.
In our analysis, we will derive conditions  on $\epsilon_c^{(t)}$ that are sufficient to guarantee convergence of {\tt TT-HF} (see Proposition~\ref{Local_disperseT}).}
%Considering this definition with Lemma~\ref{lemma:cons}, the values of $\epsilon_c^{(t)}$ and $\epsilon^{(t)}$ will take smaller values when more rounds of D2D communications (i.e., larger $\Gamma_c$) are performed inside the clusters, as we would intuitively expect.

% In Definition \ref{paraDiv}, we characterize the consensus error in each cluster $c$ by $\epsilon_c(k)$ and the average of consensus error across clusters in the network by $\epsilon(k)$. 

We next define the expected variance in models across clusters at a given time, which we refer to as \textit{model dispersion}:

\begin{definition} \label{modDisp}
We define the expected model dispersion across the clusters at time $t$ as
\begin{align}\label{eq:defA}
A^{(t)} = \mathbb E\left[\sum\limits_{c=1}^N\varrho_{c}\big\Vert\bar{\mathbf w}_c^{(t)}-\bar{\mathbf w}^{(t)}\big\Vert^2\right],
\end{align}
where $\bar{\mathbf w}_c^{(t)}$ is defined in~\eqref{eq14} and $\bar{\mathbf w}^{(t)} = \sum\limits_{c=1}^{N}\varrho_c\bar{\mathbf w}_c^{(t)}$ is the global average of the local models at time $t$.
\end{definition}

$A^{(t)}$ measures the degree to which the cluster models deviate from their average throughout the training process. Obtaining an upper bound on this quantity is non-trivial due to the coupling between the gradient diversity and the model parameters imposed by~\eqref{eq:11}. For an appropriate choice of step size in~\eqref{8}, we upper bound this quantity at time $t$ through a set of new techniques that include the mathematics of \textit{coupled dynamic systems}. Specifically, we have the following result:

\begin{proposition}\label{Local_disperseT}
    If $\eta_t=\frac{\gamma}{t+\alpha}$ {for some $\gamma>0$}, $\epsilon^{(t)}$ is non-increasing for $t\in \mathcal T_k$, i.e., $\epsilon^{(t+1)} \leq \epsilon^{(t)}$, and $\alpha\geq
    \gamma\beta\max\{
     \lambda_+-2+\frac{\mu}{2\beta},\frac{\beta}{\mu}\}$, then % {\tt TT-HF} obtains the following upper bound on the expected model dispersion across clusters:
    %  \nm{$\gamma$ has not been defined}
        \begin{align} \label{eq:At}
            A^{(t)}\leq&
             \frac{16\omega^2}{\mu}(\Sigma_{+,t})^2 [F(\bar{\mathbf w}(t_{k-1}))-F(\mathbf w^*)]
             \nonumber \\&
            +25(\Sigma_{+,t})^2
            \left(\frac{\sigma^2+\delta^2}{\beta^2}+(\epsilon^{(0)})^2\right), ~~t\in\mathcal{T}_k,
        \end{align}
        where 
        \begin{equation*}
          \hspace{-3mm}\resizebox{.99\linewidth}{!}{$
            \Sigma_{+,t}
            % \nonumber \\&
            =\sum\limits_{\ell=t_{k-1}}^{t-1}\left(\prod_{j=t_{k-1}}^{\ell-1}(1+\eta_j\beta\lambda_+)\right)\beta\eta_\ell\left(\prod_{j=\ell+1}^{t-1}(1+\eta_j\beta)\right),
            $} \hspace{-3mm}
        \end{equation*}
    and
    % \begin{equation*}
    $
        \lambda_+=1-\frac{\mu}{4\beta}+\sqrt{(1+\frac{\mu}{4\beta})^2+2\omega}.
    $
    % \end{equation*}
\end{proposition}
\begin{proof}
See Appendix~\ref{app:Local_disperse}.
\end{proof}
% \frank{I will change it to unscaled notation once the content is confirmed}

% \nm{move to appendix? Just provide intuition here of the result!}
The bound in~\eqref{eq:At} demonstrates how the expected model dispersion across the clusters increases with respect to the consensus error ($\epsilon^{(0)}$), the noise in the gradient estimation ($\sigma^2$), and the heterogeneity of local datasets ($\delta,\omega$). 
{Intuitively, the upper bound in \eqref{eq:At} dictates that, the larger $\epsilon^{(0)}$, $\sigma^2$, $\delta$ or $\omega$, the larger the dispersion, due to error propagation in the network. Proposition \ref{Local_disperseT} will be an instrumental result in the convergence proof developed in the next section.}

\subsection{General Convergence Behavior of $\hat{\mathbf{w}}^{(t)}$}
\label{ssec:convAvg}
Next, we focus on the convergence of the global loss. In the following theorem, we bound the expected distance that the global loss is from the optimal over time, as a function of the model dispersion.

\begin{theorem} \label{co1}
% \nm{do we need any assumption on epsilon t, similarly to Prop 1?}
        When using {\tt TT-HF} for ML model training with $\eta_t \leq 1/\beta$ $\forall t$,
        % $\epsilon^{(t)}=\eta_t\phi,~\forallt$
  the one-step behavior of {the global model} $\hat{\mathbf w}^{(t)}$ {(see \eqref{15})} satisfies, for $t\in\mathcal{T}_k$,
% https://www.overleaf.com/project/5f7c9b5ce460a000011dc1e1
\begin{align} \label{eq:th1mainRes}
       &\mathbb E\left[F(\hat{\mathbf w}^{(t+1)})-F(\mathbf w^*)\right]
        \leq
        \underbrace{(1-\mu\eta_{t})\mathbb E[F(\hat{\mathbf w}^{(t)})-F(\mathbf w^*)]}_{{(a)}}
        \nonumber \\&
        +\underbrace{\frac{\eta_{t}\beta^2}{2}A^{(t)}
        +\frac{1}{2}[\eta_{t}\beta^2(\epsilon^{(t)})^2+\eta_{t}^2\beta\sigma^2+\beta(\epsilon^{(t+1)})^2]}_{(b)},
\end{align}
where $A^{(t)}$ is the model dispersion from Definition~\ref{modDisp}.
% \nm{You have already defined this.. just say where At is defined in Def 4.}
% \begin{align}
%      A^{(t)}\triangleq\mathbb E\left[\sum\limits_{c=1}^N\varrho_{c}\Vert\bar{\mathbf w}_c^{(t)}-\bar{\mathbf w}^{(t)}\Vert_2^2\right].
% \end{align}
\end{theorem}
\begin{proof}
See Appendix~\ref{app:thm1}.
\end{proof} 
 
% Theorem \ref{co1} illustrates how the global parameter $\hat{\mathbf w}(t_k)$ converges to the optimal $\mathbf w^*$ with respect to each global update in {\tt TT-HF} and how the learning and system parameters affects this convergence. 

% \nm{We need more intuition here!!}
Theorem~\ref{co1} quantifies the dynamics of the global model relative to the optimal model during a given update period $\mathcal{T}_k$ of {\tt TT-HF}. {Since the theorem holds for all $t\in\mathcal T_k$, it also quantifies the suboptimality gap when global aggregation is performed at time $t+1=t_k$.}
{Note that the term (a) corresponds to the one-step progress of a \emph{centralized} gradient descent under strongly-convex global loss (Assumption~\ref{Assump:SmoothStrong}), so that the term (b) quantifies the additional loss incurred as a result of the model dispersion across the clusters ($A^{(t)}$, which in turn is bounded by Proposition~\ref{Local_disperseT}), consensus errors ($\epsilon^{(t)}$), and SGD noise ($\sigma^2$).}
% \nm{rephrased and provided more intuition:}
%Considering the sequence $\left\{\mathbb E\left[F(\hat{\mathbf w}^{(t+1)})-F(\mathbf w^*)\right]\right\}_{t=1}^{\infty}$ in~\eqref{eq:th1mainRes}, we see that the convergence of the upper bound depends on several factors: (i) the characteristics of the loss function (i.e., $\mu, \beta$); (ii) the step size (i.e., $\eta_t$); (iii) the expected model dispersion across the clusters ($A^{(t)}$, which in turn is bounded by Proposition~\ref{Local_disperseT}); and (iv) the SGD noise ($\sigma^2$).
In fact, without careful choice of our control parameters, the sequence $\mathbb E[F(\hat{\mathbf w}^{(t)})-F(\mathbf w^*)]$ may diverge. Thus, we are motivated to find conditions under which convergence is guaranteed, and furthermore, under which the upper bound in~\eqref{eq:th1mainRes} will approach zero.

Specifically, we aim for {\tt TT-HF} to match the asymptotic convergence behavior of centralized stochastic gradient descent (SGD) under a diminishing step size, which is $\mathcal{O}(1/t)$ \cite{Bubeck}. From~\eqref{eq:th1mainRes}, we see that to match SGD, the terms in $(b)$ should be of order $\mathcal{O}(\eta_t^2)$, i.e., the same as the
{degradation due to the SGD noise, $\eta_{t}^2\beta\sigma^2/2$.}
 %SGD noise $\sigma^2$.
This implies that {control parameters need to be tuned in such a  way that} $A^{(t)}{=} \mathcal{O}(\eta_t)$ and $\epsilon^{(t)}{=}\mathcal{O}(\eta_t)$.
{Proving that these conditions hold under proper choice of parameters will be part of Theorem~\ref{thm:subLin}.}

\subsection{Sublinear Convergence Rate of $\hat{\mathbf{w}}^{(t)}$}
\label{ssec:sublinear}
Among the quantities involved in Theorem~\ref{co1}, $\eta_t, \tau_k$ and $\epsilon^{(t)}$ are the three tunable parameters that directly impact the learning performance of {\tt TT-HF}. We now prove that with proper choice of these parameters, {\tt TT-HF} enjoys sub-linear convergence with rate of $\mathcal{O}(1/t)$. %, which coincides with the convergence rate of SGD in centralized model training.

\begin{theorem} \label{thm:subLin_m}
Under Assumptions \ref{beta},~\ref{assump:cons}, and~\ref{assump:SGD_noise},
 suppose
 {$\eta_t=\frac{\gamma}{t+\alpha}$ and $\epsilon^{(t)}=\eta_t\phi$, where}
 $\gamma>1/\mu$,
{$\phi>0$, $\alpha\geq\alpha_{\min}$}
 and $\omega< \omega_{\max}(\alpha)$. Then, by using {\tt TT-HF} for ML model training,
    % Using {\tt TT-HF} for ML model training, under Assumption \ref{beta}, if we set the step size as $\eta_t=\frac{\gamma}{t+\alpha}$, $\forall t$, and assuming that $\epsilon(t)=\eta_t\phi$, $\forall t$,
    % we have
    \begin{align} \label{eq:thm2_result-1-T_m}
        &\mathbb E\left[F(\hat{\mathbf w}^{(t)})-F(\mathbf w^*)\right]\leq\frac{\nu}{t+\alpha}, ~~\forall t,
        \end{align}
%        \add{for some $\nu(\alpha)\geq 0$ (see \eqref{}\nm{..}).}
        where $\tau = \max_{1\leq \ell \leq k} \{\tau_\ell\}$,
        \begin{align}
            &\hspace{-2mm}\alpha_{\min}\triangleq
    \gamma\beta\max\Big\{
\frac{\mu}{4\beta}-1+\sqrt{(1+\frac{\mu}{4\beta})^2+2\omega},\frac{\beta}{\mu}\Big\},\hspace{-4mm}
\\ & \omega_{\max}(\alpha)\triangleq\frac{1}{\beta\gamma}\sqrt{\frac{\alpha}{Z_1}}
\sqrt{\mu\gamma-1+\frac{1}{1+\alpha}}, \label{eq:OmegaMaxTh2}
\\&\hspace{-5mm}
\nu{\triangleq} \max \left\{\frac{\beta^2\gamma^2Z_2}{\mu\gamma-1},
        \frac{\alpha Z_2/Z_1}{\omega_{\max}^2{-}\omega^2},\alpha\left[F(\hat{\mathbf w}^{(0)}){-}F(\mathbf w^*)\right]\right\},\label{eq:nuTh2}\!\!\!\\
        Z_1 &= \frac{32\beta^2\gamma}{\mu}(\tau-1)\left(1+\frac{\tau}{\alpha-1}\right)^{2}\left(1+\frac{\tau-1}{\alpha-1}\right)^{6\beta\gamma}\label{eq:z_1Th2}\\
        &Z_2 = \frac{\sigma^2+2\phi^2}{2\beta}\nonumber
                +50\gamma(\tau-1)\left(1+\frac{\tau-2}{\alpha+1}\right)\\&
                \qquad\times\left(1+\frac{\tau-1}{\alpha-1}\right)^{6\beta\gamma}\left(\sigma^2+\phi^2+\delta^2\right).\label{eq:z_2Th2}
        \end{align}
        % \nm{pleaea simplify ther expression! Dont you see that the betas cancel out??? DONE MYSELF}
\end{theorem}
\begin{proof}
See Appendix~\ref{app:subLin}.
\end{proof}
\vspace{0.05in}

Theorem~\ref{thm:subLin_m} is one of the central contributions of this paper, revealing how several parameters (some controllable and others characteristic of the environment) affect the convergence of {\tt TT-HF}, and conditions under which $\mathcal{O}(1/t)$ convergence can be achieved. We make several key observations. First,
{to achieve $\mathcal{O}(1/t)$ convergence, the gradient diversity parameter $\omega=\frac{\zeta}{2\beta}$ should not be too large ($\omega< \omega_{\max}(\alpha)$); in fact, $\omega$ induces error propagation of order $\sim\Vert\mathbf w_c-\mathbf w^*\Vert$, so that too large values of $\omega$ may cause the error to diverge.
Since $\omega_{\max}(\alpha)$ is an increasing  function of $\alpha$ (see~\eqref{eq:OmegaMaxTh2}), larger values of $\omega$ may be tolerated by increasing $\alpha$, i.e., by using a smaller step-size $\eta_t$, confirming the intuition that larger gradient diversity requires a smaller step-size for convergence. However, the penalty incurred may be slower convergence of the suboptimality gap (since $\nu$ increases with $\alpha$, see~\eqref{eq:nuTh2}).
}
% \nm{rephrased:}
%the bound obtained in~\eqref{eq:thm2_result-1-T}\nm{?? should be 21} \hl{shows how any observed gradient diversity ($\omega$) can be handled with a proper choice of step size ($\eta_t$).}\nm{this is not quite clear..Im repharsing it..} In particular, it shows that $\mathcal{O}(1/t)$ convergence is achieved when $\omega=\frac{\zeta}{2\beta} <\omega_{max}$, where $\omega_{max}$ is a function of $\gamma$ and $\alpha$. Given the fact that $\omega_{max}$ is increasing in $\alpha$, this implies that larger values of $\zeta$ can be handled via larger values of $\alpha$. This conforms to the intuition that larger gradient diversity requires a smaller step size for convergence.

% \nm{moved from Subsection B:}
We now discuss the choice of the consensus error $\epsilon^{(t)}$. To guarantee $\mathcal O(1/t)$ convergence, Theorem~\ref{thm:subLin_m} dictates that it should be chosen as $\epsilon^{(t)}=\eta_t\phi$ for a constant $\phi>0$, i.e. it should decrease over time according to the step-size.
 %Additionally, since we can control the consensus error via the number of D2D communications rounds, it would be sufficient if we could choose $\epsilon^{(t)}=\eta_t\phi$, where $\phi$ is a positive constant. %Note that this can be achieved via limited D2D communications rounds.
% Note that achieving the bound in Theorem~\ref{co1} requires $\epsilon^{(t)} = \eta_t \phi$, i.e., that the average consensus error is on the order of the step size. 
To see that this is a {feasible and practical} condition, note from Lemma \ref{lemma:cons} that the upper bound of $\Vert \mathbf e_i^{(t)} \Vert$ {increases proportionally to the divergence $\Upsilon^{(t)}_c$ (see \eqref{eq:cons}), and decreases at geometric rate with the number of consensus steps. In turn, $\Upsilon^{(t)}_c$ can be shown to be of the order of the step-size $\eta_t$ (assuming $\eta_t \approx \eta_{t-1}$):
$$
\Upsilon^{(t)}_{{c}}=
\max_{j,j'\in\mathcal S_c}\Vert \widetilde{\mathbf{w}}^{(t)}_{{j}} -\widetilde{\mathbf{w}}^{(t)}_{{j'}}\Vert
\approx \eta_{t} \max_{j,j'\in\mathcal S_c}\big\Vert \widehat{\mathbf g}_{j}^{(t-1)}-\widehat{\mathbf g}_{j'}^{(t-1)}\big\Vert,
$$
where we used \eqref{8}, and the approximation holds if we assume the difference ${\mathbf{w}}^{(t-1)}_{{j}} - {\mathbf{w}}^{(t-1)}_{{j'}}$ in initial model parameters at $t-1$ is negligible compared to the gradients.
This is a legit assumption, since ${\mathbf{w}}^{(t-1)}_{{j}}$ is the model parameter at node $j$, \emph{after} the consensus rounds at time $t-1$.
Using Lemma~\ref{lemma:cons}, it then follows that, to make $\epsilon^{(t)}=\eta_t\phi$, the number of consensus rounds should be chosen such that $
(\lambda_{{c}})^{\Gamma^{(t)}_{{c}}}
\approx \frac{1}{\sqrt{s_c}}\phi/\max_{j,j'\in\mathcal S_c}\big\Vert \widehat{\mathbf g}_{j}^{(t-1)}-\widehat{\mathbf g}_{j'}^{(t-1)}\big\Vert$,
{and are thus dominated by the divergence of local gradients within the cluster and SGD noise, irrespective of the step-size.}
We will use this property in the development of our control algorithm for $\Gamma^{(t)}_c$ in Sec.~\ref{Sec:controlAlg}.
}
% \nm{rephraxsed}
 %is a linear function of the divergence $\Upsilon^{(t)}_c$ (see \eqref{eq:cons}). $\Upsilon^{(t)}_c$ in turn is on the order of $\eta_t$:
% $\Upsilon^{(t)}_c = {\arg\max}_{z, i,i'\in\mathcal{S}_c} \lvert (\widetilde{\mathbf{w}}^{(t)}_{{i}})_{z} - (\widetilde{\mathbf{w}}^{(t)}_{{i'}})_z \big\rvert \leq \big\Vert \widetilde{\mathbf{w}}^{(t)}_{{i}} -\widetilde{\mathbf{w}}^{(t)}_{{i'}}\big\Vert \approx \eta_t \big\Vert \widehat{\mathbf g}_{i}^{(t-1)}-\widehat{\mathbf g}_{i'}^{(t-1)}\big\Vert$ from \eqref{8} if we assume the difference ${\mathbf{w}}^{(t-1)}_{{i}} - {\mathbf{w}}^{(t-1)}_{{i'}}$ in initial model parameters at $t-1$ is negligible compared to the gradients. Further, from Lemma~\ref{lemma:cons}, we can control the value of $\epsilon^{(t)}$ through the rounds of consensus performed in each cluster $c$, as it is decreasing exponentially in $\Gamma^{(t)}_c$. We will use this property in the development of our control algorithm for $\Gamma^{(t)}_c$ in Sec.~\ref{Sec:controlAlg}.

%We will next build upon this logic to derive a sufficient set of conditions to obtain a convergence rate of $\mathcal{O}(1/t)$ for $\hat{\mathbf w}^{(t)}$.

The bound also shows the impact of the duration of local model training intervals $\tau$ on the convergence, through the term $\nu$ in~\eqref{eq:thm2_result-1-T_m}.
%(i.e., $\tau$ embedded in terms $Z_1$ and $Z_2$) on the coefficient of the instantaneous upper bound of convergence (i.e., $\nu$).
In particular, from~\eqref{eq:nuTh2},
% \nm{you should reference nu here!}
it can be seen that increasing $\tau$ results in a sharp increase of $\nu$ (through the factors $Z_1$ and $Z_2$ defined in \eqref{eq:z_1Th2} and~\eqref{eq:z_2Th2}). Moreover, we also observe a quadratic impact on $\nu$ with respect to the consensus error $\epsilon^{(t)}$ (through $\phi$). It then follows that, all else constant, increasing the value of $\tau$ requires a smaller value of $\phi$ (i.e., more accurate consensus) to achieve a desired value of $\nu$ in~\eqref{eq:nuTh2}. This is consistent with how {\tt TT-HF} is designed, since the motivation for including consensus rounds (to decrease $\epsilon^{(t)}$) is to reduce the global aggregation frequency, which results in uplink bandwidth utilization and power consumption savings.

% An immediate result of Theorem~\ref{thm:subLin} is that for a given (attainable) value of $\tilde{\Gamma}\geq\gamma^2\beta^2(A+B)/\Big[\mu\gamma-[3+8(1/\vartheta+(2\omega-\vartheta/2))^2]\Big]+\vartheta\gamma\phi^2/2$, we can guarantee the sublinear convergence condition in~\eqref{eq:thm2_result-1} by choosing $\tau$ such that
%         \begin{align} \label{eq:thm2_case2-1}
%              1 \leq \tau \leq \sqrt{\Big[\mu\gamma-[3+8(1/\vartheta+(2\omega-\vartheta/2))^2]\Big]\Big[\tilde{\Gamma}-\vartheta\gamma\phi^2/2\Big]/(\gamma^2\beta^2 B)-A/B}.
%         \end{align} 
% From~\eqref{eq:thm2_case2-1}, for a given setting of $\tilde{\Gamma}$, the local model training interval periods are always bounded. Furthermore, from~\eqref{eq:thm2_case2-1}, larger gradient diversity (i.e., from $\omega$, and from $\delta$ through the $A$ term) lowers the upper range on $\tau$: as the local datasets exhibit more statistical heterogeneity, we need shorter local model training intervals with more frequent global aggregations to obtain the sublinear convergence characteristic. The same phenomena can be observed considering the impact of the consensus error (i.e., through $\phi$) and of the SGD noise (i.e., through $\sigma^2$ from the $A$ and $B$ terms) on~\eqref{eq:thm2_case2-1}: as these errors increase, we need more frequent global aggregations. In the case of $\epsilon^{(t)}$ this is consistent with how {\tt TT-HF} is architected, since the motivation for including consensus rounds (which decreases $\epsilon^{(t)}$) is to reduce the global aggregation frequency.

{These observations reveal a trade-off between accuracy, delay, and energy consumption.}
In the next section, we leverage these relationships in developing an adaptive algorithm for {\tt TT-HF} that tunes the control parameters to achieve the convergence bound in Theorem~\ref{thm:subLin_m} while minimizing network costs.

\section{Adaptive Control Algorithm for {\tt TT-HF}}\label{Sec:controlAlg}
\noindent There are three parameters in {\tt TT-HF} that can be tuned over time: (i) local model training intervals $\tau_k$, (ii) gradient descent step size $\eta_t$, and (iii) rounds of D2D communications $\Gamma_c^{(t)}$. In this section, we develop a control algorithm (Sec.~\ref{ssec:control}) based on Theorem~\ref{thm:subLin_m} for tuning (i), (ii) at the main server at the beginning of each global aggregation, and (iii) at each device cluster in a decentralized manner. To do so, we propose an approach for determining the learning-related parameters (Sec.~\ref{subsec:learniParam}), a resource-performance tradeoff optimization for $\tau_k$ and $\Gamma_c^{(t)}$ (Sec.~\ref{subsec:learnduration}), and estimation procedures for dataset-related parameters (Sec.~\ref{subsec:cont}).

\subsection{Learning-Related Parameters ($\alpha, \gamma, \phi$, $\eta_t$)}\label{subsec:learniParam}
We aim to tune the step size-related parameters ($\alpha, \gamma$) and the consensus error coefficient ($\phi$) to satisfy the conditions in Theorem~\ref{thm:subLin_m}. In this section, we present a method for doing so  given properties of the ML model, local datasets, and SGD noise ($\beta, \mu, \zeta, \delta, \sigma$, and thus $\omega=\zeta/(2\beta)$). Later in Sec.~\ref{subsec:cont}, we will develop methods for estimating $\zeta, \delta, \sigma$ at the server.\footnote{We assume that $\beta$ and $\mu$ can be computed at the server prior to training given the knowledge of the deployed ML model.} We assume that the latency-sensitivity of the learning application specifies a tolerable amount of time that {\tt TT-HF} can wait between consecutive global aggregations, i.e., the value of $\tau$. 

To tune the step size parameters, first, a value of $\gamma$ is determined such that $\gamma>1/\mu$. Then, since smaller values of $\alpha$ are associated with faster convergence, the minimum value of $\alpha$ that simultaneously satisfies the conditions in the statement of Theorem~\ref{thm:subLin_m} is chosen, i.e., $\alpha\geq\alpha_{\min}$ and $\omega_{\max}>\omega$ (note that $\omega_{\max}$ is a function of $\alpha$, see~\eqref{eq:OmegaMaxTh2}.
% \nm{no need to take space to repeat these equations,}
% \begin{align} \label{eq:alphaCond}
%       \hspace{-3mm}\nm{REMOVE:}  \begin{cases}
%             \alpha >\max\{1,\beta\gamma[\frac{\mu}{4\beta}-1+\sqrt{(1+\frac{\mu}{4\beta})^2+2\omega}],\frac{\beta^2\gamma}{\mu}\},& \\
%             \frac{1}{\beta\gamma}\sqrt{\alpha\frac{\mu\gamma-1+\frac{1}{1+\alpha}}{Z_1}} \geq \omega.  & 
%         \end{cases}
%         \hspace{-4mm}
%     \end{align}

Let $T$ be a (maximum) desirable duration of the entire {\tt TT-HF} algorithm, and $\xi$ be a (maximum) desirable loss at the end of the model training, which may be chosen based on the learning application. 
To satisfy the loss requirement, from Theorem~\ref{thm:subLin_m} the following condition needs to be satisfied,
%The inequality from~\eqref{eq:thm2_result-1-T}\nm{you keep referring to the wrong eq! should be 21? Please check all instances..} must hold:
\begin{equation}\label{42}
    \frac{\nu}{T+\alpha} \leq \xi,
\end{equation} 
yielding a maximum value tolerated for $\nu$,
 %Subsequently, the maximum value of $\nu$ that satisfies the above inequality is obtained,
i.e., $\nu^{\mathsf{max}}=\xi(T+\alpha)$. 
Since $\nu$ is a function of the local model training period $\tau$ and consensus coefficient $\phi$ (see~\eqref{eq:nuTh2}), this bound places a condition on the parameters $\tau$ and $\phi$.
% \nm{a bit unclear, rephrased above, does that look ok?}%This results in the loosest potential condition on the maximum period of local model training ($\tau$) and the consensus coefficient ($\phi$).
Furthermore, with the values of $\alpha$ and $\gamma$ chosen above, along with the value of $\tau$, the algorithm may not always be able to provide any arbitrary desired loss $\xi$ at time $T$. Therefore, 
considering the expression for $\nu$ from Theorem~\ref{thm:subLin_m}, the following feasibility check is conducted:
\begin{equation}\label{46}
\hspace{-4.1mm}\resizebox{.93\linewidth}{!}{$
\max\hspace{.51mm} \left\{\frac{\beta^2\gamma^2Z^{\mathsf{min}}_2}{\mu\gamma-1},
\frac{\alpha Z^{\mathsf{min}}_2/Z_1}{\omega_{\max}^2-\omega^2},\frac{\alpha\Vert\nabla F(\hat{\mathbf w}^{(0)})\Vert^2}{2\mu }\right\}\hspace{-.7mm} \leq \nu^{\mathsf{max}}\hspace{-.15mm},$}\hspace{-4mm}
\end{equation} 
where

% \nm{please dont change fontsize of equations.. it does not look good when you keep changing fontsize. Rather, add a new line if needed. or simplify a bit the equations: 
% $\Big(\frac{\sigma^2}{\beta}+\frac{\delta^2}{\beta}\Big)\mapsto\frac{\sigma^2+\delta^2}{\beta}$ which actually cancels with the other 50beta. Btw, why are equations not simplified???
% }
\vspace{-0.2in}
{ \begin{equation}
Z^{\mathsf{min}}_2 =
        % \nonumber \\&
        \frac{\sigma^2}{2\beta}
        +50\gamma(\tau-1)\Big(1+\frac{\tau-2}{\alpha+1}\Big)
        \Big(1+\frac{\tau-1}{\alpha-1}\Big)^{6\beta\gamma}\big(\sigma^2+\delta^2\big) \nonumber
\end{equation} }
% \nm{why are equations not simplified? the beta in the denom cancels the beta in the numerator!}

\vspace{-0.1in}
\noindent is the value of $Z_2$  obtained by setting the consensus coefficient $\phi=0$ in~\eqref{eq:z_2Th2}. The third term inside the $\max$ of \eqref{46} is obtained via the Polyak-Lojasiewicz 
% \nm{no need to define acronym unless oyu use it somewhere else} 
inequality $\left\Vert\nabla F(\hat{\mathbf w}^{(t)})\right\Vert^2 \geq 2\mu [F(\hat{\mathbf w}^{(t)})-F(\mathbf w^*)]$, since the value of $F(\mathbf w^*)$ is not known, {whereas $\nabla F(\hat{\mathbf w}^{(t)})$} can be estimated using the local gradient of the sampled devices at the server.
% \nm{short clarification on how $\nabla F(\hat{\mathbf w}^{(t)})$ can be computed?}
If~\eqref{46} is not satisfied, the chosen values of $\tau$, $\xi$ and/or $T$ must be loosened, and this procedure must be repeated until \eqref{46} becomes feasible.

Once $\alpha, \gamma$ and $\tau$ are chosen, we move to selecting $\phi$. All else constant, larger consensus errors would be more favorable in {\tt TT-HF} due to requiring less rounds of D2D communications (Lemma~\ref{lemma:cons}). The largest possible value of $\phi$, denoted $\phi^{\mathsf{max}}$, can be obtained directly from~\eqref{46} via replacing $Z_2^{\mathsf{min}}$ with $Z_2$ and considering the definition of $Z_2$ in~\eqref{eq:z_2Th2}:\footnote{In the $\max$ function in~\eqref{47}, only the first two arguments from the function in~\eqref{46} are present as the third is independent of $Z_2$ and $\phi$.}
% \nm{isnt $\mu\gamma-1$ always negative? (you have max of that and smt else)}
% \begin{align} \label{47}
% \hspace{-4mm}
%     \phi^{\mathsf{max}}=\sqrt{\frac{\frac{\beta\nu^{\mathsf{max}}}{\max \left\{\frac{\beta^2\gamma^2}{\mu\gamma-1},
% \frac{\alpha}{Z_1\left(\omega_{\max}^2-\omega^2\right)}\right\}}-\beta Z_2^{\mathsf {min}}}{1+50\beta\gamma(\tau-1)\left(1+\frac{\tau-2}{\alpha+1}\right)
%     \left(1+\frac{\tau-1}{\alpha-1}\right)^{6\beta\gamma}}}.\hspace{-4mm}
% \end{align}
% \nm{this is barely readable: rewritten as}
\begin{align} \label{47}
\hspace{-4mm}
    \phi^{\mathsf{max}}\hspace{-.5mm}=\hspace{-1mm}\sqrt{\beta}\hspace{-1mm}
    \sqrt{\hspace{-1mm}\frac{\nu^{\mathsf{max}}
    \min \left\{
    \frac{\mu\gamma-1}{\beta^2\gamma^2},
\frac{Z_1\left(\omega_{\max}^2-\omega^2\right)}{\alpha}\right\}-Z_2^{\mathsf {min}}}{1+50\beta\gamma(\tau-1)\left(1+\frac{\tau-2}{\alpha+1}\right)
    \left(1+\frac{\tau-1}{\alpha-1}\right)^{6\beta\gamma}}}.\hspace{-4mm}
\end{align}
Note that \eqref{47} exists if the feasibility check in \eqref{46} is satisfied.

The values of $\nu^{\mathsf{max}}$ and $\alpha$ are re-computed at the server at each global aggregation. The devices use this to set their step sizes $\eta_t$ during the next local update period accordingly.

\subsection{Local Training Periods ($\tau_k$) and Consensus Rounds ($\Gamma^{(t)}_c$)}\label{subsec:learnduration}
One of the main motivations behind {\tt TT-HF} is minimizing the resource consumption among edge devices during model training. We thus propose tuning the $\tau_k$ and $\Gamma^{(t)}_c$ parameters according to the joint impact of three metrics: energy consumption, training delay imposed by consensus, and trained model performance. To capture this {trade-off}, we formulate an optimization problem $(\bm{\mathcal{P}})$ solved by the main server at the beginning of each global aggregation period $\mathcal{T}_k$, i.e., when $t=t_{k-1}$:
\begin{align*} 
    &(\bm{\mathcal{P}}): ~~\min_{\tau_k}  \underbrace{\frac{c_1\Big(E_{\textrm{Glob}}+\sum\limits_{t=t_{k-1}}^{t_{k-1}+\tau_k}\sum\limits_{c=1}^{N} \Gamma_c^{(t)}s_c E_{\textrm{D2D}}\Big)}{\tau_k}}_{(a)}+
    \nonumber \\&
    \!\!\!\! \underbrace{\frac{c_2\Big(\Delta_{\textrm{Glob}}+\sum\limits_{t=t_{k-1}}^{t_{k-1}+\tau_k}\sum\limits_{c=1}^{N} \Gamma_c^{(t)}\Delta_{\textrm{D2D}}\Big)}{{\tau_k}}}_{(b)}
   + c_3\Big(\underbrace{1-\frac{t_{k-1}+\alpha}{t_{k-1}+\tau_k+\alpha}}_{(c)}\Big)
\end{align*}
\vspace{-0.3in}
    \begin{align}
    & \textrm{s.t.}\nonumber \\
   & \;\;\; \Gamma_c^{(t)}=\max\Big\{\Big\lceil\log\Big(\frac{\eta_t\phi}{\sqrt{s_c }\Upsilon_c^{(t)}}\Big)/\log\Big(\lambda_c\Big)\Big\rceil,0\Big\},\forall c, \hspace{-3mm} \label{eq:consensusNum}\\ 
   & \;\;\; 1 \leq \tau_k \leq \min{\{\tau, T-t_{k-1}\}}, \tau_k\in\mathbb{Z}^+, \label{eq:tauMax}\\ 
    % &K=\floor{\frac{T-t_{k-1}}{\tau_k}} \\
    & \;\;\; \Upsilon_c^{(t_{k-1})}=0,~\forall c, \label{eq:Up_init}\\ 
    & \;\;\; \Upsilon_c^{(t)} = \mathbbm{1}_{\{\Gamma^{(t-1)}_c =0\}} (\underbrace{A_c^{(k)}\Upsilon_c^{(t-1)}+B_c^{(k)}}_{(d)}) \; + \nonumber\\& 
    ~~~~~~~~~~~\left(1-\mathbbm{1}_{\{\Gamma^{(t-1)}_c =0\}}\right) (\underbrace{a_c^{(k)}\Upsilon_c^{(t-1)}+b_c^{(k)}}_{(e)}),~\forall c, \label{eq:Up_dyn}
    % \\
%   & \tilde{\chi}_c^{(t)}= \left(\lambda^{(k)}_{{c}}\right)^{\Gamma_c^{(t)}} s_c^{(k)}{\Upsilon^{(t-1)}_{{c}}}
\end{align}
where $E_{\textrm{D2D}}$ is the energy consumption of each D2D communication round for each device, $E_{\textrm{Glob}}$ is the energy consumption for device-to-server communications, $\Delta_{\textrm{D2D}}$ is the communication delay per D2D round conducted in parallel among the devices, and $\Delta_{\textrm{Glob}}$ is the device-to-server communication delay.
The objective function captures the trade-off between average energy consumption (term $(a)$), average D2D delay (term $(b)$), and expected ML model performance (term $(c)$). In particular, term $(c)$ is a penalty on the ratio of the upper bound given in~\eqref{eq:thm2_result-1-T_m}
% \nm{please check you references! This is the wrong one! Do you mean 21?}
between the updated model and the previous model at the main server. A larger ratio implies the difference in performance between the aggregations is smaller, and thus that synchronization is occurring frequently, consistent with $\tau_k$ appearing in the denominator. This term also contains a diminishing marginal return from global aggregations as the learning proceeds: smaller values of $\tau_k$ are more favorable in the initial stages of ML model training, i.e., for smaller $t_{k-1}$. This matches well with the intuition that ML model performance has a sharper increase at the beginning of model training, so frequent aggregations at smaller $t_{k-1}$ will have larger benefit to the model performance stored at the main server. The coefficients $c_1,c_2,c_3 \geq 0$ are introduced to weigh each of the design considerations.
% $\Gamma_c^{(t)}$ is the rounds of consensus performed at cluster $c$ at time $t$,

The equality constraint on $\Gamma^{(t)}_c$ in~\eqref{eq:consensusNum} forces the condition $\epsilon^{(t)}=\eta_t\phi$ imposed by Theorem~\ref{thm:subLin_m}, obtained using the result in Lemma~\ref{lemma:cons}. This equality reveals the condition under which the local aggregations, i.e., D2D communication, are triggered. Note that since the spectral radius is less than one, we have $\log(\lambda_c)<0$, thus the requirement to conduct D2D communications, i.e., triggering in cluster model synchronization, is $\sqrt{s_c}\Upsilon_c^{(t)}>\eta\phi$. In other words, when the divergence of local models exceeds a predefined threshold $\Upsilon_c^{(t)}>\frac{\eta\phi}{\sqrt{s_c}}$, local synchronization is triggered via D2D communication, and the number of D2D rounds is given by $\Gamma_c^{(t)}$. Also, ~\eqref{eq:tauMax} ensures the feasible ranges for $\tau_k$.\label{parg:in-cluster}

%, where $\tau$ was determined from the learning application in Sec.~\ref{subsec:learniParam}.

As can be seen from~\eqref{eq:consensusNum}, to obtain the desired consensus rounds for future times $t \in \mathcal{T}_k$, the values of $\Upsilon_c^{(t)}$ -- the divergence of model parameters in each cluster -- are needed. Obtaining these exact values at $t = t_{k-1}$ is not possible since it requires the knowledge of the model parameters $\widetilde{\mathbf{w}}_i^{(t)}$ of the devices for the future timesteps, which is non-causal. To address this {challenge, problem $(\bm{\mathcal{P}})$ incorporates the additional constraints \eqref{eq:Up_init} and \eqref{eq:Up_dyn}, which aim to} estimate the future values of $\Upsilon^{(t)}_c$, $\forall c$ through a time-series predictor, initialized as $\Upsilon_c^{(t_{k-1})}=0$ in \eqref{eq:Up_init} (since, at the beginning of the period, the nodes start with the same model provided by the server).
 %described in~\eqref{eq:Up_init},\eqref{eq:Up_dyn}.
 In the expression \eqref{eq:Up_dyn}, $\mathbbm{1}_{\{\Gamma^{(t-1)}_c =0\}}$ takes the value of $1$ when no D2D communication rounds are performed at $t-1$, and $0$ otherwise. Two linear terms ($(d)$ and $(e)$) are included, one for each of these cases, characterized by coefficients $A_c^{(k)},B_c^{(k)},a_c^{(k)},b_c^{(k)}\in\mathbb{R}$ which vary across clusters and global aggregations. These coefficients are estimated through fitting the linear functions to the values of $\Upsilon_c^{(t)}$ obtained from the previous global aggregation $\mathcal{T}_{k-1}$. These values of $\Upsilon_c^{(t)}$ from $\mathcal{T}_{k-1}$ are in turn estimated in a distributed manner through a method presented in Sec.~\ref{subsec:cont}. 

%(this parameter is estimated in distributed manner in Section~\ref{subsec:cont}), where two linear functions, i.e., $(d)$ and $(e)$, are fitted to the data to capture the variations when D2D communications are performed/not-performed.
%Finally,~\eqref{eq:Up_init} and~\eqref{eq:Up_dyn} are used to predict the future values of divergence of intermediate updated local model parameters (see Definition~\ref{def:clustDiv}) through a successive linear approximation explained below.

Note that $(\bm{\mathcal{P}})$ is a non-convex and integer optimization problem. Given the parameters in Sec.~\ref{subsec:learniParam}, the solution for $\tau_k$ can be obtained via a line search over the integer values in the range of $\tau_k$ given in~\eqref{eq:tauMax}. Solving our optimization problem involves two steps: (i) linear regression of the constants used in (53), i.e., $A_c^{(k)},B_c^{(k)},a_c^{(k)},b_c^{(k)}$ using the history of observations, and (ii) line search over the feasible integer values for $\tau_k$. The complexity of part (i) is $\mathcal O(\tau_{k-1})$, since the dimension of each observant, i.e., $\Upsilon_c^{(t)}$, is one and the observations are obtained via looking back into the previous global aggregation interval.   Also, the complexity of (ii) is $\mathcal O(\tau_{\max})$ since it is just an exhaustive search 
{over} the range of $\tau\leq\tau_{\max}$, where $\tau_{\max}$ is the maximum tolerable interval that satisfies the feasibility conditions in Sec.~\ref{subsec:learniParam}.
While the optimization produces predictions of $\Gamma^{(t)}_c$ for $t \in \mathcal{T}_k$ through~\eqref{eq:consensusNum}, the devices will later compute~\eqref{eq:consensusNum} at time $t$ when the real-time estimates of $\Upsilon^{(t)}_c$ can be made through~\eqref{eq:Ups_est}, as will be discussed next.

\subsection{Data and Model-Related Parameters ($\delta, \zeta, \sigma^2, \Upsilon^{(t)}_c$)} \label{subsec:cont}

We also need techniques for estimating the gradient diversity ($\delta$, $\zeta$), SGD noise ($\sigma^2$), and cluster parameter divergence ($\Upsilon^{(t)}_c$).

%We first overview the estimation method of the ML parameters and the divergence of local models inside the clusters, using which we develop our control algorithm for {\tt TT-HF} that dynamically tunes the number of D2D communication rounds and obtains the length of interval of local model training.

\subsubsection{Estimation of $\delta,\zeta,\sigma^2$}\label{subsub:estparam} 
These parameters can be estimated by the main server during model training. The server can estimate $\delta$ and $\zeta$ at each global aggregation by receiving the latest gradients from SGD at the sampled devices. $\sigma^2$ can first be estimated locally at the sampled devices, and then decided at the main server.

Specifically, to estimate $\delta,\zeta$, since the value of $\mathbf w^*$ is not known, we upper bound the gradient diversity in Definition \ref{gradDiv} by introducing a new parameter $\delta'$:
\begin{align} \label{eq:estGradDiv}
  \hspace{-3mm}  \Vert\nabla\hat F_c(\mathbf w)-\nabla F(\mathbf w)\Vert \leq \delta+ \zeta \Vert\mathbf w-\mathbf w^*\Vert \leq  \delta' + \zeta \Vert\mathbf w\Vert,  \hspace{-3mm} 
\end{align}
which satisfies $\delta'\geq \delta+\zeta\Vert\mathbf w^*\Vert$. Thus, a value of $\zeta< 2\beta$ is set, and then the value of $\delta'$ is estimated using~\eqref{eq:estGradDiv}, where the server uses the SGD gradients $\widehat{\mathbf g}_{n_c}^{(t_k)}$ from the sampled devices $n_c$ at the instance of each global aggregation $k$, and chooses the smallest $\delta'$ such that $\Vert\nabla\hat F_c(\hat{\mathbf w}^{(t_k)})-\nabla F(\hat{\mathbf w}^{(t_k)})\Vert \approx \Vert \widehat{\mathbf g}_{n_c}^{(t_k)}-\sum_{c'=1}^N \varrho_{c'}\widehat{\mathbf g}_{n_{c'}}^{(t_k)}\Vert\leq \delta'+\zeta \Vert \hat{\mathbf w}^{(t_k)}\Vert$ $\forall c$.
% instead of their full-batch counterparts

% for each instance of $k'\in\{ 1,\cdots,n\}$ the maximum among which is chosen as the final value of $\delta'$.  Given $\Vert \mathbf w^{(k')}\Vert$, the server uses the SGD values instead of full-batch counterparts: $\Vert\nabla\hat F_c(\mathbf w^{(k')})-\nabla F(\mathbf w^{(k')})\Vert \approx \Vert \widehat{\mathbf g}_{i}^{(k')}-\sum_{j=1}^I \varrho_j\widehat{\mathbf g}_{j}^{(k')}\Vert\leq \delta'+\Vert \mathbf w^{(k')}\Vert$.

% To estimate the value of SGD variance $\sigma^2$, since the estimation of full-batch gradient might be impractical due to large number of local data points, each node $i$ estimates its local SGD noise at the end of each local model training round $k$ as $\sigma^2_i=\Vert\widehat{\mathbf g}_{j}^{(t_k)}-\widetilde{\nabla F}_j(\mathbf w_{j}^{(t_k)})\Vert^2$, where $\widehat{\mathbf g}_{j}^{(t_k)}$ is the latest SGD computed at the node and $\widetilde{\nabla F}_j(\mathbf w_{j}^{(t_k)})$ is a more accurate estimation of the full-batch gradient obtained via sampling a larger number of data points as compared to the mini-batch size. These scalars are then transferred to the main server, which chooses $\sigma^2=\max\{\sigma^2_1,\cdots,\sigma^2_I\}$.

From Assumption~\ref{assump:SGD_noise}, a simple way of obtaining the value of $\sigma^2$ would be comparing the gradients from sampled devices with their full-batch counterparts. But this might be impractical if the local datasets $\mathcal{D}_i$ are large. Thus, we propose an approach where $\sigma^2$ is computed at each device through two independent mini-batches of data. Recall $|\xi_i|$ denotes the mini-batch size used at node $i$ during the model training. 
At each instance of global aggregation, the sampled devices each select two mini-batches of size $|\xi_i|$ and compute two SGD realizations $\mathbf g_1$, $\mathbf g_2$ from which $\widehat{\mathbf g}_{i}^{(t_k)} = (\mathbf g_1 + \mathbf g_2) /2$. Since $\mathbf  g_1= \nabla F_i(\mathbf w^{(t_k)})+\mathbf n_1$, $\mathbf g_2= \nabla F_i(\mathbf w^{(t_k)}) +\mathbf n_2$, we use the fact that $\mathbf n_1$ and $\mathbf n_2$ are independent random variables with the same upper bound on variance $\sigma^2$, and thus $\Vert \mathbf  g_1-\mathbf  g_2 \Vert^2 = \Vert \mathbf n_1-\mathbf n_2 \Vert^2 \leq  2\sigma^2$, from which $\sigma^2$ can be approximated locally. These scalars are then transferred to the main server, which in turn chooses the maximum reported $\sigma^2$ from the sampled devices.
 
\subsubsection{Estimation of $\Upsilon^{(t)}_c$} 
Based on~\eqref{eq:cons}, we propose the following approximation to estimate the value of $\Upsilon^{(t)}_c$:
% By Definition \ref{paraDiv} and given the conditions in XXX, since the server needs to determine the learning parameters at the beginning of global aggregation, we are interesting in approximating the prediction of $\epsilon_c(k)$ for $t\in\mathcal T_k$ by $\epsilon_c(k)$ at the beginning of the global aggregation $k$. Therefore, $\epsilon_c(k)$ is approximated as
% \begin{align*}
%     \epsilon_c(k) &\geq \Vert\mathbf w_i(t)-\mathbf w_j(t)\Vert
%     \\&
%     \geq \vert\Vert\mathbf w_i(t)\Vert-\Vert\mathbf w_j(t)\Vert\vert, ~\forall i, j\in\mathcal{S}_c^{(k)}, ~\forall t\in\mathcal T_k
% \end{align*}
    \begin{align} \label{eq:Ups_est}
         \Upsilon^{(t)}_c&
         =\max_{j,j'\in\mathcal S_c}\Vert\tilde{\mathbf w}_j^{(t)}-\tilde{\mathbf w}_{j'}^{(t)}\Vert
%         \approx \max_{i, j\in\mathcal{S}_c,1\leq z\leq M}\{\vert[\mathbf w_i(t)]_z-[\mathbf w_j(t)]_z\vert\} 
    \nonumber \\ &
    \approx \underbrace{\max_{j\in\mathcal S_c}\Vert\tilde{\mathbf w}_j^{(t)}\Vert}_{(a)}- \underbrace{\min_{j\in\mathcal S_c}\Vert\tilde{\mathbf w}_j^{(t)}\Vert}_{(b)},%~ i\neq j, 
    %\approx \underbrace{\max_{i\in\mathcal{S}_c} \Vert  \mathbf w_i(t)\Vert_{\infty}}_{(a)}- \underbrace{\min_{j\in\mathcal{S}_c} \Vert  \mathbf w_j(t)\Vert_{\infty}}_{(b)},~ i\neq j, 
    % \leq \max_{i, j\in\mathcal{S}^{(k)}_c,1\leq z\leq M}\{\vert[\mathbf w_i(t)]_z\vert+\vert[\mathbf w_j(t)]_z\vert\}
    % \nonumber \\& \leq 2\underbrace{\max_{i, j\in\mathcal{S}^{(k)}_c,1\leq z\leq M}\{\vert[\mathbf w_i(t)]_z\vert\}}_{(a)},~\forall t\in \tau_k,
    \end{align}
% \begin{align}
%     \Upsilon_c(k) &\approx \max_{i, j\in\mathcal{S}_c.}\{\Vert\mathbf w_i(t)\Vert-\Vert\mathbf w_j(t)\Vert\}
%     \\& \label{epsProx}
%     =\underbrace{\max_{i\in\mathcal{S}_c}\Vert\mathbf w_i(t)\Vert}_{(a)}-\underbrace{\min_{j\in\mathcal{S}_c}\Vert\mathbf w_j(t)\Vert}_{(b)}, ~\forall t\in\mathcal T_k,
% \end{align} 
where we have used the lower bound $\Vert\mathbf a - \mathbf b \Vert \geq \Vert\mathbf a\Vert-\Vert\mathbf b\Vert$ for vectors $\mathbf a$ and $\mathbf b$, which we experimentally observe gives a better approximation of $\Upsilon^{(t)}_c$. In~\eqref{eq:Ups_est}, $(a)$ and $(b)$ can be both obtained in a distributed manner through scalar message passing, where each device $i\in\mathcal{S}_c$ computes $\Vert\tilde{\mathbf w}_i^{(t)}\Vert$  %and $\min_{1\leq z\leq M}\{\vert[\mathbf w_i(t)]_z\vert\}$
and shares it with its neighbors $j \in \mathcal{N}_i$. The devices update their $\max$ and $\min$ accordingly, share these updated values, and the process continues. After the rounds of message passing has exceeded the diameter of the graph, each node has the value of $(a)$ and $(b)$, and thus the estimate of $\Upsilon^{(t)}_c$. The server can obtain these values for $t \in \mathcal{T}_k$ from the node $n_c$ it samples for cluster $c$ at $t = t_k$. %The aforementioned estimated parameters are used in our control algorithm that adaptively tunes the number of D2D communication rounds inside the clusters in a distributed manner and obtains the length of local model training intervals.

\subsection{{\tt TT-HF} with Adaptive Parameter Control}
\label{ssec:control}
The full {\tt TT-HF} procedure with adaptive parameter control is summarized in Algorithm~\ref{GT}. The values of $\tau$, desired $\xi$ and $T$, and model characteristics $\mu, \beta$ are provided as inputs.

First, estimates of different parameters are initialized, the value of $\phi$ is determined, and the first period of model training is set (lines 2-6). Then, during the local model training intervals, in each timestep, the devices (i) compute the SGD updates, (ii) estimate the cluster model divergence, (iii) determine the number of D2D consensus rounds, and (iv) conduct the consensus process with their neighboring nodes (lines 12-16).

At global aggregation instances, the sampled devices compute their estimated local SGD noise, and transmit it along with their model parameter vector, gradient vector, and estimates of cluster parameter divergence over the previous global aggregation round to the server (lines 20-21). Then, the main server (i) updates the global model, (ii) estimates $\zeta,\delta',\sigma$ for the step size, (iii) estimates the linear model coefficients used in~\eqref{eq:Up_dyn}, (iv) obtains the optimal length $\tau_{k+1}$ of the next local model training interval, and (v) broadcasts the updated global model, step size coefficients, local model training interval, and consensus coefficient, along with the indices of the sampled devices for the next global aggregation (line 23-29).

\begin{algorithm}[t]
{\footnotesize 
\SetAlgoLined
\caption{{\tt TT-HF} with adaptive control parameters.} \label{GT}
% \KwResult{Write here the result }
\KwIn{Desirable loss criterion $\xi$, length of model training $T$, maximum tolerable $\tau$, and model-related parameters $\beta,\mu$} 
\KwOut{Global model $\hat{\mathbf w}^{(T)}$}
// Start of initialization by the server\\
 Initialize $\hat{\mathbf w}^{(0)}$ and broadcast it among the devices along with the indices $n_c$ of the sampled devices for the first global aggregation.\\
 Initialize estimates of $\zeta \ll 2\beta,\delta',\sigma$.\\
 Initialize $\alpha$ and $\gamma > 1/\mu$ for the step size $\eta_t=\frac{\gamma}{t+\alpha}$, where $\alpha$ is the smallest solution that satisfies the condition mentioned in Sec.~\ref{subsec:learniParam}, and $\alpha, \gamma,\xi,\tau, T$ satisfy~\eqref{46}.\\
 Obtain $\phi^{\mathsf{max}}$ from~\eqref{47}.\\
 Initialize $\tau_1$ randomly, where $\tau_1\leq \tau$.\\
// End of initialization by the server\\
%  \textbf{Initialization by the server:} 
Initialize $t=1, ~k=1,~ t_0=0, ~t_1=\tau_1$.\\
 \While{$t\leq T$}{
     \While{$t\leq t_k$}{
        \For{$c=1:N$}
        {
        // Operation at the clusters\\
         Each device $i\in\mathcal{S}_c$ performs a local SGD update based on~\eqref{eq:SGD} and~\eqref{8} using $\widehat{\mathbf w}_i^{(t-1)}$ to obtain~$\widetilde{\mathbf{w}}_i^{(t)}$.\\
        %  with:\\
        %   $\widetilde{\mathbf w}_i^{(t)} = 
        %   \mathbf w_i^{(t-1)}-\eta_{t-1} \widehat{\mathbf g}_{i}^{(t-1)}$\; 
        Devices estimate the value of $\Upsilon^{(t)}_c$ using~\eqref{eq:Ups_est} with distributed message passing.\\
        Devices compute the number of D2D communication consensus rounds  $\Gamma_c^{(t)}$ according to~\eqref{eq:consensusNum}.\\
        Devices inside the cluster conduct $\Gamma^{(t)}_{{c}}$ rounds of consensus procedure based on~\eqref{eq:ConsCenter}, initializing  $\textbf{z}_{i}^{(0)}=\widetilde{\mathbf{w}}_i^{(t)}$, and setting $\mathbf w_i^{(t)} = \textbf{z}_{i}^{(\Gamma^{(t)}_{{c}})}$.
        %  For each edge device $i\in\mathcal{S}_c^{(k)}$ in parallel, perform local SGD update with:\\
        %   $\widetilde{\mathbf w}_i^{(t)} = 
        %   \mathbf w_i^{(t-1)}-\eta_{t-1} \widehat{\mathbf g}_{i}^{(t-1)}$\; 
        % Perform $\Gamma^{(t)}_{{c}}$ rounds of consensus on $\widetilde{\mathbf w}_i^{(t)}$ as:\\
        % $\textbf{z}_{i}^{(0)}=\widetilde{\mathbf{w}}_i^{(t)}$\;
        % \For{$t'=0:\Gamma^{(t)}_{{c}}-1$}{
        %     $\textbf{z}_{i}^{(t'+1)}= v^{(k)}_{i,i} \textbf{z}_{i}^{(t')}+\hspace{-2mm} \sum_{j\in \mathcal{N}^{(k)}_i} v^{(k)}_{i,j}\textbf{z}_{j}^{(t')}$
        % }
        % $\mathbf w_i^{(t)} = \textbf{z}_{i}^{(\Gamma^{(t)}_{{C}})}$\; 
      }
      \If{$t=t_k$}{
      // Operation at the clusters\\
      Each sampled device $n_c$ estimates the local SGD noise 
      as described in Sec.~\ref{subsub:estparam}.\\
      Each sampled devices $n_c$ sends $\mathbf w_{n_c}^{(t_k)}$, $\widehat{\mathbf g}_{n_c}^{(t_k)}$, the estimated local SGD noise, and the estimated values of $\Upsilon_c(t)$, $t \in \mathcal{T}_k$ to the server. \\
      // Operation at the server\\
    %   Estimate $\delta$\;
      Compute $\widehat{\mathbf w}^{(t_k)}$ using \eqref{15}.\\
      Set $\zeta\ll 2\beta$, and compute $\delta'=\Big[\max_{c} \{\Vert \widehat{\mathbf g}_{n_c}^{(t_k)}-\sum_{c'=1}^N \varrho_{c'}\widehat{\mathbf g}_{n_{c'}}^{(t_k)}\Vert- \zeta \Vert \hat{\mathbf w}^{(t_k)}\Vert\}\Big]^+$. \\
      Choose the maximum among the reported local SGD noise values as $\sigma^2$.\\
      Characterize $\alpha$ and $\gamma > 1/\mu$ for the step size $\eta_t=\frac{\gamma}{t+\alpha}$ according to the condition on $\alpha$ in Sec.~\ref{subsec:learniParam} and \eqref{46}, and compute $\phi^{\mathsf{max}}$ according to~\eqref{47}.\\
    %   Synchronize local models with the global model $\mathbf w_i(t)=\hat{\mathbf w}(t),~\forall i$ //Global Synchronization\;
    %   Estimate $\Upsilon^{(t)}_c,~\forall t\in\mathcal T_{k+1}$\;
    %   Compute $\Gamma^{(t)}_c,~\forall t\in\mathcal T_{k+1}$ by~\eqref{eq:consensusNum} //Rounds of consensus\; 
   Estimate $A^{(k+1)}_c$, $B^{(k+1)}_c$, $a^{(k+1)}_c$, and $b^{(k+1)}_c$, $\forall c$ in~\eqref{eq:Up_dyn} via linear data fitting.\\
   Solve the optimization~$(\bm{\mathcal{P}})$ to obtain $\tau_{k+1}$.\\
     Broadcast $\widehat{\mathbf w}^{(t_k)}$ among the devices along with (i) the $n_c$ for $k+1$, (ii) $\alpha$, (iii) $\gamma$, (iv) $\tau_{k+1}$, and (iv) $\phi$.
    %  Solve Problem $(\bm{\mathcal{P}})$ to obtain $\tau_{k+1}$, and set $t_{k+1}=t+\tau_{k+1}$\\
    %  Compute $\Xi^{(k)}$ using~\eqref{eq:Xi_est} and set $\mu=\min\{ \}_{k'=1}^{k}$, and $\beta=\max\{ \}_{k'=1}^{k}$
       }
     $t=t+1$
 }
}
}
\end{algorithm}

\vspace{-1mm}
\section{Numerical Evaluations}
\label{sec:experiments}

\noindent In this section, we conduct numerical experiments to verify the performance of {\tt TT-HF}. After describing the setup in Sec.~\ref{ssec:setup}, we study model performance/convergence in Sec.~\ref{ssec:conv-eval} and the impact of our adaptive control algorithm in Sec.~\ref{ssec:control-eval}. Overall, we will see that {\tt TT-HF} provides substantial improvements in training time, accuracy, and/or resource utilization compared to conventional federated learning~\cite{wang2019adaptive,Li}. 

\vspace{-2mm}
\subsection{Experimental Setup}\label{channelModel}
\label{ssec:setup}

\noindent \textbf{Network architecture.} 
We consider a network consisting of $I = 125$ edge devices placed into $N = 25$ clusters, each with $s_c = 5$ devices placed uniformly at random in a $50\text{ m}\times50\text{ m}$ square field (in each cluster). The channel model and D2D network configuration are explain below.
% The links among devices within each cluster $c$ are generated using a random geometric graph~\cite{hosseinalipour2020multi}, tuned such that the clusters have an average spectral radius of $\rho = 0.7$. \\
% For the cluster consensus algorithm in~\eqref{eq:consensus}, we follow the distributed average consensus approach outlined above Assumption~\ref{assump:cons}. \vspace{-1mm}

\textit{\textit{Channel model:}} We assume that the D2D communications are conducted using orthogonal frequency division techniques, e.g., OFDMA, to reduce the interference across the devices. We consider the instantaneous channel capacity for transmitting data from node $i$ to $i'$, both belonging to the same cluster $c$ following this formula:
\begin{equation}\label{eq:shannon_m}
    C_{i,i'}^{(t)}=W\log_2 \left(1+ \frac{p_i^{(t)} \vert {h}^{(t)}_{i,i'}\vert^2}{\sigma^2} \right),
\end{equation}
where $\sigma^2=N_0 W$ is the noise power, with $N_0=-173$ dBm/Hz denoting the white noise power spectral density;
$W=1$ MHz is the bandwidth; $p^{(t)}_i=24$ dBm, $\forall i,t$ is the transmit power; ${h}^{(t)}_{i,i'}$ is the channel coefficient. We incorporate the effect of both large-scale and small scaling fading in $h^{(t)}_{i,i'}$, given by \cite{tse2005fundamentals,channelJun2021}:
\begin{equation}
     h_{i,i'}^{(t)}= \sqrt{\beta_{i,i'}^{(t)}} u_{i,i'}^{(t)},
\end{equation}
where $\beta_{i,i'}^{(t)}$ is the large-scale pathloss coefficient and $ u_{i,i'}^{(t)} \sim \mathcal{CN}(0,1)$ captures Rayleigh fading, varying i.i.d. over time. We assume channel reciprocity, i.e., $h^{(t)}_{i,i'}=h^{(t)}_{i',i}$, for simplicity. We model $\beta_{i,i'}^{(t)}$ as \cite{tse2005fundamentals,channelJun2021}
\begin{equation}
    \beta_{i,i'}^{(t)} = \beta_0 - 10\alpha\log_{10}(d^{(t)}_{i,i'}/d_0).
\end{equation}
where $\beta_0=-30$ dB denotes the large-scale pathloss coefficient at a reference distance of $d_0=1$ m, $\alpha$ is the path loss exponent chosen as $3.75$ suitable for urban areas, and $d^{(t)}_{i,i'}$ denotes the instantaneous Euclidean distance between the respective nodes.
% We assume that the nodes are placed in an urban area with path-loss exponent 3.75. We consider the noise spectral density of $N_0=-173$ dBm/Hz and further incorporate fading coefficient as Rayleigh fading channel, where the fading coefficient varies according to the distance between two devices.
% \nm{moved:}%We assume that the devices in each cluster are placed uniformly at random in a $50m\times50m$ square field. 
\\

\textit{D2D network configuration:} 
We incorporate the wireless channel model explained above into our scenario to define the set of D2D neighbors and configure the cluster topologies.
% We consider two cases: (i) unknown Channel State Information (CSI) to the devices and (ii) known CSI at the devices. For both cases, 
We assume that the nodes moves slowly so that their locations remain static during each global aggregation period, although it may change between consecutive global aggregations.  
We build the cluster topology based on channel reliability across the nodes quantified via the outage probability. Specifically, considering~(\ref{eq:shannon_m}), the probability of outage upon transmitting with data rate of $R^{(t)}_{i,i'}$ between two nodes $i,i'$  is given by
\begin{figure}[t]
\includegraphics[width=1.0\columnwidth]{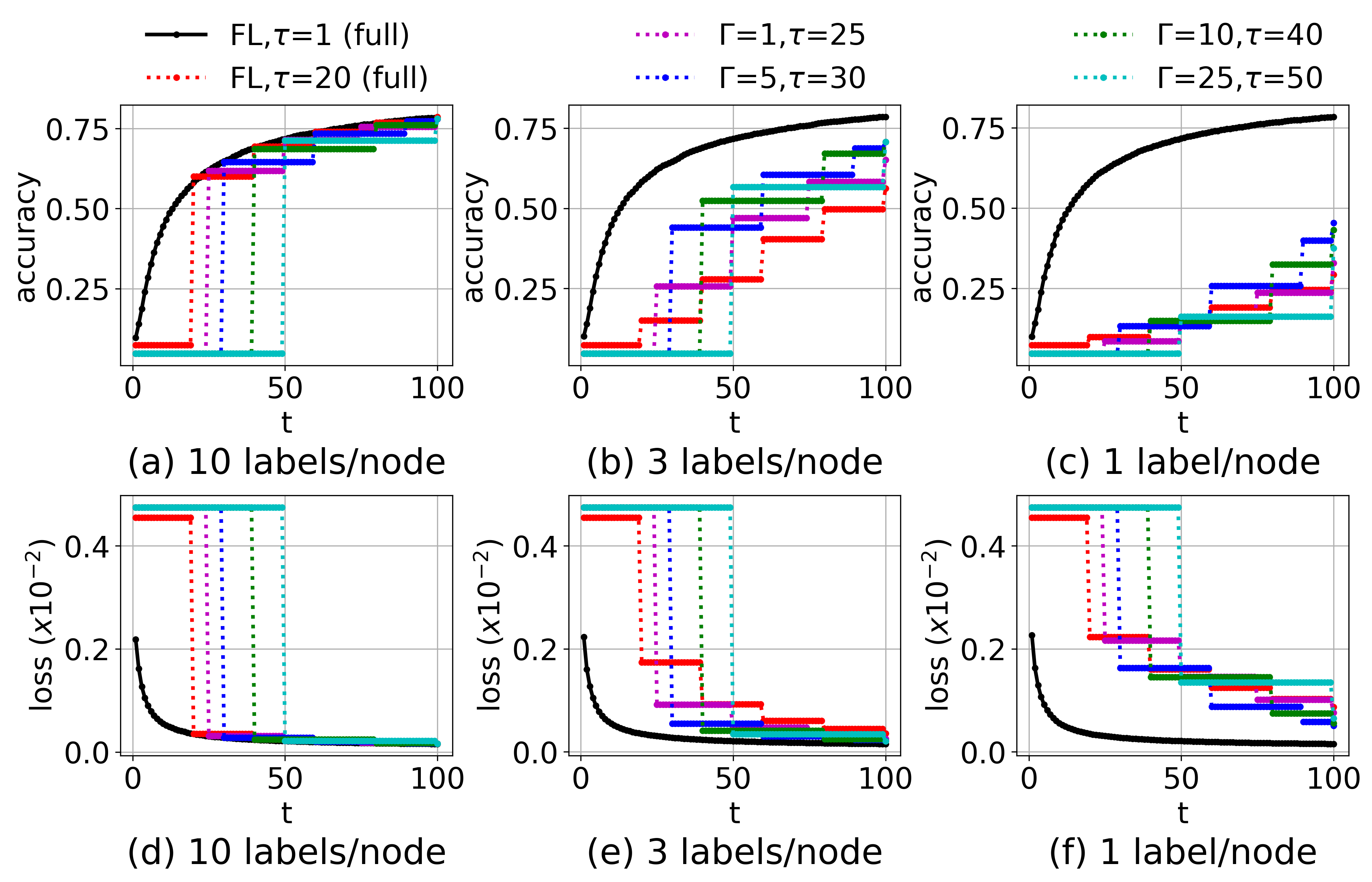}
\centering
\caption{Performance comparison between {\tt TT-HF} and baseline methods when varying the local model training interval ($\tau$) and the number of D2D consensus rounds ($\Gamma$). With a larger $\tau$, {\tt TT-HF} can still outperform the baseline federated learning~\cite{wang2019adaptive,Li} if $\Gamma$ is increased, i.e., local D2D communications can be used to offset the frequency of global aggregations. \emph{full} implies that baseline schemes do not leverage D2D and instead require all the device to engage in uplink transmissions. SVM is  used for classification. }
\label{fig:mnist_poc_2_all}
\vspace{-5mm}
\end{figure}

\begin{equation}\label{eq:out}
    \textrm{p}^{\mathsf{out},(t)}_{i,i'}=1-\exp\Big(\frac{-(2^{R_{i,i'}^{(t)}}-1)}{\mathrm{SNR}^{(t)}_{i,i'}}\Big),
\end{equation}
 where $\mathrm{SNR}^{(t)}_{i,i'}=\frac{p_i^{(t)}\vert h_{i,i'}^{(t)}\vert^2}{\sigma^2}$.
To construct the graph topology of each cluster $c$, we create an edge between two nodes $i$ and $i'$ if and only if their respective outage probability satisfies $\textrm{p}^{\mathsf{out},(t)}_{i,i'} \leq 5\%$ given a defined common data rate $R^{(t)}_{i,i'}=R^{(t)}_c$, chosen as $R_c^{(t)}=14$ Mbps. This value is used since it is large enough to neglect the effect of quantization error in digital communication of the signals, and at the same time results in connected graphs inside the clusters (numerically, we found an average degree of 2 nodes in each cluster).  %In particular, we find out that with this common rate, an average degree around $2$ in each cluster is achieved.
{After creating the topology based on the large-scale pathloss and outage probability requirements, we model outages during the consensus phase as follows:} if the {instantaneous channel capacity} (given by~\eqref{eq:shannon_m}, which captures the effect of fast fading) on an edge drops below $R_c^{(t)}$, outage occurs, so that the packet is lost and the model update is not received at the respective receiver. Therefore, although nodes are assumed to be static during each global aggregation period,
     the {instantaneous} cluster topology, i.e., the communication configuration among the nodes, changes
     with respect to every local SGD iteration in a model training interval due to outages.
% The configuration of $R_c^{(t)}$ in the two scenarios (i.e., unknown and known CSI) are described as follows: 
\\

     Given a communication graph we choose $d_c=1/8$ to form the consensus iteration at each node $i$ as $\textbf{z}_{i}^{(t'+1)} = \textbf{z}_{i}^{(t')}+d_{{c}}\sum_{j\in \mathcal{N}_i} (\textbf{z}_{j}^{(t')}-\textbf{z}_{i}^{(t')})$  (refer to the discussion provided after Assumption~\ref{assump:cons}). Note that given $d_c$, broadcast by the server at the beginning of each global aggregation, each node can conduct D2D communications and local averaging without any global coordination.  
    
\noindent\textbf{Datasets.}
We consider MNIST~\cite{MNIST} and Fashion-MNIST (F-MNIST)~\cite{xiao2017}, two datasets commonly used in image classification tasks. Each dataset contains $70$K images ($60$K for training, $10$K for testing), where each image is one of 10 labels of hand-written digits and fashion products, respectively. For brevity, we present the results for MNIST here, and refer the reader to Appendix~\ref{app:experiments} for FMNIST; the results are qualitatively similar.\\
\textbf{Data distributions.}
To simulate varying degrees of statistical data heterogeneity among the devices, we divide the datasets into the devices' local $\mathcal{D}_i$ in three ways: (a) \textit{extreme non-i.i.d.}, where each local dataset has only data points from a single label; (b) \textit{moderate non-i.i.d.}, where each local dataset contains datapoints from three of the 10 labels; and (c) \textit{i.i.d.}, where each local dataset has datapoints covering all $10$ labels. In each case, $\mathcal{D}_i$ is selected randomly (without replacement) from the full dataset of labels assigned to device $i$. \\
% per node, where each of which corresponds to a particular level of data heterogeneity. (a) represents the case of extreme non-i.i.d, where each device possesses data labeled with only $1$ of the $10$ classes; (b) represents the case of moderate non-i.i.d, where each device samples data from $3$ classes out of $10$; (c) represents the case of i.i.d, where each device has samples uniformly sampled from the entire dataset.\\

\begin{figure}[t]
\includegraphics[width=1.0\columnwidth]{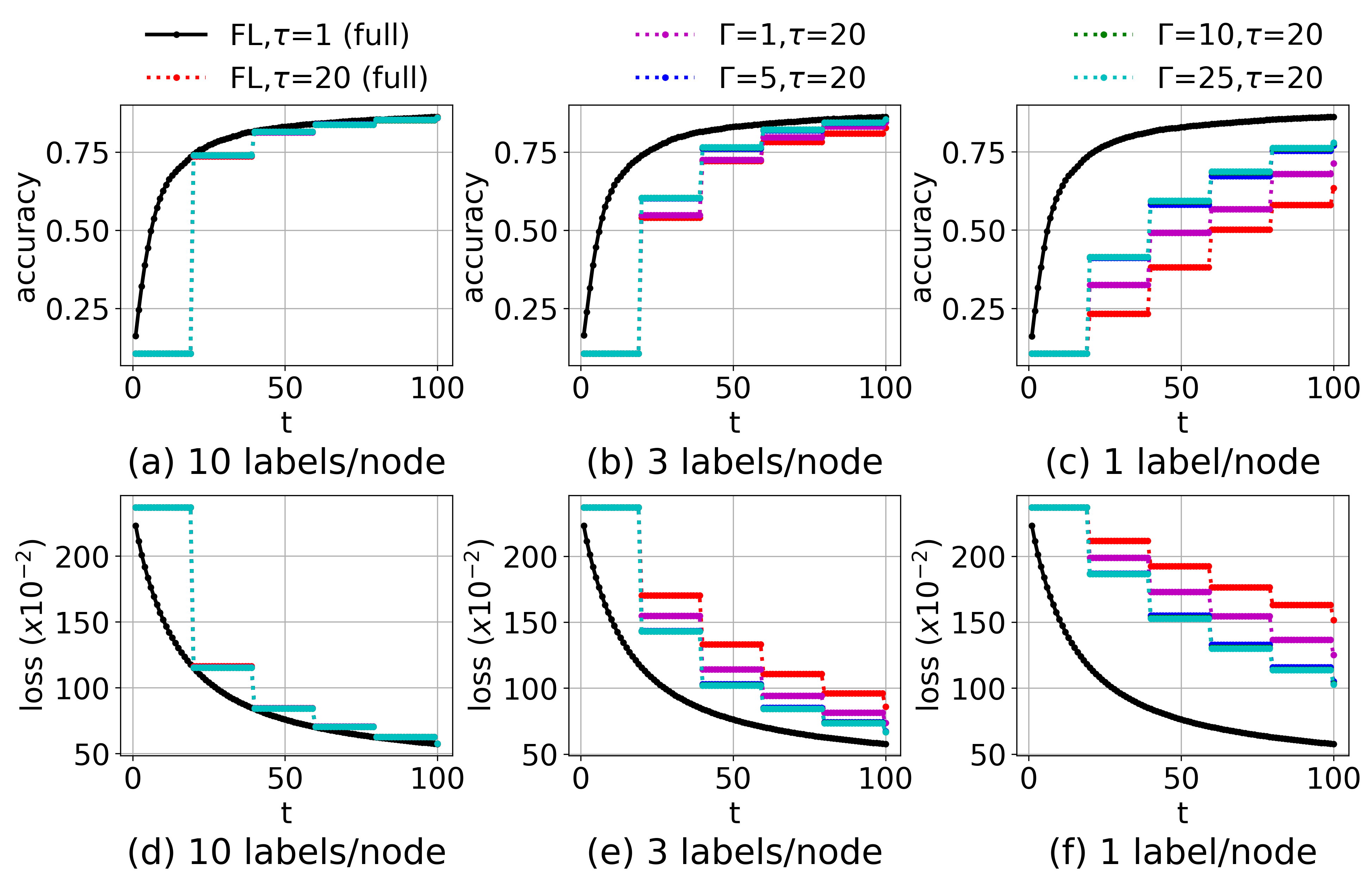}
\centering
\caption{Performance comparison between {\tt TT-HF} and baseline methods when varying the number of D2D consensus rounds ($\Gamma$). Under the same period of local model training ($\tau$), increasing $\Gamma$ results in a considerable improvement in the model accuracy/loss over time as compared to the current art~\cite{wang2019adaptive,Li} when data is non-i.i.d. \emph{full} implies that baseline schemes do not leverage D2D and instead require all the device to engage in uplink transmissions. NN is  used for classification.}
\label{fig:mnist_poc_1_all}
\vspace{-5mm}
\end{figure} 
\textbf{ML models.} 
We consider loss functions from two different ML classifiers: regularized (squared) support vector machines (SVM) and a fully connected neural network (NN). In both cases, we use the standard implementations in PyTorch which results in a model dimension of $M = 7840$ on MNIST. Note that the SVM satisfies Assumption~\ref{Assump:SmoothStrong}, while the NN does not.  The numerical results obtained for both classifiers are qualitatively similar. Thus, for brevity, we show a selection of results for each classifier here, and refer the reader to Appendix~\ref{app:experiments} for the extensive simulation results on both classifiers, where we also explain the implementation of our control algorithm for non-convex loss functions. The SVM uses a linear kernel, and the weights initialization follows a uniform distribution, with mean and variance calculated according to~\cite{he2015delving}. All of our implementations can be accessed at~\cite{ShamsGithub}.
% \url{https://github.com/shams-sam/TwoTimeScaleHybridLearning}.

% where we explain the implementaion of our control algorithm for non-convex functions and ....
% More details on the experiment setup is given in Appendix.... , where we also complement the results by investigating the performance considering other network size and datasets. 

% \begin{figure}[t]
% \includegraphics[width=1.0\columnwidth]{mnist_125/poc_1_all_noSCI.png}
% \centering
% \caption{Performance comparison between {\tt TT-HF} and baseline methods when varying the number of D2D consensus rounds ($\Gamma$). Under the same period of local model training ($\tau$), increasing $\Gamma$ results in a considerable improvement in the model accuracy/loss over time as compared to the current art~\cite{wang2019adaptive,Li} when data is non-i.i.d.}
% , even $\Gamma=1$ increases the convergence rate (magenta line) compared to the case with no D2D consensus between global aggregations (red line marked FL). This difference is not apparent when the data is i.i.d distributed (all 10 labels per node: subplot (a) and (d)) but becomes prominent in the extreme non-i.i.d case (1 labels per node: subplots (c) and (f)).}
% moderate non-i.i.d (3 label per node: subplots (a) and (b)), but becomes more visible in extreme non-i.i.d case (1 labels per node: subplots (c) and (d)). The difference in performance is not apparent when the data is i.i.d distributed (all 10 labels per node: subplot (e)).} 
% \label{fig:mnist_poc_1_all}
% \vspace{-5mm}
% \end{figure} 

\subsection{{\tt TT-HF} Model Training Performance and Convergence}
\label{ssec:conv-eval}
One of the main premises of {\tt TT-HF} is that cooperative consensus procedure within clusters during the local model training interval can (i) preserve model performance while reducing the required frequency of global aggregations and/or (ii) increase the model training accuracy, especially when statistical data heterogeneity is present across the devices. Our first set of experiments seek to validate these facts:

\subsubsection{Local consensus reducing global aggregation frequency}
In Fig.~\ref{fig:mnist_poc_2_all}, we compare the performance of {\tt TT-HF} for increased local model training intervals $\tau$ against the current federated learning algorithms that do not exploit local D2D model consensus procedure. The baselines both assume full device participation (i.e., all devices upload their local model to the server at each global aggregation), and thus are 5x more uplink resource-intensive at each aggregation. One baseline conducts global aggregations after each round of training ($\tau = 1$), and the other, based on~\cite{wang2019adaptive}, has local update intervals of $20$ ($\tau = 20$). Recall that longer local training periods are desirable to reduce the frequency of communication between devices and the main server. 
We conduct consensus after every $t = 5$ time instances, and increase $\Gamma$ as $\tau$ increases. The $\tau = 1$ baseline is an upper bound on the achievable performance since it replicates centralized model training.

% As in Fig.~\ref{fig:mnist_poc_1_all}, we also conduct consensus after every $t = 5$ time instances, and increase $\Gamma$ as $\tau$ increases.

Fig.~\ref{fig:mnist_poc_2_all} confirms that {\tt TT-HF} can still outperform the baseline FL with $\tau = 20$ when the frequency of global aggregations is decreased: in other words, increasing $\tau$ can be counteracted with a higher degree of local consensus procedure $\Gamma^{(t)}_c=\Gamma$, $\forall c,t$. Considering the moderate non-i.i.d. plots ((b) and (e)), we also see that the jumps in global model performance, while less frequent, are substantially larger for {\tt TT-HF} than the baseline. This result shows that D2D communications can reduce reliance on the main server for a more distributed model training process.
It can also be noted that {\tt TT-HF} achieves this performance gain despite the communication impairments, i.e., packet lost due to fast fading, that we assumed in D2D communications. This implies the robustness of {\tt TT-HF} to imperfect D2D communications among the devices. 
% \nm{you need to add a few lines highlighting that TTHF is robust against impairments caused by packet losses due to fading.}

\begin{figure}[t]
\includegraphics[width=1.0\columnwidth]{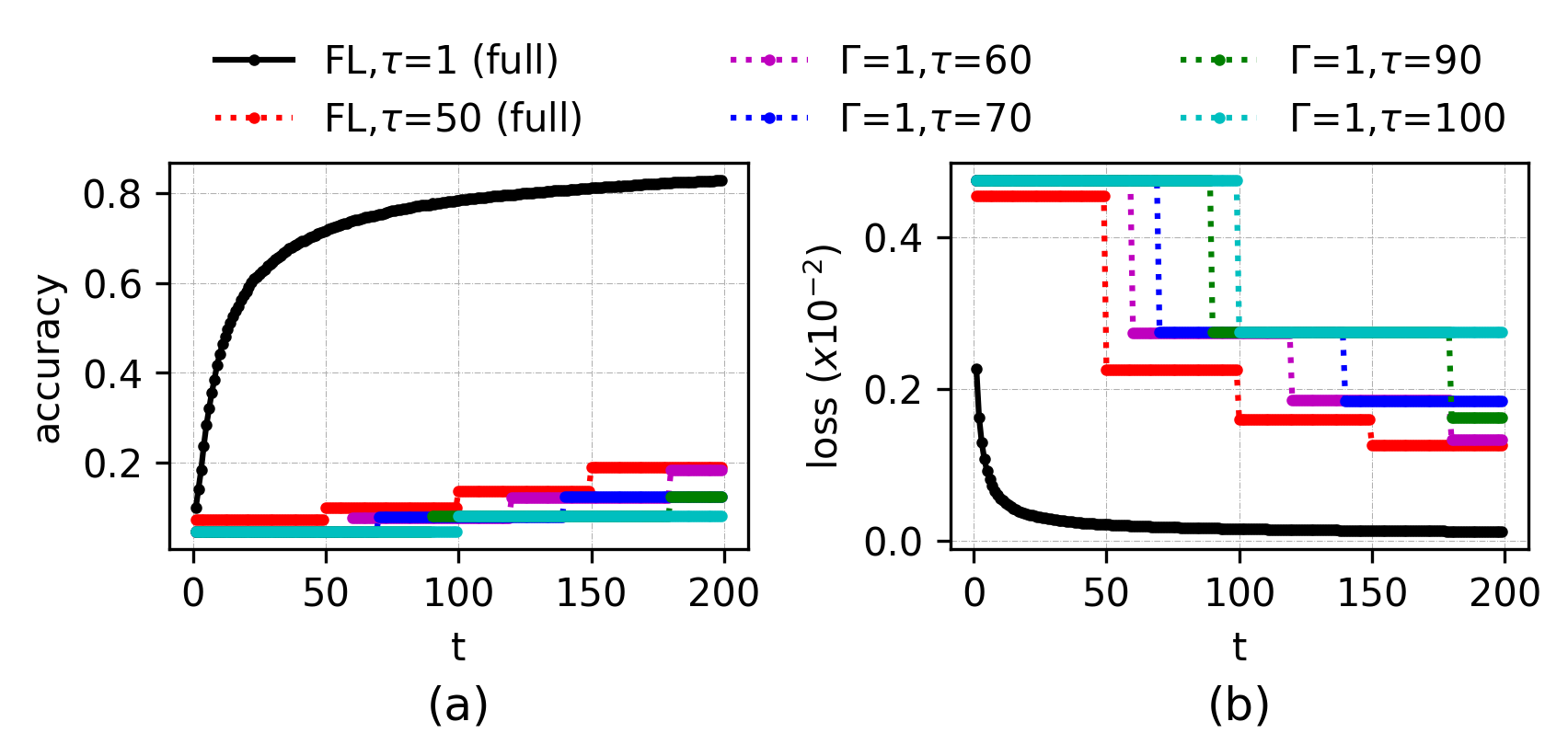}
\centering
\caption{Performance of {\tt TT-HF} in the extreme non-i.i.d. case for the setting in Fig.~\ref{fig:mnist_poc_2_all} when $\Gamma$ is small and the local model training interval length is increased substantially. {\tt TT-HF} exhibits poor convergence behavior when $\tau$ exceeds a certain value, due to model dispersion. SVM is  used for classification.}
% similar to the baseline~\cite{wang2019adaptive,Li} (red curve).} 
\label{fig:poc_3_iid_1_gamma_1_lut_50}
\vspace{-5mm}
\end{figure}

\subsubsection{D2D enhancing ML model performance} 
In Fig.~\ref{fig:mnist_poc_1_all}, we compare the performance of {\tt TT-HF} with the baseline methods, where we set $\tau_k = \tau = 20$ and conduct a fixed number of D2D rounds in clusters after every $5$ time instances, i.e., $\Gamma^{(t)}_c = \Gamma$ for different values of $\Gamma$. 
Fig.~\ref{fig:mnist_poc_1_all} verifies that local D2D communications can significantly boost the performance of ML model training. Specifically, when the data distributions are moderate non-i.i.d ((b) and (e)) or extreme non-i.i.d. ((c) and (f)), we see that increasing $\Gamma$ improves the trained model accuracy/loss substantially from FL with $\tau = 20$. It also reveals that there is a diminishing reward of increasing $\Gamma$ as the performance of {\tt TT-HF} approaches that of FL with $\tau = 1$. Finally, we observe that the gains obtained through D2D communications are only present when the data distributions across the nodes are non-i.i.d., as compared to the i.i.d. scenario ((a) and (d)), which emphasizes the purpose of {\tt TT-HF} for handling statistical heterogeneity. This result further shows the applicability of {\tt TT-HF} to non-convex classifiers such as NN.

% \nm{you need to add a few lines highlighting the TTHF performs well despite the fact that NN is non convex.}

%One of the main motivations for prolonging the period of local model training interval is to remove the burden of extensive uplink communications from the devices. We thus compare the method of~\cite{wang2019adaptive,Li} when global models are performed after every 20 SGD local updates (red curve) with {\tt TT-HF} when longer period of local model training are deployed in Fig.~\ref{fig:mnist_poc_2_all}. In this simulation we assume that $\Gamma$ D2D communication rounds are performed after every $5$ local SGD updates at the devices.  The results demonstrates that {\tt TT-HF} can outperform the baseline method (red curve) even when the duration of local model training are extended. This result demonstrate that with the exploitation of D2D communications {\tt TT-HF} has less reliance on the main server can model training can be prolonged while a satisfactory performance achieved.

\subsubsection{Convergence behavior}
Recall that the upper bound on convergence in Theorem~\ref{co1} is dependent on the expected model dispersion $A^{(t)}$ and the consensus error $\epsilon^{(t)}$ across clusters. For the settings in Figs.~\ref{fig:mnist_poc_1_all}\&\ref{fig:mnist_poc_2_all}, increasing the local model training period $\tau$ and decreasing the consensus rounds $\Gamma$ will result in increased $A^{(t)}$ and $\epsilon^{(t)}$, respectively, for a given $t$. In Fig.~\ref{fig:poc_3_iid_1_gamma_1_lut_50}, we show that {\tt TT-HF} suffers from poor convergence behavior in the extreme non-i.i.d. case when the period of local descents $\tau$ are excessively prolonged, similar to the baseline FL when $\tau = 50$~\cite{wang2019adaptive}. This further emphasizes the importance of Algorithm~\ref{GT} tuning these parameters around Theorem~\ref{thm:subLin_m}'s result.

%Although prolonging the local descent intervals can be achieved via D2D communications conducted in {\tt TT-HF}, as highlighted by the result of Theorem~\ref{thm:subLin}, similar to the current art \cite{wang2019adaptive}, {\tt TT-HF} also suffers from model divergence (poor convergence behavior) when the period of local descents are excessively prolonged. 
%Fig.~\ref{fig:poc_3_iid_1_gamma_1_lut_50} demonstrates this phenomenon in which after increasing the local model training interval beyond 60, marginal performance gain are achieved at the instance of global aggregations. 
% In this section, we study the behavior of {\tt TT-HF} under different hyper-parameters. In particular, in ... we fix the rounds of D2D communications and change the period of local model training interval. This result demonstrate that after a certain threshold, the model training does not exhibit a convincing behavior, which is also highlighted in Theorem 2, via the upper bound on the value of $\tau$. Furthermore, it can be seen that increasing the rounds of D2D communications results in longer tolerable local model training intervals.

\subsection{{\tt TT-HF} with Adaptive Parameter Control}
\label{ssec:control-eval}
We turn now to evaluating the efficacy and analyzing the behavior of {\tt TT-HF} under parameter tuning from Algorithm~\ref{GT}.

%Our control algorithm migrates from static choice of $\alpha$, $\gamma$ and fixed values of rounds of consensus over the clusters $\Gamma$, consensus error coefficient $\phi$, and local model training intervals $\tau$ to a dynamic setting where these parameters are tuned over time with respect to the gradient diversity and divergence of model parameters. Our next set of experiments evaluate the resource efficiency and behavior of our control algorithm.

\subsubsection{Improved resource efficiency compared with baselines}
Fig.~\ref{fig:resource_bar_0} compares the performance of {\tt TT-HF} under our control algorithm with the two baselines: (i) FL with full device participation and $\tau = 1$ (from Sec.~\ref{ssec:conv-eval}), and (ii) FL with $\tau = 20$ but only one device sampled from each cluster for global aggregations.\footnote{The baseline of FL, $\tau = 20$ with full participation is omitted because it results in very poor costs.} The result is shown under different ratios of delays $\frac{\Delta_{\textrm{D2D}}}{\Delta_{\textrm{Glob}}}$ and different ratios of energy consumption $\frac{E_{\textrm{D2D}}}{E_{\textrm{Glob}}}$ between D2D communications and global aggregations.\footnote{These plots are generated for some typical ratios observed in the literature. For example, a similar data rate in D2D and uplink transmission can be achieved via typical values of transmit powers of $10$dbm in D2D mode and $24$dbm in uplink mode~\cite{hmila2019energy,dominic2020joint}, which coincides with a ratio of ${E_{\textrm{D2D}}}/{E_{\textrm{Glob}}}=0.04$. In practice, the actual values are dependent on many environmental factors.} Three metrics are shown: (a) total cost based on the objective of $(\bm{\mathcal{P}})$, (b) total energy consumed, and (c) total delay experienced up to the point where $75\%$ of peak accuracy is reached.

Overall, in (a), we see that {\tt TT-HF} (depicted through the bars) outperforms the baselines (depicted through the horizontal lines) substantially in terms of total cost, by at least 75\% in each case. In (b), we observe that for smaller values of ${E_{\textrm{D2D}}}/{E_{\textrm{Glob}}}$, {\tt TT-HF} lowers the overall power consumption, but after the D2D energy consumption reaches a certain threshold, it does not result in energy savings anymore. The same impact can be observed regarding the delay from (c), i.e., once $\frac{\Delta_{\textrm{D2D}}}{\Delta_{\textrm{Glob}}} \approx 0.1$ there is no longer an advantage in terms of delay. Ratios of $0.1$ for either of these metrics, however, is significantly larger than what is being observed in 5G networks~\cite{hmila2019energy,dominic2020joint}, indicating that {\tt TT-HF} would be effective in practical systems.

%This figure provides a guideline to the network designers to determine the gain of conducting D2D communications. In particular, when D2Ds are low power and low delay, they can result in both significant model performance gain (Figs.~\ref{fig:mnist_poc_1_all} and~\ref{fig:mnist_poc_2_all}) and network costs savings (Fig.~\ref{fig:resource_bar_0}).

\begin{figure}[t]
\hspace{-3mm}
\includegraphics[width=0.99\columnwidth]{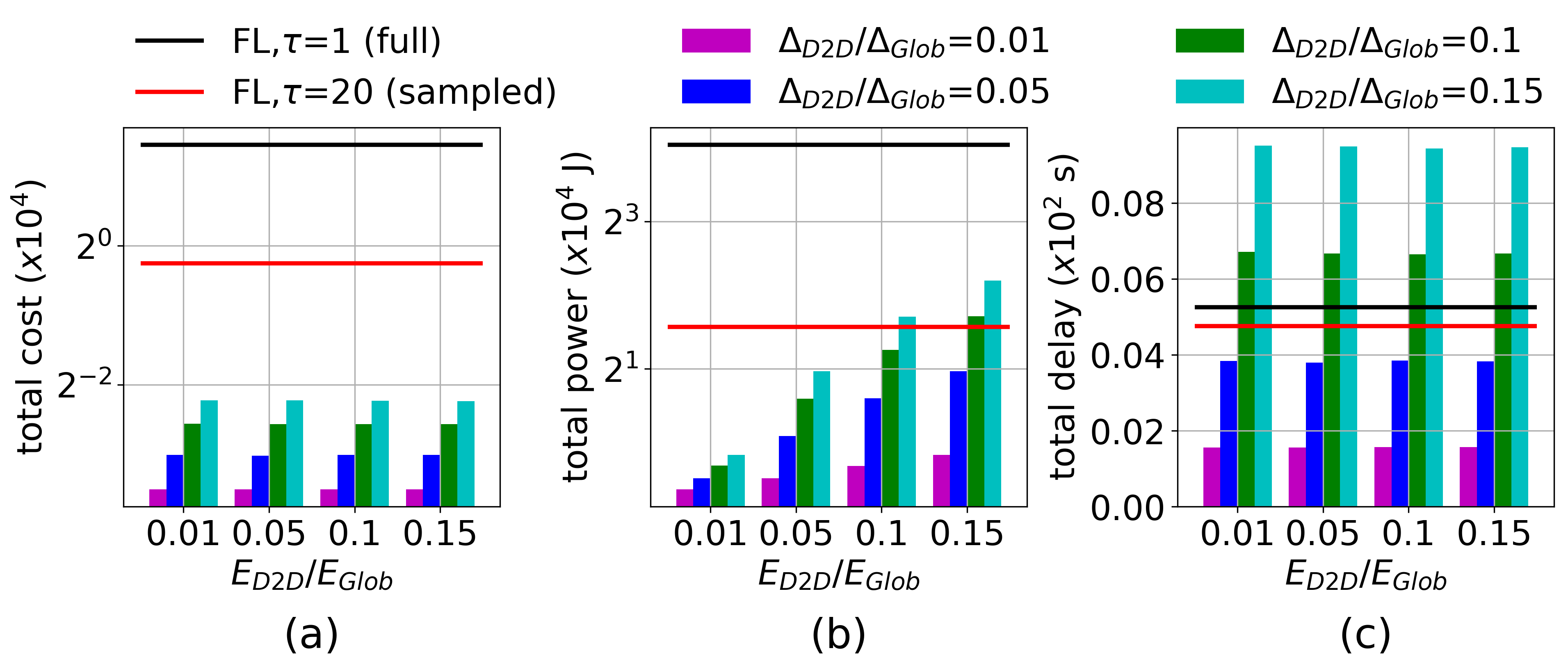}
\centering
\caption{Comparing total (a) cost, (b) power, and (c) delay metrics from the optimization objective in $(\bm{\mathcal{P}})$ achieved by {\tt TT-HF} versus baselines upon reaching $75\%$ of peak accuracy, for different configurations of delay and energy consumption. {\tt TT-HF} obtains a significantly lower total cost in (a). (b) and (c) demonstrate the region under which {\tt TT-HF} attains energy savings and delay gains. SVM is used for classification.}
\label{fig:resource_bar_0}
\vspace{-5mm}
\end{figure}

\subsubsection{Impact of design choices on local model training interval}
We are also interested in how the design weights $c_1, c_2, c_3$ in $(\bm{\mathcal{P}})$ affect the behavior of the control algorithm. In Fig.~\ref{fig:behaviour_of_tau}, we plot the value of $\tau_2$, i.e., the length of the second local model training interval, for different configurations of $c_1$, $c_2$ and $c_3$.\footnote{The specific ranges of values chosen gives comparable objective terms $(a)$, $(b)$, and $(c)$ in $(\bm{\mathcal{P}})$.} The maximum tolerable value of $\tau$ is assumed to be $40$. As we can see, increasing $c_1$ and $c_2$ -- which elevates the priority on minimizing energy consumption and delay, respectively -- results in a longer local model training interval, since D2D communication is more efficient. On the other hand, increasing $c_3$ -- which prioritizes the global model convergence rate -- results in a quicker global aggregation.

\subsection{Main Takeaways}\label{subsec:takeway}
Data heterogeneity in local dataset across local devices can result in considerable performance degradation  of federated learning algorithms. In this case, longer local update periods will result in models that are significantly
biased towards local datasets and degrade the convergence speed of the global model and the resulting model accuracy. By blending federated aggregations with cooperative D2D consensus procedure among local device clusters in {\tt TT-HF}, we effectively decrease the bias of the local models to the local datasets and speed up the convergence at a lower cost (i.e., utilizing low power D2D communications to reduce the frequency of performing global aggregation via uplink transmissions). Due to the low network cost in performing D2D transmission, {\tt TT-HF} provides a practical solution for federated learning to achieve faster convergence or to prolong the local model training interval, leading to delay and energy consumption savings.

Although we develop our algorithm based on federated learning with vanilla SGD local optimizer, our method can benefit other counterparts in the literature. This is due to the fact that, intuitively, conducting D2D communications via the method proposed on this paper reduces the local bias of the nodes' models to their local datasets, which is one of the main challenges faced in federated learning. In Appendix~\ref{app:experiments} we conduct some preliminary experiment to show the impact of our method on FedProx~\cite{sahu2018convergence}.

% \subsubsection{Impact of non-convex ML model}

% It can be seen that as the energy consumption of D2D communications decreases the tendency to remain in the local model training increases via tuning longer period of local model training intervals. The same impact can be seen as the delay of D2D communications start decreasing. 

% \begin{figure*}[t]
% \includegraphics[width=1.01\textwidth]{mnist_125/resource_consumption.eps}
% \centering
% \caption{\textbf{Resource consumption under different configurations}.} 
%     \label{fig:resource_consumption}
% % \vspace{-7mm}
% \end{figure*}

\begin{figure}[t]
\vspace{4mm}
\hspace{-2.5mm}
\includegraphics[width=1\columnwidth]{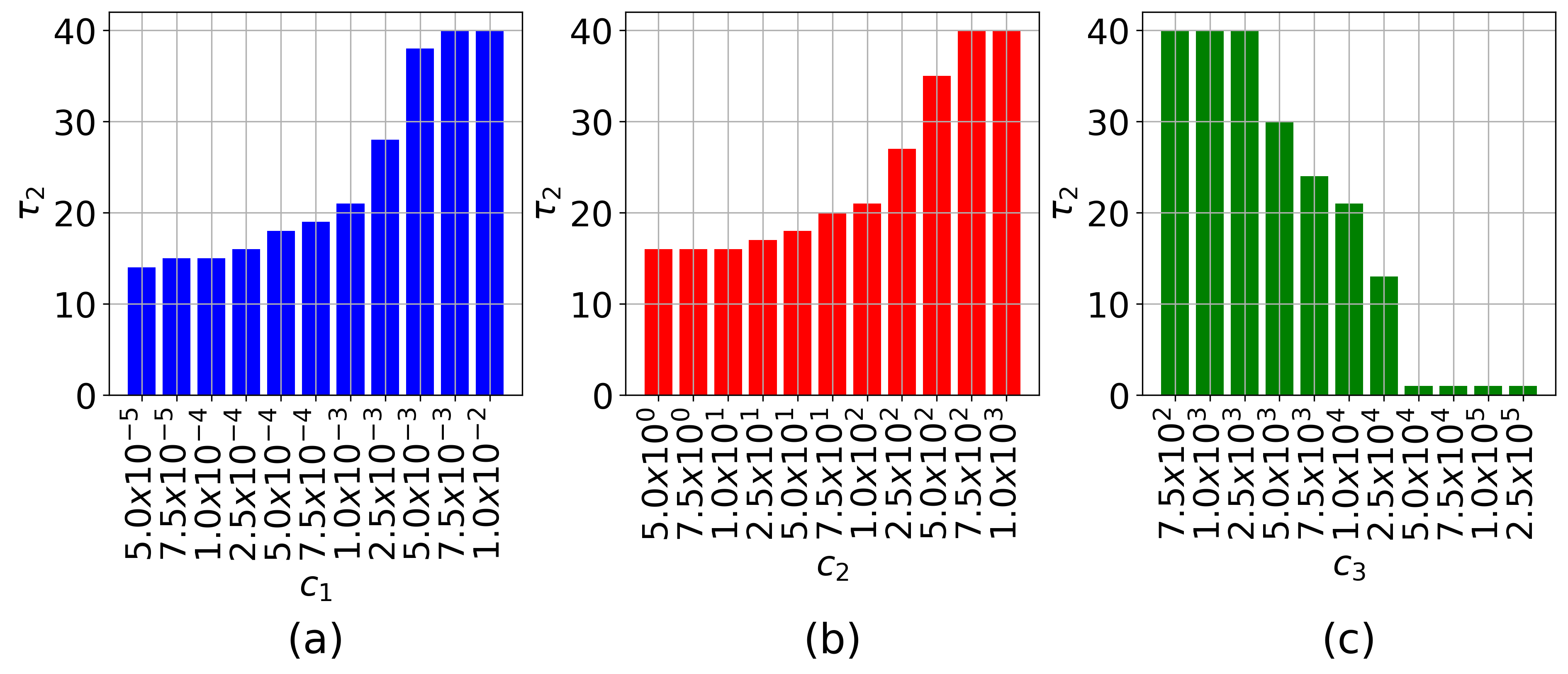}
\centering
\caption{Value of the second local model training interval obtained through $(\bm{\mathcal{P}})$ for different configurations of weighing coefficients $c_1, c_2, c_3$ (default $c_1 = 10^{-3}, c_2=10^2, c_3=10^4$). Higher weight on energy and delay (larger $c_1$ and $c_2$) prolongs the local training period, while higher weight on the global model loss (larger $c_3$) decreases the length, resulting in more rapid global aggregations.
% \textbf{Effect of $c_1$, $c_2$, $c_3$ on $\tau$.} In (a), $c_2=10^2$ and $c_3=10^4$, in (b) $c_1=10^{-3}$ and $c_3=10^4$, and in (c) $c_1=10^{-3}$ and $c_3=10^2$. Increasing $c_1$ and $c_2$ promotes energy and delay efficiency, prolonging the value of $\tau$; increasing $c_3$ promotes higher performance gain obtained through rapid aggregations, decreasing the value of $\tau$.
}
\label{fig:behaviour_of_tau}
\vspace{-5mm}
\end{figure}

\section{Conclusion and Future Work} \label{sec:conclude}
\noindent We proposed {\tt TT-HF}, a methodology which improves the efficiency of federated learning in D2D-enabled wireless networks by augmenting global aggregations with cooperative consensus procedure among device clusters. We conducted a formal convergence analysis of {\tt TT-HF}, resulting in a bound which quantifies the impact of gradient diversity, consensus error, and global aggregation periods on the convergence behavior. Using this bound, we characterized a set of conditions under which {\tt TT-HF} is guaranteed to converge sublinearly with rate of $\mathcal{O}(1/t)$. Based on these conditions, we developed an adaptive control algorithm that actively tunes the device learning rate, cluster consensus rounds, and global aggregation periods throughout the training process. Our experimental results demonstrated the robustness of {\tt TT-HF} against data heterogeneity among edge devices, and its improvement in trained model accuracy, training time, and/or network resource utilization in different scenarios compared to the current art.

There are several avenues for future work. To further enhance the flexibility of {\tt TT-HF}, one may consider (i) heterogeneity in computation capabilities across edge devices, (ii) different communication delays from the clusters to the server, and (iii) wireless interference caused by D2D communications. 
Furthermore, using the set of new techniques we provided to conduct convergence analysis in this paper, we aim to extend our convergence analysis to non-convex settings in future work. This includes obtaining the conditions under which approaching a stationary point of the global loss function is guaranteed, and the rate under which the convergence is achieved.

\bibliographystyle{IEEEtran}
\bibliography{ref}

\newpage

\begingroup
% \let\clearpage\relax 
% \onecolumn %%% For
\onecolumn

\appendices 
\setcounter{lemma}{0}
\setcounter{proposition}{0}
\setcounter{theorem}{0}
\setcounter{definition}{0}
\setcounter{assumption}{0}
\section*{Preliminaries and Notations used in the Proofs}\label{app:notations}
In the following Appendices, in order to increase the tractability of the the expressions inside the proofs, we introduce the  the following scaled parameters: (i) strong convexity denoted by $\tilde{\mu}$, normalized gradient diversity by $\tilde{\delta}$, step size by $\tilde{\eta_t}$, SGD variance $\tilde{\sigma}$, and consensus error inside the clusters $\tilde{\epsilon}_c^{(t)}$ and across the network $\tilde{\epsilon}^{(t)}$ inside the cluster as follows:
\begin{itemize}[leftmargin=5mm]
\item  \textbf{Strong convexity}:
 $F$ is $\mu$-strongly convex, i.e.,
\begin{align} 
    F(\mathbf w_1) \geq  F(\mathbf w_2)&+\nabla F(\mathbf w_2)^\top(\mathbf w_1-\mathbf w_2)+\frac{\tilde{\mu}\beta}{2}\Big\Vert\mathbf w_1-\mathbf w_2\Big\Vert^2,~\forall { \mathbf w_1,\mathbf w_2},
\end{align}
where as compared to~Assumption~\ref{Assump:SmoothStrong}, we considered $\tilde{\mu}=\mu/\beta\in(0,1)$.
\item \textbf{Gradient diversity}: The gradient diversity across the device clusters $c$ is measured via two non-negative constants $\delta,\zeta$ that satisfy 
\begin{align} 
    \Vert\nabla\hat F_c(\mathbf w)-\nabla F(\mathbf w)\Vert
    \leq \sqrt{\beta}\tilde{\delta}+ 2\omega\beta \Vert\mathbf w-\mathbf w^*\Vert,~\forall c, \mathbf w,
\end{align}
where as compared to Assumption~\ref{gradDiv}, we presumed $\tilde{\delta}=\delta/\sqrt{\beta}$ and $\omega=\zeta/(2\beta)\in [0,1]$.

\item \textbf{Step size}: The local updates to compute \textit{intermediate updated local model} at the devices is expressed as follows:
\begin{align} 
    {\widetilde{\mathbf{w}}}_i^{(t)} = 
           \mathbf w_i^{(t-1)}-\frac{\tilde{\eta}_{t-1}}{\beta} \widehat{\mathbf g}_{i}^{(t-1)},~t\in\mathcal T_k,
\end{align}
where we used the scaled in the step size, i.e., ${\tilde{\eta}_{t-1}}=\eta_{t-1}{\beta}$. Also, when we consider decreasing step size, we consider scaled parameter $\tilde{\gamma}$ in the step size as follows: $\frac{\gamma}{t+\alpha}=\frac{\tilde{\gamma}/\beta}{t+\alpha}$ indicating that $\tilde{\gamma}=\gamma\beta$.

\item \textbf{Variance of the noise of the estimated gradient through SGD}:
The variance on the SGD noise is bounded as:
\begin{align}
    \mathbb{E}[\Vert{\mathbf n}_{j}^{(t)}\Vert^2]\leq \beta\tilde{\sigma}^2, \forall j,t,
\end{align}
where we consider scaled SGD noise as: $\tilde{\sigma}^2=\sigma^2/\beta$.

\item \textbf{Average of the consensus error inside cluster $c$ and across the network}: $\epsilon_c^{(t)}$ is an upper bound on the average of the consensus error inside cluster $c$ for time $t$, i.e.,
    \begin{align}
        \frac{1}{s_c}\sum\limits_{i\in \mathcal S_c}\Vert \mathbf{e}_i^{(t)}\Vert^2 \leq (\tilde{\epsilon}_c^{(t)})^2/\beta,
    \end{align}
    where we use the scaled consensus error $(\tilde{\epsilon}_c^{(t)})^2=\beta(\epsilon_c^{(t)})^2$.
    % When the consensus is assumed to be decreasing over time, we use the scaled coefficient $\tilde{\phi}_c$, where ... ....
    % $(\epsilon_c^{(t)})^2=(\tilde{\epsilon}_c^{(t)})^2/\beta=\eta_t^2\phi_c^2=\tilde{\eta}_t^2\tilde{\phi_c}^2/\beta$, which indicates that $\phi_c^2=\tilde{\phi_c}^2/\beta$. 
    Also, in the proofs we use the notation $\epsilon$ to denote the average consensus error across the network defined as $(\epsilon^{(t)})^2=\sum_{c=1}^N\varrho_c(\epsilon_c^{(t)})^2$. When the consensus is assumed to be decreasing over time we use the scaled coefficient $\tilde{\phi}^2=\phi^2/\beta$, resulting in  $(\epsilon^{(t)})^2=\eta_t^2\tilde{\phi}^2\beta $.
    % \begin{align}
    %  \sum_{c=1}^N\varrho_c\phi_c^2=\phi^2,
    % \end{align}
    % where $\phi^2=\tilde{\phi}^2/\beta$.
\end{itemize}
    
  Finally, to track the global model variations, we introduce the instantaneous global model $\hat{\mathbf w}^{(t)}= 
          \sum\limits_{c=1}^N \varrho_c \mathbf w_{n_c}^{(t)}$, where $n_c$ is a node uniformly sampled from cluster $c$. We note that $\hat{\mathbf w}^{(t)}$ is only realized at the server at the instance of the global aggregations.

\section{Proof of Proposition \ref{Local_disperse}} \label{app:Local_disperse}
\begin{proposition} \label{Local_disperse}
    Under Assumptions \ref{beta} and~\ref{assump:SGD_noise}, if $\eta_t=\frac{\gamma}{t+\alpha}$, $\epsilon^{(t)}$ is non-increasing with respect to $t\in \mathcal T_k$, i.e., $\epsilon^{(t+1)}/\epsilon^{(t)} \leq 1$ and $\alpha\geq\max\{\beta\gamma[\frac{\mu}{4\beta}-1+\sqrt{(1+\frac{\mu}{4\beta})^2+2\omega}],\frac{\beta^2\gamma}{\mu}\}$, using {\tt TT-HF} for ML model training, the following upper bound on the expected model dispersion across the clusters holds:
        \begin{align}
            &A^{(t)}\leq
             \frac{16\omega^2}{\mu}[\Sigma_{+,t}]^2 [F(\bar{\mathbf w}(t_{k-1}))-F(\mathbf w^*)]
            +25[\Sigma_{+,t}]^2
            \left[\frac{\sigma^2}{\beta^2}+\frac{\delta^2}{\beta^2}+(\epsilon^{(0)})^2\right], ~~t\in\mathcal{T}_k,
        \end{align}
        where 
        \begin{align}
            &[\Sigma_{+,t}]^2=\left[\sum\limits_{\ell=t_{k-1}}^{t-1}\left(\prod_{j=t_{k-1}}^{\ell-1}(1+\eta_j\beta\lambda_+)\right)\beta\eta_\ell\left(\prod_{j=\ell+1}^{t-1}(1+\eta_j\beta)\right)\right]^2,
        \end{align}
    and
    % \nm{please dont use the symbol $\vartheta$. Replace it with $\mu/\beta$ or normalize $\mu$, so that $\mu\beta$ is the strong convexity param}
    \begin{equation}
        \lambda_+=1-\frac{\mu}{4\beta}+\sqrt{(1+\frac{\mu}{4\beta})^2+2\omega}.
    \end{equation}
\end{proposition}
% \addFL{
% \textbf{Condition for} $\Sigma_{-,t}\geq0$: from the following expression  
% % \begin{align}
% %     \Sigma_{\{+,-,0\},t}=\sum\limits_{\ell=t_{k-1}}^{t-1}\left(\prod_{j=t_{k-1}}^{\ell-1}(1+\eta_j\beta\lambda_{\{+,-,0\}})\right)\beta\eta_\ell\left(\prod_{j=\ell+1}^{t-1}(1+\eta_j\beta)\right),
% % \end{align}
% \begin{align}
%     \Sigma_{-,t}=\sum\limits_{\ell=t_{k-1}}^{t-1}\left(\prod_{j=t_{k-1}}^{\ell-1}(1+\eta_j\beta\lambda_{-})\right)\beta\eta_\ell\left(\prod_{j=\ell+1}^{t-1}(1+\eta_j\beta)\right),
% \end{align}
% where $\lambda_{\pm}=1-\frac{\vartheta}{4}\pm\sqrt{(1+\frac{\vartheta}{4})^2+2\omega}$ and $\lambda_{-}=-\frac{\vartheta+2\omega}{\lambda_+}<0$.
% We can observe that the condition for $\Sigma_{-,t}\geq0$ is $$1+\eta_t\beta\lambda_{-}\geq0,~\forall t \in \mathcal T_k,$$ 
% which can be satisfied when
% $$
% 1+\frac{\gamma\beta\lambda_-}{\alpha}\geq0,
% $$
% or equivalently 
% $$
% \gamma\leq-\frac{\alpha}{\beta\lambda_-}.
% $$
% Since $2\leq\lambda_+\leq3$ and $0\leq\omega\leq1$, we have
% $$
% -\frac{\vartheta+2}{2}\leq-\frac{\vartheta+2\omega}{2}\leq\lambda_-\leq-\frac{\vartheta+2\omega}{3}\leq-\frac{\vartheta}{3}.
% $$
% Then the condition on gamma to satisfy all possible values of $\lambda_-$ to make $\Sigma_{-,t}\geq0$ is
% $$
% \gamma\leq\frac{2\alpha}{\beta(\vartheta+2)}.
% $$
% }
\begin{proof}
% \addFL{

We break down the proof into 3 parts: in Part I we find the relationship between $\Vert\bar{\mathbf w}^{(t)}-\mathbf w^*\Vert$ and $\sum\limits_{c=1}^N\varrho_{c}\Vert\bar{\mathbf w}_c^{(t)}-\bar{\mathbf w}^{(t)}\Vert$, which turns out to form a coupled dynamic system, which is solved in Part II. Finally, Part III draws the connection between $A^{(t)}$ and the solution of the coupled dynamic system and obtains the upper bound on $A^{(t)}$.

\textbf{(Part I) Finding the relationship between $\Vert\bar{\mathbf w}^{(t)}-\mathbf w^*\Vert$ and $\sum\limits_{c=1}^N\varrho_{c}\Vert\bar{\mathbf w}_c^{(t)}-\bar{\mathbf w}^{(t)}\Vert$}:
Using the definition of $\bar{\mathbf w}^{(t+1)}$ given in Definition~\ref{modDisp}, and the notations introduced in Appendix~\ref{app:notations}, we have:  
\begin{equation}
    \bar{\mathbf w}^{(t+1)}= \bar{\mathbf w}^{(t)}-\frac{\tilde{\eta_t}}{\beta}\sum\limits_{c=1}^N\varrho_{c}\frac{1}{s_{c}}\sum\limits_{j\in\mathcal S_{c}} \widehat{\mathbf g}_{j,t},~t\in\mathcal{T}_k.
\end{equation}
Adding and subtracting terms in the above equality gives us:
\begin{align} 
        \bar{\mathbf w}^{(t+1)}-\mathbf w^* =&~ \bar{\mathbf w}^{(t)}-\mathbf w^*-\frac{\tilde{\eta_t}}{\beta} \nabla F(\bar{\mathbf w}^{(t)})
        \nonumber \\&
        -\frac{\tilde{\eta_t}}{\beta}\sum\limits_{c=1}^N\varrho_{c}\frac{1}{s_{c}}\sum\limits_{j\in\mathcal S_{c}} [\widehat{\mathbf g}_{j,t}-\nabla F_j(\mathbf w_j^{(t)})]
        \nonumber \\&
        -\frac{\tilde{\eta_t}}{\beta} \sum\limits_{c=1}^N\varrho_{c}\frac{1}{s_{c}}\sum\limits_{j\in\mathcal S_{c}} [\nabla F_j(\mathbf w_j^{(t)})-\nabla F_j(\bar{\mathbf w}_c^{(t)})]
        \nonumber \\&
        -\frac{\tilde{\eta_t}}{\beta} \sum\limits_{c=1}^N\varrho_{c} [\nabla F_c(\bar{\mathbf w}_c^{(t)})-\nabla F_c(\bar{\mathbf w}^{(t)})].
    \end{align}
    Taking the norm-2 from the both hand sides of the above equality and applying the triangle inequality yields:
    \begin{align} \label{29}
        \Vert\bar{\mathbf w}^{(t+1)}-\mathbf w^*\Vert \leq& \Vert\bar{\mathbf w}^{(t)}-\mathbf w^*-\frac{\tilde{\eta_t}}{\beta} \nabla F(\bar{\mathbf w}^{(t)})\Vert
        +\frac{\tilde{\eta_t}}{\beta} \sum\limits_{c=1}^N\varrho_{c}\frac{1}{s_{c}}\sum\limits_{j\in\mathcal S_{c}} \Vert\mathbf n_{j}^{(t)}\Vert
        \nonumber \\&
        +\frac{\tilde{\eta_t}}{\beta} \sum\limits_{c=1}^N\varrho_{c}\frac{1}{s_{c}}\sum\limits_{j\in\mathcal S_{c}} \Vert\nabla F_j(\mathbf w_j^{(t)})-\nabla F_j(\bar{\mathbf w}_c^{(t)})\Vert
        \nonumber \\&
        +\frac{\tilde{\eta_t}}{\beta} \sum\limits_{c=1}^N\varrho_{c} \Vert\nabla F_c(\bar{\mathbf w}_c^{(t)})-\nabla F_c(\bar{\mathbf w}^{(t)})\Vert.
    \end{align}
 To bound the terms on the right hand side above, we first use the $\mu$-strong convexity and $\beta$-smoothness of $F(\cdot)$, when $\eta_t\leq\frac{\mu}{\beta^2}$, to get 
    \begin{align} \label{eq:stx}
        &\Vert\bar{\mathbf w}^{(t)}-\mathbf w^*-\frac{\tilde{\eta_t}}{\beta} \nabla F(\bar{\mathbf w}^{(t)})\Vert
        \nonumber \\&
        =
        \sqrt{\Vert\bar{\mathbf w}^{(t)}-\mathbf w^*\Vert^2+(\frac{\tilde{\eta_t}}{\beta})^2\Vert\nabla F(\bar{\mathbf w}^{(t)})\Vert^2-\frac{2\tilde{\eta_t}}{\beta}(\bar{\mathbf w}^{(t)}-\mathbf w^*)^\top\nabla F(\bar{\mathbf w}^{(t)})}
        \nonumber \\&
        \overset{(a)}{\leq} 
        \sqrt{(1-2\tilde{\eta}_t\tilde{\mu})\Vert\bar{\mathbf w}^{(t)}-\mathbf w^*\Vert^2+(\frac{\tilde{\eta_t}}{\beta})^2\Vert\nabla F(\bar{\mathbf w}^{(t)})\Vert^2}
        \nonumber \\&
        \overset{(b)}{\leq}
        \sqrt{1-2\tilde{\eta}_t\tilde{\mu}+\tilde{\eta}_t^2}\Vert\bar{\mathbf w}^{(t)}-\mathbf w^*\Vert
        \overset{(c)}{\leq} 
        (1-\frac{\tilde{\eta}_t\tilde{\mu}}{2})\Vert\bar{\mathbf w}^{(t)}-\mathbf w^*\Vert,
    \end{align}
    where $(a)$ results from the property of a strongly convex function, i.e., $(\bar{\mathbf w}^{(t)}-\mathbf w^*)^\top\nabla F(\bar{\mathbf w}^{(t)})\geq\tilde{\mu}\beta\Vert\bar{\mathbf w}^{(t)}-\mathbf w^*\Vert^2$, $(b)$ comes from the property of smooth functions, i.e., $\Vert\nabla F(\bar{\mathbf w}^{(t)})\Vert^2\leq\beta^2\Vert\bar{\mathbf w}^{(t)}-\mathbf w^*\Vert^2$ and the last step $(c)$ follows from the fact that $\tilde{\eta}_t\leq \tilde{\eta}_0$ and assuming $ \tilde{\eta}_0\leq\tilde{\mu}$, implying $\alpha\geq  \tilde{\gamma}/\tilde{\mu}$.
    Also, considering the other terms on the right hand side of~\eqref{29}, using $\beta$-smoothness, we have 
    \begin{align} \label{eq:beta_use}
        \Vert\nabla F_j(\mathbf w_j^{(t)})-\nabla F_j(\bar{\mathbf w}_c^{(t)})\Vert
        \leq
        \beta\frac{1}{s_{c}}\sum\limits_{j\in\mathcal S_{c}}\Vert\mathbf w_j^{(t)}-\bar{\mathbf w}_c^{(t)}\Vert.
    \end{align}
    Moreover, using Condition~\ref{paraDiv}, we get
    \begin{align} \label{eq:e_c}
        \frac{1}{s_{c}}\sum\limits_{j\in\mathcal S_{c}}\Vert\mathbf w_j^{(t)}-\bar{\mathbf w}_c^{(t)}\Vert
        &=
        \frac{1}{s_{c}}\sum\limits_{j\in\mathcal S_{c}}\Vert\mathbf e_j^{(t)}\Vert
        \nonumber \\&\leq
        \sqrt{\frac{1}{s_{c}}\sum\limits_{j\in\mathcal S_{c}}\Vert\mathbf e_j^{(t)}\Vert^2}
        \leq\tilde{\epsilon}_c^{(t)}/\sqrt{\beta}.
    \end{align}
Combining \eqref{eq:beta_use} and \eqref{eq:e_c} gives us:
    \begin{align} \label{eq:epsilon_use}
        \frac{1}{s_{c}}\sum\limits_{j\in\mathcal S_{c}}\Vert\nabla F_j(\mathbf w_j^{(t)})-\nabla F_j(\bar{\mathbf w}_c^{(t)})\Vert
        \leq
        \sqrt{\beta}\tilde{\epsilon}_c^{(t)}.
    \end{align}
    Replacing the result of~\eqref{eq:stx} and~\eqref{eq:epsilon_use} in~\eqref{29} yields:
    \begin{align}
        \Vert\bar{\mathbf w}^{(t+1)}-\mathbf w^*\Vert \leq& 
         (1-\frac{\tilde{\eta}_t\tilde{\mu}}{2})
\Vert\bar{\mathbf w}^{(t)}-\mathbf w^*\Vert
        +\frac{\tilde{\eta}_t}{\beta} \sum\limits_{c=1}^N\varrho_{c}\frac{1}{s_{c}}\sum\limits_{j\in\mathcal S_{c}} \Vert\mathbf n_{j}^{(t)}\Vert
        \nonumber \\&
        +\frac{\tilde{\eta}_t}{\sqrt{\beta}}\sum\limits_{c=1}^N\varrho_{c}\tilde{\epsilon}_{c}^{(t)}
        +\tilde{\eta}_t \sum\limits_{c=1}^N\varrho_{c} \Vert\bar{\mathbf w}_c^{(t)}-\bar{\mathbf w}^{(t)}\Vert.
    \end{align}
Multiplying the both hand sides of the above inequality by $\sqrt{\beta}$ followed by taking square and expectation, we get
\begin{align}
        \mathbb E\left[\sqrt{\beta}\Vert\bar{\mathbf w}^{(t+1)}-\mathbf w^*\Vert\right]^2 \leq& 
         \mathbb E\bigg[\sqrt{\beta}(1-\frac{\tilde{\eta}_t\tilde{\mu}}{2})
\Vert\bar{\mathbf w}^{(t)}-\mathbf w^*\Vert
        +\frac{\tilde{\eta}_t}{\sqrt{\beta}} \sum\limits_{c=1}^N\varrho_{c}\frac{1}{s_{c}}\sum\limits_{j\in\mathcal S_{c}} \Vert\mathbf n_{j}^{(t)}\Vert
        \nonumber \\&
        +{\tilde{\eta}_t}\sum\limits_{c=1}^N\varrho_{c}\tilde{\epsilon}_{c}^{(t)}
        +\tilde{\eta}_t \sqrt{\beta} \sum\limits_{c=1}^N\varrho_{c}  \Vert\bar{\mathbf w}_c^{(t)}-\bar{\mathbf w}^{(t)}\Vert\bigg]^2.
    \end{align}
Taking the square roots from the both hand sides and using Fact \ref{fact:1} (See Appendix \ref{app:fact1}) yields:
\begin{align} \label{eq:fact_x2}
       \sqrt{\beta\mathbb E[\Vert\bar{\mathbf w}^{(t+1)}-\mathbf w^*\Vert^2]} 
       \leq& 
         (1-\frac{\tilde{\eta}_t\tilde{\mu}}{2})
\sqrt{\beta\mathbb E[\Vert\bar{\mathbf w}^{(t)}-\mathbf w^*\Vert^2]}
        +{\tilde{\eta}_t} \tilde{\sigma}
        \nonumber \\&
        +{\tilde{\eta}_t}\sum\limits_{c=1}^N\varrho_{c}\tilde{\epsilon}_{c}^{(t)}
        +\tilde{\eta}_t \sqrt{\beta \Big(\sum\limits_{c=1}^N\varrho_{c} \mathbb E\Vert\bar{\mathbf w}_c^{(t)}-\bar{\mathbf w}^{(t)}\Vert\Big)^2}.
    \end{align}
    We compact~\eqref{eq:fact_x2} and represent it via the following relationship:
\begin{align} \label{eq:x2_x1}
        &x_2^{(t+1)} \leq \begin{bmatrix}
                \tilde{\eta}_t & (1-\frac{\tilde{\eta}_t\tilde{\mu}}{2})
        \end{bmatrix}\mathbf x^{(t)}+\tilde{\eta}_t\left(\tilde{\sigma}+\sum\limits_{c=1}^N\varrho_{c}\tilde{\epsilon}_{c}^{(t)}\right),
    \end{align}
     where $\mathbf x^{(t)}=\begin{bmatrix}
             x_1^{(t)} & x_2^{(t)}
     \end{bmatrix}^\top$, 
         $x_1^{(t)}=\sqrt{\beta\mathbb E[(\sum\limits_{c=1}^N\varrho_{c}\Vert\bar{\mathbf w}_c^{(t)}-\bar{\mathbf w}^{(t)}\Vert)^2]}$, and $x_2^{(t)}=\sqrt{\beta\mathbb E[\Vert\bar{\mathbf w}^{(t)}-\mathbf w^*\Vert^2]}$.

The relationship in~\eqref{eq:x2_x1} reveals the dependency of $x_2^{(t+1)}$ on $x_2^{(t)}$ and $x_1^{(t)}$. 
    To bound $x^{(t)}_1$, we first use the fact that $\bar{\mathbf w}_c^{(t+1)}$ can be written as follows:
    \begin{align} \label{eq:w_c}
        &\bar{\mathbf w}_c^{(t+1)}=\bar{\mathbf w}_c^{(t)}
        -\frac{\tilde{\eta}_t}{\beta}\frac{1}{s_{c}}\sum\limits_{j\in\mathcal S_{c}}\nabla F_j(\mathbf w_j^{(t)})
        -\frac{\tilde{\eta}_t}{\beta}\frac{1}{s_{c}}\sum\limits_{j\in\mathcal S_{c}}\mathbf n_j^{(t)}.
    \end{align}
    Similarly, $\bar{\mathbf w}^{(t+1)}$ can be written as:
    \begin{align} \label{eq:w-}
        &\bar{\mathbf w}^{(t+1)}=
        \bar{\mathbf w}^{(t)}
        -\frac{\tilde{\eta}_t}{\beta}\sum\limits_{d=1}^N\varrho_{d}\frac{1}{s_{d}}\sum\limits_{j\in\mathcal S_{d}}\nabla F_j(\mathbf w_j^{(t)})
        -\frac{\tilde{\eta}_t}{\beta}\sum\limits_{d=1}^N\varrho_{d}\frac{1}{s_{d}}\sum\limits_{j\in\mathcal S_{d}}\mathbf n_j^{(t)}.
    \end{align}
    Combining~\eqref{eq:w_c} and~\eqref{eq:w-} and performing some algebraic manipulations yields:
    \begin{align} 
        &\bar{\mathbf w}_c^{(t+1)}-\bar{\mathbf w}^{(t+1)}=\bar{\mathbf w}_c^{(t)}-\bar{\mathbf w}^{(t)}
        -\frac{\tilde{\eta}_t}{\beta}\frac{1}{s_{c}}\sum\limits_{j\in\mathcal S_{c}}\mathbf n_{j}^{(t)}
        +\frac{\tilde{\eta}_t}{\beta}\sum\limits_{d=1}^N\varrho_{d}\frac{1}{s_{d}}\sum\limits_{j\in\mathcal S_{d}}\mathbf n_{j}^{(t)}
        \nonumber \\&
        -\frac{\tilde{\eta}_t}{\beta}\frac{1}{s_{c}}\sum\limits_{j\in\mathcal S_{c}}\Big[\nabla F_j(\bar{\mathbf w}_j ^{(t)})-\nabla F_j(\bar{\mathbf w}_c ^{(t)})\Big]
        +\frac{\tilde{\eta}_t}{\beta}\sum\limits_{d=1}^N\varrho_{d}\frac{1}{s_{d}}\sum\limits_{j\in\mathcal S_{d}}\Big[\nabla F_j(\bar{\mathbf w}_j ^{(t)})-\nabla F_j(\bar{\mathbf w}_d^{(t)})\Big]
        \nonumber \\&
        -\frac{\tilde{\eta}_t}{\beta}\Big[\nabla\hat F_c(\bar{\mathbf w}_c^{(t)})-\nabla\hat F_c(\bar{\mathbf w}^{(t)})\Big]
        +\frac{\tilde{\eta}_t}{\beta}\sum\limits_{d=1}^N\varrho_{d}\Big[\nabla\hat F_d(\bar{\mathbf w}_d^{(t)})-\nabla\hat F_d(\bar{\mathbf w}^{(t)})\Big]
       \nonumber \\&
        -\frac{\tilde{\eta}_t}{\beta}\Big[\nabla\hat F_c(\bar{\mathbf w}^{(t)})-\nabla F(\bar{\mathbf w}^{(t)})\Big].
    \end{align}   
    Taking the norm-2 of the both hand sides of the above equality and applying the triangle inequality gives us
    \begin{align} \label{eq:tri_wc}
        &\Vert\bar{\mathbf w}_c^{(t+1)}-\bar{\mathbf w}^{(t+1)}\Vert\leq
        \Vert\bar{\mathbf w}_c^{(t)}-\bar{\mathbf w}^{(t)}\Vert
        +\frac{\tilde{\eta}_t}{\beta}\Vert\frac{1}{s_{c}}\sum\limits_{j\in\mathcal S_{c}}\mathbf n_{j}^{(t)}\Vert
        +\frac{\tilde{\eta}_t}{\beta}\Vert\sum\limits_{d=1}^N\varrho_{d}\frac{1}{s_{d}}\sum\limits_{j\in\mathcal S_{d}}\mathbf n_{j}^{(t)}\Vert
        \nonumber \\&
        +\frac{\tilde{\eta}_t}{\beta}\frac{1}{s_{c}}\sum\limits_{j\in\mathcal S_{c}}\Vert\nabla F_j(\bar{\mathbf w}_j ^{(t)})-\nabla F_j(\bar{\mathbf w}_c ^{(t)})\Vert
        +\frac{\tilde{\eta}_t}{\beta}\sum\limits_{d=1}^N\varrho_{d}\frac{1}{s_{d}}\sum\limits_{j\in\mathcal S_{d}}\Vert\nabla F_j(\bar{\mathbf w}_j ^{(t)})-\nabla F_j(\bar{\mathbf w}_d^{(t)})\Vert
        \nonumber \\&
        +\frac{\tilde{\eta}_t}{\beta}\Vert\nabla\hat F_c(\bar{\mathbf w}_c^{(t)})-\nabla\hat F_c(\bar{\mathbf w}^{(t)})\Vert
        +\frac{\tilde{\eta}_t}{\beta}\sum\limits_{d=1}^N\varrho_{d}\Vert\nabla\hat F_d(\bar{\mathbf w}_d^{(t)})-\nabla\hat F_d(\bar{\mathbf w}^{(t)})\Vert
        \nonumber \\&
        +\frac{\tilde{\eta}_t}{\beta}\Vert\nabla\hat F_c(\bar{\mathbf w}^{(t)})-\nabla F(\bar{\mathbf w}^{(t)})\Vert.
    \end{align}   
    Using $\beta$-smoothness of $F_j(\cdot)$, $\forall j$, and $\hat F_c(\cdot)$, $\forall c$, Definition \ref{gradDiv} and Condition~\ref{paraDiv}, we further bound the right hand side of~\eqref{eq:tri_wc} to get
\begin{align} \label{eq:tri_wc2}
        &\Vert\bar{\mathbf w}_c^{(t+1)}-\bar{\mathbf w}^{(t+1)}\Vert\leq
        (1+\tilde{\eta}_t)\Vert\bar{\mathbf w}_c^{(t)}-\bar{\mathbf w}^{(t)}\Vert
        +2\omega\tilde{\eta}_t\Vert\bar{\mathbf w}^{(t)}-\mathbf w^*\Vert
                +\tilde{\eta}_t\sum\limits_{d=1}^N\varrho_{d}\Vert\bar{\mathbf w}_d^{(t)}-\bar{\mathbf w}^{(t)}\Vert
        \nonumber \\&
        +\frac{\tilde{\eta}_t}{\beta}\Vert\frac{1}{s_{c}}\sum\limits_{j\in\mathcal S_{c}}\mathbf n_{j}^{(t)}\Vert
        +\frac{\tilde{\eta}_t}{\beta}\Vert\sum\limits_{d=1}^N\varrho_{d}\frac{1}{s_{d}}\sum\limits_{j\in\mathcal S_{d}}\mathbf n_{j}^{(t)}\Vert
        \nonumber \\&
        +\tilde{\eta}_t\frac{1}{s_{c}}\sum\limits_{j\in\mathcal S_{c}}
        \Vert\bar{\mathbf w}_j ^{(t)}-\bar{\mathbf w}_c ^{(t)}\Vert
        +\tilde{\eta}_t\sum\limits_{d=1}^N\varrho_{d}\frac{1}{s_{d}}\sum\limits_{j\in\mathcal S_{d}}
        \Vert\bar{\mathbf w}_j ^{(t)}-\bar{\mathbf w}_d^{(t)}\Vert
        +\frac{\tilde{\eta}_t}{\sqrt{\beta}}\tilde{\delta}.
    \end{align}   
    Using~\eqref{eq:e_c} we have $\frac{1}{s_{d}}\sum\limits_{j\in\mathcal S_{d}}
        \Vert\bar{\mathbf w}_j ^{(t)}-\bar{\mathbf w}_d^{(t)}\Vert\leq\frac{\tilde{\epsilon}_{d}^{(t)}}{\sqrt{\beta}}$, and thus~\eqref{eq:tri_wc2} can be written as
        \begin{align} \label{eq:wc_w}
        &\Vert\bar{\mathbf w}_c^{(t+1)}-\bar{\mathbf w}^{(t+1)}\Vert\leq
        (1+\tilde{\eta}_t)\Vert\bar{\mathbf w}_c^{(t)}-\bar{\mathbf w}^{(t)}\Vert
        +2\omega\tilde{\eta}_t\Vert\bar{\mathbf w}^{(t)}-\mathbf w^*\Vert
                +\tilde{\eta}_t\sum\limits_{d=1}^N\varrho_{d}\Vert\bar{\mathbf w}_d^{(t)}-\bar{\mathbf w}^{(t)}\Vert
        \nonumber \\&
        +\frac{\tilde{\eta}_t}{\beta}\Vert\frac{1}{s_{c}}\sum\limits_{j\in\mathcal S_{c}}\mathbf n_{j}^{(t)}\Vert
        +\frac{\tilde{\eta}_t}{\beta}\Vert\sum\limits_{d=1}^N\varrho_{d}\frac{1}{s_{d}}\sum\limits_{j\in\mathcal S_{d}}\mathbf n_{j}^{(t)}\Vert
        +\frac{\tilde{\eta}_t}{\sqrt{\beta}}\left(\tilde{\epsilon}_{c}^{(t)}+\sum\limits_{d=1}^N\varrho_{d}\tilde{\epsilon}_{d}^{(t)}+\tilde{\delta}\right).
    \end{align} 
  Taking the weighted sum $\sum\limits_{c=1}^N\varrho_c$ from the both hand sides of the above inequality gives us
    \begin{align} 
        &\sum\limits_{c=1}^N\varrho_c\Vert\bar{\mathbf w}_c^{(t+1)}-\bar{\mathbf w}^{(t+1)}\Vert\leq
        (1+2\tilde{\eta}_t)\sum\limits_{c=1}^N\varrho_c\Vert\bar{\mathbf w}_c^{(t)}-\bar{\mathbf w}^{(t)}\Vert
        +2\omega\tilde{\eta}_t\Vert\bar{\mathbf w}^{(t)}-\mathbf w^*\Vert
        \nonumber \\&
        +\frac{\tilde{\eta}_t}{\beta}\sum\limits_{c=1}^N\varrho_c\Vert\frac{1}{s_{c}}\sum\limits_{j\in\mathcal S_{c}}\mathbf n_{j}^{(t)}\Vert
        +\frac{\tilde{\eta}_t}{\beta}\Vert\sum\limits_{d=1}^N\varrho_{d}\frac{1}{s_{d}}\sum\limits_{j\in\mathcal S_{d}}\mathbf n_{j}^{(t)}\Vert
        +\frac{\tilde{\eta}_t}{\sqrt{\beta}}\left(2\sum\limits_{d=1}^N\varrho_{d}\tilde{\epsilon}_{d}^{(t)}+\tilde{\delta}\right).
    \end{align}  
    
    Multiplying the both hand side of the above inequality by $\sqrt{\beta}$, followed by taking square and expectation, using a similar procedure used to obtain~\eqref{eq:fact_x2}, we get the bound on $x_1^{(t+1)}$ as follows:
     \begin{align} \label{eq:fact_x1}
        &
x_1^{(t+1)}\leq
\begin{bmatrix}
        (1+2\tilde{\eta}_t) & 2\omega\tilde{\eta}_t
\end{bmatrix}
\mathbf x^{(t)}
        +\tilde{\eta}_t\left(2\sum\limits_{d=1}^N\varrho_{d}\tilde{\epsilon}_{d}^{(t)}+\tilde{\delta}+2\tilde{\sigma}\right).
    \end{align}

\textbf{(Part II) Solving the coupled dynamic system:}
To bound $\mathbf x^{(t)}$, we need to bound $x_1^{(t)}$ and $x_2^{(t)}$, where $x_2^{(t)}$ is given by~\eqref{eq:x2_x1}, which is dependent on $\mathbf x^{(t-1)}$. Also, $x_1^{(t)}$ is given in~\eqref{eq:fact_x1} which is dependent on  $\mathbf x^{(t-1)}$. This leads to a \textit{coupled dynamic system} where $\mathbf x^{(t)}$ can be expressed in a compact form as follows:
    \begin{align} \label{69}
        \mathbf x^{(t+1)}
        &\leq 
        [\mathbf I+\tilde{\eta}_t\mathbf B]\mathbf  x^{(t)}
        +\tilde{\eta}_t\mathbf z, 
    \end{align}
    where $\mathbf  x^{(t_{k-1})}=\mathbf e_2\sqrt{\beta}\Vert\bar{\mathbf w}^{(t_{k-1})}-\mathbf w^*\Vert$,
    $
\mathbf z= \begin{bmatrix}
        2 & 1
\end{bmatrix}^\top[\tilde{\sigma}+\sum\limits_{d=1}^N\varrho_{d}\tilde{\epsilon}_{d}^{(0)}]
        +\mathbf e_1\tilde{\delta}
    $
    ,    
    $\mathbf B=\begin{bmatrix} 2 & 2\omega
    \\ 1 & -\tilde{\mu}/2\end{bmatrix}$, $\mathbf e_1=\begin{bmatrix}1\\0\end{bmatrix}$ and $\mathbf e_2=\begin{bmatrix}0\\1\end{bmatrix}$.
    We aim to characterize an upper bound on $\mathbf x^{(t)}$ denoted by $\bar{\mathbf x}^{(t)}$, where 
     \begin{align} \label{69}
 \bar{\mathbf x}^{(t+1)}
= 
        [\mathbf I+\tilde{\eta}_t\mathbf B]\bar{\mathbf x}^{(t)}
        +\tilde{\eta}_t\mathbf z.
    \end{align}
 To solve the coupled dynamic, we use the eigen-decomposition on $\mathbf B$:
    $\mathbf B=\mathbf U\mathbf D\mathbf U^{-1}$, where
    $$\mathbf D=\begin{bmatrix} \lambda_+ & 0
    \\ 0 & \lambda_-\end{bmatrix},\ 
    \mathbf U=\begin{bmatrix} \omega & \omega
    \\ \frac{\lambda_+}{2}-1 & \frac{\lambda_-}{2}-1\end{bmatrix},\ 
\mathbf U^{-1}=\frac{1}{
    \omega\sqrt{(1+\tilde{\mu}/4)^2+2\omega}
    }\begin{bmatrix} 1-\frac{\lambda_-}{2} & \omega
    \\ \frac{\lambda_+}{2}-1 &-\omega\end{bmatrix}$$
        and the eigenvalues in $\mathbf D$ are given by
        \begin{align} \label{eq:lambda+}
            \lambda_+ =1-\tilde{\mu}/4+\sqrt{(1+\tilde{\mu}/4)^2+2\omega}>0
        \end{align}
    and 
    \begin{align} \label{eq:lambda-}
        \lambda_-=
        1-\tilde{\mu}/4-\sqrt{(1+\tilde{\mu}/4)^2+2\omega}
        =
        -\frac{\tilde{\mu}+2\omega}{\lambda_+}<0
    \end{align}
    
    To further compact the relationship in~\eqref{69}, we introduce a variable $\mathbf f^{(t)}$, where $\mathbf f^{(t)}=\mathbf U^{-1}\bar{\mathbf x}^{(t)}+\mathbf U^{-1}\mathbf B^{-1}\mathbf z$, satisfying the following recursive expression:
    % Then, letting $\mathbf f_t=\mathbf U^{-1}\bar{\mathbf x}^{(t)}+\mathbf U^{-1}\mathbf B^{-1}\mathbf z$
    % we obtain
    \begin{align} 
        \mathbf f^{(t+1)}
        &= 
         [\mathbf I+\tilde{\eta}_t\mathbf D]\mathbf f^{(t)}.
    \end{align}
  Recursive expansion of the right hand side of the above equality yields: 
    \begin{align} 
        \mathbf f^{(t)}
        &= 
         \prod_{\ell=t_{k-1}}^{t-1}[\mathbf I+\tilde{\eta}_\ell\mathbf D]\mathbf f^{(t_{k-1})}.
    \end{align}
    Using the fact that
    $\bar{\mathbf x}^{(t)}=\mathbf U\mathbf f^{(t)}-\mathbf B^{-1}\mathbf z$,
    we obtain the following expression for $\bar{\mathbf x}^{(t)}$:
    \begin{align} 
    \bar{\mathbf x}^{(t)}=
\mathbf U\prod_{\ell=t_{k-1}}^{t-1}(\mathbf I+\tilde{\eta}_\ell\mathbf D)
         \mathbf U^{-1}\mathbf e_2\Vert\bar{\mathbf w}^{(t_{k-1})}-\mathbf w^*\Vert
+\mathbf U\left[\prod_{\ell=t_{k-1}}^{t-1}(\mathbf I+\tilde{\eta}_\ell\mathbf D)-\mathbf I\right]
         \mathbf U^{-1}\mathbf B^{-1}\mathbf z.
    \end{align}
    \textbf{(Part III) Finding the connection between $A^{(t)}$ and ${\mathbf x}^{(t)}$ and the expression for $A^{(t)}$}:
To bound the model dispersion across the clusters, we revisit~\eqref{eq:wc_w}, where we multiply its both hand side by $\sqrt{\beta}$, followed by taking square and expectation and follow a similar procedure used to obtain~\eqref{eq:fact_x2} to get:
%  Bounding the left hand side of~\eqref{eq:expanded} requires bounding $y^{(t)}$ appearing on the right hand side. To bound $y^{(t)}$ 
% Considering the definition of $A^(t)$ that we aimed to bound, we first use ..and 
% Multiplying the both hand side of the above inequality by $\sqrt{\beta}$, followed by taking square and expectation, using a similar procedure used to obtain~\eqref{eq:fact_x2}, we get:
    \begin{align} \label{eq:recurse}
        \sqrt{\beta\mathbb E[\Vert\bar{\mathbf w}_c^{(t+1)}-\bar{\mathbf w}^{(t+1)}\Vert^2]}
        \leq&
        (1+\tilde{\eta}_t)\sqrt{\beta\mathbb E[\Vert\bar{\mathbf w}_c^{(t)}-\bar{\mathbf w}^{(t)}\Vert^2]}
        +\tilde{\eta}_t y^{(t)}
        \nonumber \\&
        +\tilde{\eta}_t[\tilde{\epsilon}_{c}^{(t)}+\sum\limits_{d=1}^N\varrho_{d}\tilde{\epsilon}_{d}^{(t)}+\tilde{\delta}+2\tilde{\sigma}],
    \end{align}  
    where $y^{(t)}=\begin{bmatrix}
        1 & 2\omega
    \end{bmatrix}\mathbf x^{(t)}$. Recursive expansion of~\eqref{eq:recurse} yields: 
    \begin{align} \label{eq:expanded}
        &
\sqrt{\beta\mathbb E[\Vert\bar{\mathbf w}_c^{(t)}-\bar{\mathbf w}^{(t)}\Vert^2]}
        \leq
        \sum_{\ell=t_{k-1}}^{t-1}\tilde{\eta}_\ell\prod_{j=\ell+1}^{t-1}(1+\tilde{\eta}_j)
        y^{(\ell)}
        \nonumber \\&
        +\sum_{\ell=t_{k-1}}^{t-1}\tilde{\eta}_\ell\prod_{j=\ell+1}^{t-1}(1+\tilde{\eta}_j)
        [\tilde{\epsilon}_{c}^{(0)}+\sum\limits_{d=1}^N\varrho_{d}\tilde{\epsilon}_{d}^{(0)}+\tilde{\delta}+2\tilde{\sigma}].
    \end{align}    
The expression in~\eqref{eq:expanded} reveals the dependency of the difference between the model in one cluster $c$ and the global average of models, i.e., the left hand side, on $y^{(t)}$ which by itself depends on $\mathbf x^{(t)}$. Considering the fact that $A^{(t)}\triangleq\mathbb E\left[\sum\limits_{c=1}^N\varrho_{c}\big\Vert\bar{\mathbf w}_c^{(t)}-\bar{\mathbf w}^{(t)}\big\Vert^2\right]$, the aforementioned dependency implies the dependency of $A^{(t)}$ on $\mathbf x^{(t)}$.

So, the key to obtain $A^{(t)}$ is to bound $y^{(t)}$, which can be expressed as follows:
% \triangleq \bar y^{(t)}
\begin{align} \label{eq:yt}
    y^{(t)}&=\begin{bmatrix}
        1 & 2\omega
    \end{bmatrix}\mathbf x^{(t)}\leq 
    \begin{bmatrix}
        1 & 2\omega
    \end{bmatrix}\bar{\mathbf x}^{(t)}
   \nonumber\\&=
  [g_1\Pi_{+,t}+g_2\Pi_{-,t}]\sqrt{\beta}\Vert\bar{\mathbf w}^{(t_{k-1})}-\mathbf w^*\Vert
  \nonumber \\&
+[g_3(\Pi_{+,t}-\Pi_{0,t})+g_4(\Pi_{-,t}-\Pi_{0,t})]
[\tilde{\sigma}+\sum\limits_{d=1}^N\varrho_{d}\tilde{\epsilon}_{d}^{(0)}] \nonumber\\&
+[g_5(\Pi_{+,t}-\Pi_{0,t})+g_6(\Pi_{-,t}-\Pi_{0,t})]\tilde{\delta},
    \end{align}
    where we define
    $
    \Pi_{\{+,-,0\},t}=\prod_{\ell=t_{k-1}}^{t-1}[1+\tilde{\eta}_\ell\lambda_{\{+,-,0\}}],\ 
    $ with $\lambda_+$ given by~\eqref{eq:lambda+} and $\lambda_-$ given by~\eqref{eq:lambda-} and $\lambda_0=0$. Also, the constants $g_1$, $g_2$, $g_3$, $g_4$, $g_5$, and $g_6$ are given by: 
    $$
        g_1\triangleq
        \begin{bmatrix}
        1 & 2\omega
    \end{bmatrix}\mathbf U\mathbf e_1\mathbf e_1^\top\mathbf U^{-1}\mathbf e_2
        =
        \omega
        \left[1-\frac{\tilde{\mu}/4}{\sqrt{(1+\tilde{\mu}/4)^2+2\omega}}\right]>0,
    $$
    $$
g_2\triangleq
    \begin{bmatrix}
        1 & 2\omega
    \end{bmatrix}\mathbf U\mathbf e_2\mathbf e_2^\top\mathbf U^{-1}\mathbf e_2
    =
    \omega\left[1+\frac{\tilde{\mu}/4}{\sqrt{(1+\tilde{\mu}/4)^2+2\omega}}\right]
    =g_2=2\omega-g_1>0,
    $$
    $$
g_3\triangleq
    \begin{bmatrix}
        1 & 2\omega
    \end{bmatrix}\mathbf U\mathbf e_1\mathbf e_1^\top\mathbf U^{-1}\mathbf B^{-1}\begin{bmatrix}
        2 & 1
    \end{bmatrix}^\top
    =
\frac{1}{2}+\frac{1+\tilde{\mu}/4+2\omega}{2\sqrt{(1+\tilde{\mu}/4)^2+2\omega}}
    =g_3>1,
$$
$$
g_4\triangleq
    \begin{bmatrix}
        1 & 2\omega
    \end{bmatrix}\mathbf U\mathbf e_2\mathbf e_2^\top\mathbf U^{-1}\mathbf B^{-1}\begin{bmatrix}
        2 & 1
    \end{bmatrix}^\top
=
\frac{1}{2}-\frac{1+\tilde{\mu}/4+2\omega}{2\sqrt{(1+\tilde{\mu}/4)^2+2\omega}},
$$$$
=
-\omega\frac{1+2\omega+\tilde{\mu}/2}{\sqrt{(1+\tilde{\mu}/4)^2+2\omega}}
\frac{1}{\sqrt{(1+\tilde{\mu}/4)^2+2\omega}+[1+\tilde{\mu}/4+2\omega]}
=1- g_3<0,
$$

$$
g_5
\triangleq
    \begin{bmatrix}
        1 & 2\omega
    \end{bmatrix}\mathbf U\mathbf e_1\mathbf e_1^\top\mathbf U^{-1}\mathbf B^{-1}\mathbf e_1
=
\frac{1}{[\tilde{\mu}+2\omega]\sqrt{(1+\tilde{\mu}/4)^2+2\omega}}
$$
$$
\cdot\frac{
\frac{\tilde{\mu}}{2}(1+\tilde{\mu}/4)^2
+\omega[1+\frac{5\tilde{\mu}}{4}+\tilde{\mu}^2/8]
+2\omega^2
+\sqrt{(1+\tilde{\mu}/4)^2+2\omega}[\frac{\tilde{\mu}}{2}(1+\tilde{\mu}/4)+\omega[1+\tilde{\mu}/2]]
}{1+\tilde{\mu}/4+\sqrt{(1+\tilde{\mu}/4)^2+2\omega}}
>0,
$$

$$
g_6\triangleq
\begin{bmatrix}
        1 & 2\omega
    \end{bmatrix}\mathbf U\mathbf e_2\mathbf e_2^\top\mathbf U^{-1}\mathbf B^{-1}\mathbf e_1
$$$$
=
\frac{\omega}{[\tilde{\mu}+2\omega]\sqrt{(1+\tilde{\mu}/4)^2+2\omega}}
\frac{1+\frac{3\tilde{\mu}}{4}+2\omega+\sqrt{(1+\tilde{\mu}/4)^2+2\omega}}{1+\tilde{\mu}/4+\sqrt{(1+\tilde{\mu}/4)^2+2\omega}}
 =
\frac{\tilde{\mu}/2+2\omega}{\tilde{\mu}+2\omega}-g_5>0.
$$
% \nm{these coeff can later be bounded to simplify}

Revisiting~\eqref{eq:recurse} with the result of~\eqref{eq:yt} gives us:
\begin{align} \label{eq:mid_goal}
        &
        \sqrt{\beta\mathbb E[\Vert\bar{\mathbf w}_c^{(t)}-\bar{\mathbf w}^{(t)}\Vert^2]}
        \leq
2\omega\frac{g_1\Sigma_{+,t}+g_2\Sigma_{-,t}}{g_1+g_2}\sqrt{\beta}\Vert\bar{\mathbf w}(t_{k-1})-\mathbf w^*\Vert
         \nonumber\\&
+[\Sigma_{+,t}+(g_3-1)(\Sigma_{+,t}-\Sigma_{-,t})]
[\tilde{\sigma}+\sum\limits_{d=1}^N\varrho_{d}\tilde{\epsilon}_{d}^{(0)}]
\nonumber\\&
+\frac{\tilde{\mu}/2}{\tilde{\mu}+2\omega}[\frac{g_5}{g_5+g_6}\Sigma_{+,t}+\frac{g_6}{g_5+g_6}\Sigma_{-,t}+\Sigma_{0,t}]\tilde{\delta}
+\Sigma_{0,t}[\tilde{\epsilon}_{c}^{(0)}+\tilde{\sigma}],
    \end{align}
    where we used the facts that $g_3+g_4=1$, $g_5+g_6=\frac{\tilde{\mu}/2+2\omega}{\tilde{\mu}+2\omega}$,
    $g_1+g_2=2\omega$,  and $g_3>1$, and defined  $\Sigma_{+,t}$, $\Sigma_{-,t}$, and $\Sigma_{0,t}$ as follows:
    $$
 \Sigma_{\{+,-,0\},t}
 =\sum_{\ell=t_{k-1}}^{t-1}\tilde{\eta}_\ell\prod_{j=\ell+1}^{t-1}(1+\tilde{\eta}_j)\Pi_{\{+,-,0\},\ell}
  =\sum_{\ell=t_{k-1}}^{t-1}
  \left[\prod_{j=t_{k-1}}^{\ell-1}(1+\tilde{\eta}_j\lambda_{\{+,-,0\}})\right]
  \tilde{\eta}_\ell
  \left[\prod_{j=\ell+1}^{t-1}(1+\tilde{\eta}_j)\right].
    $$

    We now demonstrate that: (i)
    $\Sigma_{-,t}\leq \Sigma_{+,t}$, (ii)
    $\Sigma_{0,t}\leq \Sigma_{+,t}$, and (iii)
    $\Sigma_{-,t}\geq 0$.

To prove $\Sigma_{-,t}\leq \Sigma_{+,t}$, we upper bound  $\Sigma_{-,t}$ as follows:
\begin{align}
 \Sigma_{-,t}
 &\leq\sum_{\ell=t_{k-1}}^{t-1}
  \left[\prod_{j=t_{k-1}}^{\ell-1}|1+\tilde{\eta}_j \lambda_{-}|\right]
  \tilde{\eta}_\ell
  \left[\prod_{j=\ell+1}^{t-1}(1+\tilde{\eta}_j)\right]
\nonumber\\&
   \leq\sum_{\ell=t_{k-1}}^{t-1}
  \left[\prod_{j=t_{k-1}}^{\ell-1}(1+\tilde{\eta}_j \lambda_{+})\right]
  \tilde{\eta}_\ell
  \left[\prod_{j=\ell+1}^{t-1}(1+\tilde{\eta}_j)\right]=\Sigma_{+,t}.
    \end{align}
    Similarly it can be shown that $\Sigma_{0,t}\leq \Sigma_{+,t}$ since $\lambda_+>1$.

    To prove $\Sigma_{-,t}\geq 0$, it is sufficient to impose the condition
    $(1+\tilde{\eta}_j \lambda_-)\geq 0,\forall j$, i.e.
    $(1+\tilde{\eta}_0 \lambda_-)\geq 0$, which implies
    $
    \alpha\geq\tilde{\gamma}[\tilde{\mu}/4-1+\sqrt{(1+\tilde{\mu}/4)^2+2\omega}].
    $
    % \nm{otherwise, we can use the a different inequality:
    % $\Sigma_{-,t}\geq 2\Sigma_{0,t}-\Sigma_{+,t}$ (TBP)}
    
   Considering~\eqref{eq:mid_goal} with the above mentioned properties for $\Sigma_{-,t}$, $\Sigma_{+,t}$, and $\Sigma_{0,t}$, we get:
    \begin{align} \label{eq:midgoal2}
        \sqrt{\beta\mathbb E[\Vert\bar{\mathbf w}_c^{(t)}-\bar{\mathbf w}^{(t)}\Vert^2]}
        \leq&
2\omega\Sigma_{+,t}\sqrt{\beta}\Vert\bar{\mathbf w}^{(t_{k-1})}-\mathbf w^*\Vert
        \nonumber  \\&
+g_3\Sigma_{+,t}
[\tilde{\sigma}+\sum\limits_{d=1}^N\varrho_{d}\tilde{\epsilon}_{d}^{(0)}]
 \nonumber \\&
+\frac{\tilde{\mu}}{\tilde{\mu}+2\omega}\Sigma_{+,t}\tilde{\delta}
+\Sigma_{+,t}[\tilde{\sigma}+\tilde{\epsilon}_{c}^{(0)}].
    \end{align}
    Moreover, since $\frac{\tilde{\mu}}{\tilde{\mu}+2\omega}\leq 1$,
    $\sum\limits_{d=1}^N\varrho_{d}\tilde{\epsilon}_{d}^{(0)}= \tilde{\epsilon}^{(0)}$ and
        $
g_3\leq\frac{1+\sqrt{3}}{2}
$
(since $g_3$ is increasing with respect to $\omega$ and decreasing with respect to $\tilde{\mu}$), from~\eqref{eq:midgoal2} we obtain
 \begin{align} 
        \sqrt{\beta\mathbb E[\Vert\bar{\mathbf w}_c^{(t)}-\bar{\mathbf w}^{(t)}\Vert^2]}
        \leq &
2\omega\Sigma_{+,t}\sqrt{\beta}\Vert\bar{\mathbf w}^{(t_{k-1})}-\mathbf w^*\Vert
        %  \nonumber \\&
+\Sigma_{+,t}
\left[\frac{3+\sqrt{3}}{2}\tilde{\sigma}+\frac{1+\sqrt{3}}{2}\tilde{\epsilon}^{(0)}+\tilde{\epsilon}_{c}^{(0)}+\tilde{\delta}\right].
    \end{align}
    Taking the square of the both hand sides followed by taking the weighted sum $\sum_{c=1}^N\varrho_c$, we get:
    \begin{align} \label{eq:Amid}
        &
 \beta A^{(t)}=\beta\mathbb E\left[\sum_{c=1}^N\varrho_c\Vert\bar{\mathbf w}_c^{(t)}-\bar{\mathbf w}^{(t)}\Vert^2\right]
        \leq
8\omega^2[\Sigma_{+,t}]^2\beta\Vert\bar{\mathbf w}^{(t_{k-1})}-\mathbf w^*\Vert^2
\nonumber \\&
+2[\Sigma_{+,t}]^2\sum_{c=1}^N\varrho_c
\left[\frac{3+\sqrt{3}}{2}\tilde{\sigma}+\frac{1+\sqrt{3}}{2}\tilde{\epsilon}^{(0)}+\tilde{\epsilon}_{c}^{(0)}+\tilde{\delta}\right]^2
\nonumber \\&
\leq
8\omega^2[\Sigma_{+,t}]^2\beta\Vert\bar{\mathbf w}^{(t_{k-1})}-\mathbf w^*\Vert^2
+2[\Sigma_{+,t}]^2
\left[\frac{3+\sqrt{3}}{2}\tilde{\sigma}+\frac{3+\sqrt{3}}{2}\tilde{\epsilon}^{(0)}+\tilde{\delta}
\right]^2
\nonumber \\&
\leq
8\omega^2[\Sigma_{+,t}]^2\beta\Vert\bar{\mathbf w}^{(t_{k-1})}-\mathbf w^*\Vert^2
+25[\Sigma_{+,t}]^2
\left[\tilde{\sigma}^2+\tilde{\delta}^2+(\tilde{\epsilon}^{(0)})^2\right].
    \end{align}
    Using the strong convexity of $F(.)$, we have
    $\Vert\bar{\mathbf w}(t_{k-1})-\mathbf w^*\Vert^2
    \leq \frac{2}{\tilde{\mu}\beta}
    [F(\bar{\mathbf w}(t_{k-1}))-F(\mathbf w^*)]
    $, using which in~\eqref{eq:Amid} yields:
    \begin{align} 
    &\beta A^{(t)}
    \leq
    \frac{16\omega^2}{\tilde{\mu}}[\Sigma_{+,t}]^2 [F(\bar{\mathbf w}(t_{k-1}))-F(\mathbf w^*)]
    +25[\Sigma_{+,t}]^2
    \left[\tilde{\sigma}^2+(\tilde{\epsilon}^{(0)})^2+\tilde{\delta}^2\right]
    \nonumber \\&
    = 
    \frac{16\omega^2\beta}{\mu}[\Sigma_{+,t}]^2 [F(\bar{\mathbf w}(t_{k-1}))-F(\mathbf w^*)]
    +25[\Sigma_{+,t}]^2
    \left[\frac{\sigma^2}{\beta}+\frac{\delta^2}{\beta}+\beta(\epsilon^{(0)})^2\right].
    \end{align}
   This concludes the proofs.
    
\end{proof}

\section{Proof of Theorem \ref{co1}} \label{app:thm1}
% \begin{corollary} \label{co1}
% Let $\bar{\mathbf w}^{(t)}
%         =\sum\limits_{c=1}^{N}\varrho_c^{(k)}\bar{\mathbf w}_c^{(t)}$ denote the global average of the local models, $\forall t$.
%         Under Assumption \ref{beta}, if $\eta_t=\frac{\gamma}{t+\alpha}\leq\vartheta/\beta$, $\epsilon_c^{(t)}=\eta_t\phi_c$, where $\sum_d\varrho_d\phi_d\leq\phi$ for some $\phi>0$, $\forall t$ and $\alpha\geq \max\{2/\vartheta^2,3\}$ and
%     $2/\mu\leq\gamma\leq\vartheta\alpha/\beta$, then the one-step behavior of $\bar{\mathbf w}^{(t)}$ under  {\tt TT-HF} for $t\in \mathcal{T}_k$ can be described as follows:
% % https://www.overleaf.com/project/5f7c9b5ce460a000011dc1e1
% \begin{align*} 
%       &\mathbb E\left[F(\hat{\mathbf w}^{(t+1)})-F(\mathbf w^*)\right]
%         \leq 
%         (1-\mu\eta_{t})\mathbb E\left[F(\hat{\mathbf w}^{(t)})-F(\mathbf w^*)\right]
%         % +\beta\eta_{t}^2\phi^2
%         \nonumber \\&
%         +\eta_{t}^2\beta^2
%         \bigg[ 
%          12C^2(t_{k-1}+\alpha-1)[1+\eta_{0}(2\beta\omega-\mu/2)]^2/(\mu\gamma^2) [F(\hat{\mathbf w}^{(t_{k-1})})-F(\mathbf w^*)]\\&
%         +6C^2\eta_0[1+\eta_{0}(2\beta\omega-\mu/2)]^2\tau_k^2(\beta\vartheta+\sigma)^2
%         +3C^2\eta_{0}(\sqrt{2}\beta\phi\eta_{0}+\sigma+\sqrt{2}\delta)^2
%         +\phi^2(\frac{1}{2\beta}+\eta_{0}/2)+\frac{\sigma^2}{2\beta}
%         \bigg].
% \end{align*}
% \end{corollary}

\begin{theorem} \label{co1}
        Under Assumptions \ref{beta},~\ref{assump:cons}, and~\ref{assump:SGD_noise}, upon using {\tt TT-HF} for ML model training, if $\eta_t \leq 1/\beta$, $\forall t$, the one-step behavior of $\hat{\mathbf w}^{(t)}$ can be described as follows:
% https://www.overleaf.com/project/5f7c9b5ce460a000011dc1e1
\begin{align*} 
       \mathbb E\left[F(\hat{\mathbf w}^{(t+1)})-F(\mathbf w^*)\right]
        \leq&
        (1-\mu\eta_{t})\mathbb E[F(\hat{\mathbf w}^{(t)})-F(\mathbf w^*)]
        \nonumber \\&
        +\frac{\eta_{t}\beta^2}{2}A^{(t)}
        +\frac{1}{2}[\eta_{t}\beta^2(\epsilon^{(t)})^2+\eta_{t}^2\beta\sigma^2+\beta(\epsilon^{(t+1)})^2], ~t\in\mathcal{T}_k,
\end{align*}
where
\begin{align}
     A^{(t)}\triangleq\mathbb E\left[\sum\limits_{c=1}^N\varrho_{c}\Vert\bar{\mathbf w}_c^{(t)}-\bar{\mathbf w}^{(t)}\Vert_2^2\right].
\end{align}
\end{theorem}

\begin{proof}
% \add{(we omit the dependence on $k$ for conciseness)}.
% We define 
% the average of the local models for an arbitrary cluster $c$ as
% \nm{drop the dependence on k..}
% \begin{align}
%     \bar{\mathbf w}_c^{(t)}=\frac{1}{s_c^{(k)}}
%     \sum\limits_{j\in\mathcal{S}_c^{(k)}}\mathbf w_j^{(t)}.
% \end{align}
% and also define global average of the local models as
% \begin{align}
%     \bar{\mathbf w}^{(t)}&
%         =\sum_c\varrho_c^{(k)}\bar{\mathbf w}_c^{(t)}.
% \end{align}
Considering $t \in \mathcal T_k$, using \eqref{8}, \eqref{eq14}, the definition of $\bar{\mathbf{w}}$ given in Definition~\ref{modDisp}, and the fact that $\sum\limits_{i\in\mathcal{S}_c} \mathbf e_{i}^{{(t)}}=0$, $\forall t$, under Assumption~\ref{assump:cons},
the global average of the local models follows the following dynamics:
% \nm{define ${\mathbf n}_{j,t}=\widehat{\mathbf g}^{(t)}_{j}-\nabla F_j(\mathbf w_j^{(t)})$ (or something else) from the beginning so that you keep the math more concise}
\begin{align}\label{eq:GlobDyn1}
    \bar{\mathbf w}^{(t+1)}=
    \bar{\mathbf w}^{(t)}
    -\frac{\tilde{\eta}_{t}}{\beta}\sum\limits_{c=1}^N \varrho_c\frac{1}{s_c}\sum\limits_{j\in\mathcal{S}_c} 
    \nabla F_j(\mathbf w_j^{(t)})
              -\frac{\tilde{\eta}_{t}}{\beta}\sum\limits_{c=1}^N \varrho_c\frac{1}{s_c}\sum\limits_{j\in\mathcal{S}_c} 
             \mathbf n_j^{{(t)}},
\end{align} where ${\mathbf n}_{j}^{{(t)}}=\widehat{\mathbf g}^{(t)}_{j}-\nabla F_j(\mathbf w_j^{(t)})$.
On the other hand, the $\beta$-smoothness of the global function $F$ implies
\begin{align} 
    F(\bar{\mathbf w}^{(t+1)}) \leq  F(\bar{\mathbf w}^{(t)})&+\nabla F(\bar{\mathbf w}^{(t)})^\top(\bar{\mathbf w}^{(t+1)}-\bar{\mathbf w}^{(t)})+\frac{\beta}{2}\Big\Vert \bar{\mathbf w}^{(t+1)}-\bar{\mathbf w}^{(t)}\Big\Vert^2.
\end{align}
Replacing the result of~\eqref{eq:GlobDyn1} in the above inequality, taking the conditional expectation (conditioned on the knowledge of $\bar{\mathbf w}^{(t)}$) of the both hand sides, and using the fact that $\mathbb E_t[{\mathbf n}^{(t)}_{j}]=\bf 0$ yields:
\begin{align}
    &\mathbb E_{t}\left[F(\bar{\mathbf w}^{(t+1)})-F(\mathbf w^*)\right]
    \leq F(\bar{\mathbf w}^{(t)})-F(\mathbf w^*)
    -\frac{\tilde{\eta}_{t}}{\beta}\nabla F(\bar{\mathbf w}^{(t)})^\top \sum\limits_{c=1}^N\varrho_c\frac{1}{s_c}\sum\limits_{j\in\mathcal{S}_c}\nabla F_j(\mathbf w_j^{(t)})
    \nonumber \\&
    +\frac{\tilde{\eta}_{t}^2}{2\beta}\Big
    \Vert\sum\limits_{c=1}^N\varrho_c\frac{1}{s_c}\sum\limits_{j\in\mathcal{S}_c} \nabla F_j(\mathbf w_j^{(t)})
    \Big\Vert^2
    + \frac{\tilde{\eta}_{t}^2}{2\beta}\mathbb E_t\left[
    \Big
    \Vert\sum\limits_{c=1}^N\varrho_c\frac{1}{s_c}\sum\limits_{j\in\mathcal{S}_c}\mathbf n_j^{{(t)}}\Big\Vert^2
    \right].
\end{align}
% \add{where we used the fact that $\mathbb E_t[\widehat{\mathbf g}^{(t)}_{j}]=\nabla F_j(\mathbf w_j^{(t)})$}
% \nm{in the last step you used the fact that SGD is unbiased.. please explain}
Since $\mathbb E_t[\Vert\mathbf n_i^{{(t)}}\Vert_2^2]\leq \beta\tilde{\sigma}^2$, $\forall i$, we get
\begin{align}\label{eq:aveGlob1}
    \mathbb E_{t}\left[F(\bar{\mathbf w}^{(t+1)})-F(\mathbf w^*)\right]
        \leq& F(\bar{\mathbf w}^{(t)})-F(\mathbf w^*)
        \nonumber \\&
        -\frac{\tilde{\eta}_{t}}{\beta}\nabla F(\bar{\mathbf w}^{(t)})^\top \sum\limits_{c=1}^N\varrho_c\frac{1}{s_c}\sum\limits_{j\in\mathcal{S}_c}\nabla F_j(\mathbf w_j^{(t)})
        \nonumber \\&
        +\frac{\tilde{\eta}_{t}^2}{2\beta}\Big
        \Vert\sum\limits_{c=1}^N\varrho_c\frac{1}{s_c}\sum\limits_{j\in\mathcal{S}_c} \nabla F_j(\mathbf w_j^{(t)})
        \Big\Vert^2
        +\frac{\tilde{\eta}_{t}^2\tilde{\sigma}^2}{2}.
\end{align}
Using Lemma \ref{lem1} (see Appendix \ref{app:PL-bound}), we further bound~\eqref{eq:aveGlob1} as follows:
\begin{align} \label{eq:E_t}
    &\mathbb E_{t}\left[F(\bar{\mathbf w}^{(t+1)})-F(\mathbf w^*)\right]
    \leq
        (1-\tilde{\mu}\tilde{\eta}_{t})(F(\bar{\mathbf w}^{(t)})-F(\mathbf w^*))
        \nonumber \\& 
       -\frac{\tilde{\eta}_{t}}{2\beta}(1-\tilde{\eta}_{t})\Big\Vert\sum\limits_{c=1}^N\varrho_c\frac{1}{s_c}\sum\limits_{j\in\mathcal{S}_c} \nabla F_j(\mathbf w_j^{(t)})\Big\Vert^2
       +\frac{\tilde{\eta}_{t}^2\tilde{\sigma}^2}{2}
        +\frac{\tilde{\eta}_{t}\beta}{2}\sum\limits_{c=1}^N\varrho_c \frac{1}{s_c}\sum\limits_{j\in\mathcal{S}_c}\Big\Vert\bar{\mathbf w}^{(t)}-\mathbf w_j^{(t)}\Big\Vert^2
    \nonumber \\&
    \leq
        (1-\tilde{\mu}\tilde{\eta}_{t})(F(\bar{\mathbf w}^{(t)})-F(\mathbf w^*))
        +\frac{\tilde{\eta}_{t}\beta}{2}\sum\limits_{c=1}^N\varrho_c \frac{1}{s_c}\sum\limits_{j\in\mathcal{S}_c}\Big\Vert\bar{\mathbf w}^{(t)}-\mathbf w_j^{(t)}\Big\Vert^2+\frac{\tilde{\eta}_{t}^2\tilde{\sigma}^2}{2},
\end{align}
where the last step follows from $\tilde{\eta}_t\leq 1$. 
To further bound the terms on the right hand side of~\eqref{eq:E_t}, we use the fact that
\begin{align}
        \Vert\mathbf w_i^{(t)}-\bar{\mathbf w}^{(t)}\Vert^2
        =
        \Vert\bar{\mathbf w}_c^{(t)}-\bar{\mathbf w}^{(t)}\Vert^2
        +\Vert \mathbf e_i^{{(t)}}\Vert^2
        +2[\bar{\mathbf w}_c^{(t)}-\bar{\mathbf w}^{(t)}]^\top \mathbf e_i^{{(t)}},
    \end{align} 
    % where $\mathbf e_i^{{(t_{k-1})}}=\mathbf 0$.
    which results in
    \begin{align} \label{eq:wi_w}
        \frac{1}{s_c}\sum\limits_{i\in \mathcal S_c}\Vert\mathbf w_i^{(t)}-\bar{\mathbf w}^{(t)}\Vert_2^2
        \leq
        \Vert\bar{\mathbf w}_c^{(t)}-\bar{\mathbf w}^{(t)}\Vert_2^2
        +\frac{(\tilde{\epsilon}_{c}^{(t)})^2}{\beta}.
    \end{align} 
    Replacing \eqref{eq:wi_w} in \eqref{eq:E_t}  and taking the unconditional expectation from the both hand sides of the resulting expression gives us
    \begin{align} \label{eq:ld}
        &\mathbb E\left[F(\bar{\mathbf w}^{(t+1)})-F(\mathbf w^*)\right]
        \leq
        (1-\tilde{\mu}\tilde{\eta}_{t})\mathbb E[F(\bar{\mathbf w}^{(t)})-F(\mathbf w^*)]
        \nonumber \\&
        +\frac{\tilde{\eta}_{t}\beta}{2}\sum\limits_{c=1}^N\varrho_c \left(\Vert\bar{\mathbf w}_c^{(t)}-\bar{\mathbf w}^{(t)}\Vert_2^2
        +\frac{(\tilde{\epsilon}_{c}^{(t)})^2}{\beta}\right)+\frac{\tilde{\eta}_{t}^2\tilde{\sigma}^2}{2}
        \nonumber \\&
        =
        (1-\tilde{\mu}\tilde{\eta}_{t})\mathbb E[F(\bar{\mathbf w}^{(t)})-F(\mathbf w^*)]
        +\frac{\tilde{\eta}_{t}\beta}{2}A^{(t)}
        +\frac{1}{2}[\tilde{\eta}_{t}(\tilde{\epsilon}^{(t)})^2+\tilde{\eta}_{t}^2\tilde{\sigma}^2],
    \end{align}
where  
\begin{align}
    A^{(t)}\triangleq\mathbb E\left[\sum\limits_{c=1}^N\varrho_c \Vert\bar{\mathbf w}_c^{(t)}-\bar{\mathbf w}^{(t)}\Vert^2\right].
\end{align}

By $\beta$-smoothness of $F(\cdot)$, we have
\begin{align} \label{eq:39}
    &F(\hat{\mathbf w}^{(t)})-F(\mathbf w^*) 
    \leq F(\bar{\mathbf w}^{(t)})-F(\mathbf w^*)
    +\nabla F(\bar{\mathbf w}^{(t)})^\top\Big(\hat{\mathbf w}^{(t)}-\bar{\mathbf w}^{(t)}\Big)+\frac{\beta}{2}\Vert\hat{\mathbf w}^{(t)}-\bar{\mathbf w}^{(t)}\Vert^2
    \nonumber \\&
    \leq F(\bar{\mathbf w}^{(t)})-F(\mathbf w^*)+\nabla F(\bar{\mathbf w}^{(t)})^\top\sum\limits_{c=1}^N\varrho_c \mathbf e_{s_c}^{{(t)}}
    +\frac{\beta}{2}\sum\limits_{c=1}^N\varrho_c\Vert \mathbf e_{s_c}^{{(t)}}\Vert^2.
\end{align}
Taking the expectation with respect to the device sampling from both hand sides of \eqref{eq:39}, since the sampling is conducted uniformly at random, we obtain
\begin{align} 
    \mathbb E_t\left[F(\hat{\mathbf w}^{(t)})-F(\mathbf w^*)\right]  
    \leq& F(\bar{\mathbf w}^{(t)})-F(\mathbf w^*)+\nabla F(\bar{\mathbf w}^{(t)})^\top\sum\limits_{c=1}^N\varrho_c \underbrace{\mathbb    E_t\left[\mathbf e_{s_c}^{{(t)}}\right]}_{=0}
    \nonumber \\&+\frac{\beta}{2}\sum\limits_{c=1}^N\varrho_c\mathbb E_t\left[\Vert \mathbf e_{s_c}^{{(t)}}\Vert^2\right].
\end{align}
Taking the total expectation from both hand sides of the above inequality yields:
\begin{align} \label{eq:hat-bar}
    &\mathbb E\left[F(\hat{\mathbf w}^{(t)})-F(\mathbf w^*)\right]  
    \leq \mathbb E\left[F(\bar{\mathbf w}^{(t)})-F(\mathbf w^*)\right]
    +\frac{(\tilde{\epsilon}^{(t)})^2}{2}.
\end{align}
Replace \eqref{eq:ld} into \eqref{eq:hat-bar}, we have
\begin{align} \label{eq:-}
    \mathbb E\left[F(\hat{\mathbf w}^{(t+1)})-F(\mathbf w^*)\right]
        \leq&
        (1-\tilde{\mu}\tilde{\eta}_{t})\mathbb E[F(\bar{\mathbf w}^{(t)})-F(\mathbf w^*)]+\frac{\tilde{\eta}_{t}\beta}{2}A^{(t)}
        \nonumber \\&
        +\frac{1}{2}[\tilde{\eta}_{t}(\tilde{\epsilon}^{(t)})^2+\tilde{\eta}_{t}^2\tilde{\sigma}^2+(\tilde{\epsilon}^{(t+1)})^2].
\end{align}
On the other hands, using the strong convexity of $F(\cdot)$, we have
\begin{align} \label{eq:^}
    &F(\hat{\mathbf w}^{(t)})-F(\mathbf w^*) 
    \geq F(\bar{\mathbf w}^{(t)})-F(\mathbf w^*)
    +\nabla F(\bar{\mathbf w}^{(t)})^\top\Big(\hat{\mathbf w}^{(t)}-\bar{\mathbf w}^{(t)}\Big)+\frac{\mu}{2}\Vert\hat{\mathbf w}^{(t)}-\bar{\mathbf w}^{(t)}\Vert^2
    \nonumber \\&
    \geq F(\bar{\mathbf w}^{(t)})-F(\mathbf w^*)+\nabla F(\bar{\mathbf w}^{(t)})^\top\sum\limits_{c=1}^N\varrho_c \mathbf e_{s_c}^{{(t)}}.
    % \\&
    % \add{
    % \leq F(\hat{\mathbf w}^{(t)})-F(\mathbf w^*)
    % -\nabla F(\bar{\mathbf w}^{(t)})^\top\sum\limits_{c=1}^N\varrho_c \mathbf e_{s_c}^{{(t)}}
    % +\Vert\nabla F(\hat{\mathbf w}^{(t)})-\nabla F(\bar{\mathbf w}^{(t)})\Vert\Vert\sum\limits_{c=1}^N\varrho_c \mathbf e_{s_c}^{{(t)}}\Vert
    % +\frac{\beta}{2}\sum\limits_{c=1}^N\varrho_c\Vert \mathbf e_{s_c}^{{(t)}}\Vert^2
    % }
    %     \\&
    % \add{
    % \leq F(\hat{\mathbf w}^{(t)})-F(\mathbf w^*)
    % -\nabla F(\bar{\mathbf w}^{(t)})^\top\sum\limits_{c=1}^N\varrho_c \mathbf e_{s_c}^{{(t)}}
    % +\beta\Vert\sum\limits_{c=1}^N\varrho_c \mathbf e_{s_c}^{{(t)}}\Vert^2
    % +\frac{\beta}{2}\sum\limits_{c=1}^N\varrho_c\Vert \mathbf e_{s_c}^{{(t)}}\Vert^2
    % }
    %         \\&
    % \add{
    % \leq F(\hat{\mathbf w}^{(t)})-F(\mathbf w^*)
    % -\nabla F(\bar{\mathbf w}^{(t)})^\top\sum\limits_{c=1}^N\varrho_c \mathbf e_{s_c}^{{(t)}}
    % +\frac{3\beta}{2}\sum\limits_{c=1}^N\varrho_c\Vert \mathbf e_{s_c}^{{(t)}}\Vert^2
    % }
\end{align}
Taking the expectation with respect to the device sampling from the both hand sides of \eqref{eq:^}, since the sampling is conducted uniformly at random, we obtain
\begin{align} 
    &\mathbb E_t\left[F(\hat{\mathbf w}^{(t)})-F(\mathbf w^*)\right]  
    % \nonumber \\&
    \geq 
    F(\bar{\mathbf w}^{(t)})-F(\mathbf w^*)
    +\nabla F(\bar{\mathbf w}^{(t)})^\top\sum\limits_{c=1}^N\varrho_c \underbrace{\mathbb    E_t\left[\mathbf e_{s_c}^{{(t)}}\right]}_{=0}.
\end{align}
Taking the total expectation from both hand sides of the above inequality yields:
\begin{align} \label{eq:bar-hat}
    &\mathbb E\left[F(\hat{\mathbf w}^{(t)})-F(\mathbf w^*)\right]  
    \geq \mathbb E\left[F(\bar{\mathbf w}^{(t)})-F(\mathbf w^*)\right].
\end{align}
Finally, replacing \eqref{eq:bar-hat} into \eqref{eq:-}, we obtain
\begin{align}
    &\mathbb E\left[F(\hat{\mathbf w}^{(t+1)})-F(\mathbf w^*)\right]
        \leq
        (1-\tilde{\mu}\tilde{\eta}_{t})\mathbb E[F(\hat{\mathbf w}^{(t)})-F(\mathbf w^*)]
        \nonumber \\&
        +\frac{\tilde{\eta}_{t}\beta}{2}A^{(t)}
        +\frac{1}{2}[\tilde{\eta}_{t}(\tilde{\epsilon}^{(t)})^2+\tilde{\eta}_{t}^2\tilde{\sigma}^2+(\tilde{\epsilon}^{(t+1)})^2]
        \nonumber
        \nonumber \\&
        =
        (1-\mu\eta_{t})\mathbb E[F(\hat{\mathbf w}^{(t)})-F(\mathbf w^*)]
        +\frac{\eta_{t}\beta^2}{2}A^{(t)}
        +\frac{1}{2}[\eta_{t}\beta^2(\epsilon^{(t)})^2+\eta_{t}^2\beta\sigma^2+\beta(\epsilon^{(t+1)})^2].
        % \nonumber \\&
        % \overset{(a)}{\leq}
        % (1-\tilde{\mu}\tilde{\eta}_{t})\mathbb E[F(\hat{\mathbf w}^{(t)})-F(\mathbf w^*)]
        % +\frac{\tilde{\eta}_{t}\beta}{2}A^{(t)}
        % +\frac{1}{2}[\tilde{\eta}_{t}^3\tilde{\phi}^2+\tilde{\eta}_{t}^2\tilde{\sigma}^2+\tilde{\eta}_{t+1}^2\tilde{\phi}^2]
        % \nonumber \\&
        % =
        % (1-\mu\eta_{t})\mathbb E[F(\hat{\mathbf w}^{(t)})-F(\mathbf w^*)]
        % +\frac{\eta_{t}\beta^2}{2}A^{(t)}
        % % \nonumber \\&
        % +\frac{1}{2}[\eta_{t}^3\phi^2\beta^2+\eta_{t}^2\beta\sigma^2+\eta_{t+1}^2\phi^2\beta],
\end{align}
% where $(a)$ comes from the fact that $\tilde{\epsilon}^{(t)}=\tilde{\eta}_t\tilde{\phi}$.
% $\tilde{\epsilon}^{(t)}=\tilde{\eta}_t\phi=\eta_t\phi\beta$. 
This concludes the proof.

\end{proof}

\section{Proof of Theorem \ref{thm:subLin}} \label{app:subLin}
% \addFL{
% \frank{Statement of the Theorem 2 needs modified:}
\begin{theorem} \label{thm:subLin}
Define $Z_1\triangleq \frac{32\beta^2\gamma}{\mu}(\tau-1)\left(1+\frac{\tau}{\alpha-1}\right)^{2}\left(1+\frac{\tau-1}{\alpha-1}\right)^{6\beta\gamma}$, $ Z_2\triangleq
    \frac{1}{2}[\frac{\sigma^2}{\beta}+\frac{2\phi^2}{\beta}]
    +50\beta\gamma(\tau-1)\left(1+\frac{\tau-2}{\alpha+1}\right)
    \left(1+\frac{\tau-1}{\alpha-1}\right)^{6\beta\gamma}\left[\frac{\sigma^2}{\beta}+\frac{\phi^2}{\beta}+\frac{\delta^2}{\beta}\right]$. Also, assume $\gamma>1/\mu$, $\alpha\geq\max\{\beta\gamma[\frac{\vartheta}{4}-1+\sqrt{(1+\frac{\vartheta}{4})^2+2\omega}],\frac{\beta^2\gamma}{\mu}\}$ and $\omega< \frac{1}{\beta\gamma}\sqrt{\alpha\frac{\mu\gamma-1+\frac{1}{1+\alpha}}{Z_1}}\triangleq \omega_{\max}$. Upon using {\tt TT-HF} for ML model training under Assumptions \ref{beta},~\ref{assump:cons}, and~\ref{assump:SGD_noise}, if $\eta_t=\frac{\gamma}{t+\alpha}$ and $\epsilon^{(t)}=\eta_t\phi$, $\forall t$, we have: 
    % Using {\tt TT-HF} for ML model training, under Assumption \ref{beta}, if we set the step size as $\eta_t=\frac{\gamma}{t+\alpha}$, $\forall t$, and assuming that $\epsilon(t)=\eta_t\phi$, $\forall t$,
    % we have
    \begin{align} \label{eq:thm2_result-1-A}
        &\mathbb E\left[(F(\hat{\mathbf w}^{(t)})-F(\mathbf w^*))\right]\leq\frac{\nu}{t+\alpha}, ~~\forall t,
    \end{align}
    where $\nu \triangleq Z_2 \max \left\{\frac{\beta^2\gamma^2}{\mu\gamma-1},
\frac{\alpha}{Z_1\left(\omega_{\max}^2-\omega^2\right)},\frac{\alpha}{Z_2}\left[F(\hat{\mathbf w}^{(0)})-F(\mathbf w^*)\right]\right\}$.
\end{theorem}
\begin{proof}

We carry out the proof by induction. We start by considering the first global aggregation, i.e., $k=1$. Note that the condition in~\eqref{eq:thm2_result-1-A} trivially holds at the beginning of this global aggregation $t=t_0=0$ since $
\nu\geq\alpha\left[F(\hat{\mathbf w}^{(0)})-F(\mathbf w^*)\right]$.
    Now, assume that
    \begin{align} \label{eq:thm2_result-1}
        &\mathbb E\left[F(\hat{\mathbf w}^{(t_{k-1})})-F(\mathbf w^*)\right]  
        \leq \frac{\nu}{t_{k-1}+\alpha}
    \end{align}
for some $k\geq 1$. We prove that this implies 
\begin{align} \label{eq:thm2_result-1}
        &\mathbb E\left[F(\hat{\mathbf w}^{(t)})-F(\mathbf w^*)\right]  
        \leq \frac{\nu}{t+\alpha},\ \forall t\in \mathcal{T}_k,
    \end{align}
    and as a result $\mathbb E\left[F(\hat{\mathbf w}^{(t_{k})})-F(\mathbf w^*)\right]  
        \leq \frac{\nu}{t_{k}+\alpha}$.
    % so that the theorem follows by induction over $k$.
    To prove~\eqref{eq:thm2_result-1}, we use induction over $t\in \{t_{k-1}+1,\dots,t_k\}$.
    Clearly, the condition holds for $t=t_{k-1}$ from the induction hypothesis.
    Now, we assume that it also holds for some $t\in \{t_{k-1},\dots,t_k-1\}$, and aim to show that it holds at $t+1$.
    
    From the result of Theorem~\ref{co1}, considering $\tilde{\epsilon}^{(t)}=\tilde{\eta}_t\tilde{\phi}$, we get
    \begin{align} \label{eq:thm1_temp}
        &\mathbb E\left[F(\hat{\mathbf w}^{(t+1)})-F(\mathbf w^*)\right]
        \leq
        (1-\tilde{\mu}\tilde{\eta}_{t})\frac{\nu}{t+\alpha}
        +\frac{\tilde{\eta}_{t}\beta}{2}A^{(t)}
        +\frac{1}{2}[\tilde{\eta}_{t}^3\tilde{\phi}^2+\tilde{\eta}_{t}^2\tilde{\sigma}^2+\tilde{\eta}_{t+1}^2\tilde{\phi}^2].
    \end{align}
Using the induction hypothesis and the bound on $A^{(t)}$, we can further upper bound~\eqref{eq:thm1_temp} as
\begin{align} \label{eq:induction_main}
    &\mathbb E\left[F(\hat{\mathbf w}^{(t+1)})-F(\mathbf w^*)\right]
    \leq
    \left(1-\tilde{\mu}\tilde{\eta}_{t}\right)\frac{\nu}{t+\alpha}
    +\frac{8\tilde{\eta}_t\omega^2}{\tilde{\mu}}[\Sigma_{+,t}]^2\frac{\nu}{t_{k-1}+\alpha}
    \nonumber \\&+\frac{25}{2}\tilde{\eta}_t[\Sigma_{+,t}]^2
    \left[\tilde{\sigma}^2+(\tilde{\epsilon}^{(0)})^2+\tilde{\delta}^2\right]   
    +\frac{1}{2}[\tilde{\eta}_{t}^3\tilde{\phi}^2+\tilde{\eta}_{t}^2\tilde{\sigma}^2+\tilde{\eta}_{t+1}^2\tilde{\phi}^2].
\end{align}
Since $\tilde{\eta}_{t+1}\leq \tilde{\eta}_t$, $\tilde{\eta}_t\leq \tilde{\eta}_0\leq\tilde{\mu}\leq 1$ and $\tilde{\epsilon}^{(0)}=\tilde{\eta}_0\tilde{\phi}\leq\tilde{\phi}$, we further upper bound~\eqref{eq:induction_main} as
\begin{align} \label{eq:thm1_Sig}
    &\mathbb E\left[F(\hat{\mathbf w}^{(t+1)})-F(\mathbf w^*)\right]
    \leq
    \left(1-\tilde{\mu}\tilde{\eta}_{t}\right)\frac{\nu}{t+\alpha}
    +\frac{8\tilde{\eta}_t\omega^2}{\tilde{\mu}}\underbrace{[\Sigma_{+,t}]^2}_{(a)}\frac{\nu}{t_{k-1}+\alpha}
    \nonumber \\&
    +\frac{25}{2}\tilde{\eta}_t\underbrace{[\Sigma_{+,t}]^2}_{(b)}
\left[\tilde{\sigma}^2+\tilde{\phi}^2+\tilde{\delta}^2\right]    
    + \frac{\tilde{\eta}_t^2}{2}[\tilde{\sigma}^2+2\tilde{\phi}^2].
\end{align}
To get a tight upper bound for~\eqref{eq:thm1_Sig}, we bound the two instances of $[\Sigma_{+,t}]^2$ appearing in $(a)$ and $(b)$ differently. In particular, for $(a)$, we first use the fact that
\begin{align*}
    \lambda_+ =1-\tilde{\mu}/4+\sqrt{(1+\tilde{\mu}/4)^2+2\omega}\in[2,1+\sqrt{3}],
\end{align*}
which implies that
\begin{align}
    \Sigma_{+,t}
  =&\sum_{\ell=t_{k-1}}^{t-1}
  \left[\prod_{j=t_{k-1}}^{\ell-1}(1+\tilde{\eta}_j\lambda_{+})\right]
  \eta_\ell
  \left[\prod_{j=\ell+1}^{t-1}(1+\tilde{\eta}_j)\right]
  \nonumber \\&\leq
  \sum_{\ell=t_{k-1}}^{t-1}
  \left[\prod_{j=t_{k-1}}^{\ell-1}(1+\tilde{\eta}_j\lambda_{+})\right]
  \eta_\ell
  \left[\prod_{j=\ell+1}^{t-1}(1+\tilde{\eta}_j\lambda_{+})\right]
  \nonumber \\&
  \leq
  \left[\prod_{j=t_{k-1}}^{t-1}(1+\tilde{\eta}_j\lambda_{+})\right]
\sum_{\ell=t_{k-1}}^{t-1}\frac{\tilde{\eta}_\ell}{1+\tilde{\eta}_\ell\lambda_{+}}.
\end{align}
Also, with the choice of step size $\tilde{\eta}_\ell=\frac{\tilde{\gamma}}{\ell+\alpha}$, we get
\begin{align} \label{eq:Sigma_1}
  \Sigma_{+,t}
  \leq
  \tilde{\gamma}\underbrace{\left[\prod_{j=t_{k-1}}^{t-1}\left(1+\frac{\tilde{\gamma}\lambda_+}{j+\alpha}\right)\right]}_{(i)}
  \underbrace{\sum_{\ell=t_{k-1}}^{t-1}\frac{1}{\ell+\alpha+\tilde{\gamma}\lambda_+}}_{(ii)}.
\end{align}
To bound $(ii)$, since $\frac{1}{\ell+\alpha+\tilde{\gamma}\lambda_+}$ is a decreasing function with respect to $\ell$, we have
\begin{align}
    \sum_{\ell=t_{k-1}}^{t-1}\frac{1}{\ell+\alpha+\tilde{\gamma}\lambda_+}
    \leq
    \int_{t_{k-1}-1}^{t-1}\frac{1}{\ell+\alpha+\tilde{\gamma}\lambda_+}\mathrm d\ell
    =
    \ln\left(1+\frac{t-t_{k-1}}{t_{k-1}-1+\alpha+\tilde{\gamma}\lambda_+}\right),
\end{align}
where we used the fact that $\alpha>1-\tilde{\gamma}\lambda_+$ (implied by $\alpha>1$).

To bound $(i)$, we first rewrite it as follows: 
    \begin{align} \label{eq:(i)}
        \prod_{j=t_{k-1}}^{t-1}\left(1+\frac{\tilde{\gamma}\lambda_+}{j+\alpha}\right)
        =
        e^{\sum_{j=t_{k-1}}^{t-1}\ln\big(1+\frac{\tilde{\gamma}\lambda_+}{j+\alpha}\big)}
    \end{align}
To bound~\eqref{eq:(i)}, we use the fact that $\ln(1+\frac{\tilde{\gamma}\lambda_+}{j+\alpha})$ is a decreasing function with respect to $j$, and $\alpha >1$, to get
\begin{align} \label{eq:ln(i)}
    &\sum_{j=t_{k-1}}^{t-1}\ln(1+\frac{\tilde{\gamma}\lambda_+}{j+\alpha})
    \leq
    \int_{t_{k-1}-1}^{t-1}\ln(1+\frac{\tilde{\gamma}\lambda_+}{j+\alpha})\mathrm dj
    \nonumber \\&
    \leq
    \tilde{\gamma}\lambda_+\int_{t_{k-1}-1}^{t-1}\frac{1}{j+\alpha}\mathrm dj
    =
    \tilde{\gamma} \lambda_+\ln\left(1+\frac{t-t_{k-1}}{t_{k-1}-1+\alpha}\right).
\end{align}
Considering ~\eqref{eq:(i)} and~\eqref{eq:ln(i)} together, we  bound $(i)$ as follows:
    \begin{align}
        \prod_{j=t_{k-1}}^{t-1}\left(1+\frac{\tilde{\gamma}\lambda_+}{j+\alpha}\right)
        \leq
        \left(1+\frac{t-t_{k-1}}{t_{k-1}-1+\alpha}\right)^{\tilde{\gamma} \lambda_+}.
    \end{align}

    Using the results obtained for bounding $(i)$ and $(ii)$ back in~\eqref{eq:Sigma_1}, we get:
    \begin{align} \label{eq:Sigma1_bound}
        \Sigma_{+,t}
        \leq
        \tilde{\gamma}\ln\left(1+\frac{t-t_{k-1}}{t_{k-1}-1+\alpha+\tilde{\gamma}\lambda_+}\right)
        \left(1+\frac{t-t_{k-1}}{t_{k-1}-1+\alpha}\right)^{\tilde{\gamma} \lambda_+}.
    \end{align}
    Since $\ln(1+x)\leq\ln(1+x+2\sqrt{x})=\ln((1+\sqrt{x})^2)=2\ln(1+\sqrt{x})\leq 2\sqrt{x}$ for $x\geq 0$,
    we can further bound~\eqref{eq:Sigma1_bound} as follows:
    \begin{align} \label{eq:Sig}
        &\Sigma_{+,t}
        \leq
        2\tilde{\gamma}\sqrt{\frac{t-t_{k-1}}{t_{k-1}-1+\alpha+\tilde{\gamma}\lambda_+}}
        \left(1+\frac{t-t_{k-1}}{t_{k-1}+\alpha-1}\right)^{\tilde{\gamma}\lambda_+}
        \nonumber \\&
        \leq
        2\tilde{\gamma}\sqrt{\frac{t-t_{k-1}}{t_{k-1}+\alpha+1}}
        \left(1+\frac{t-t_{k-1}}{t_{k-1}+\alpha-1}\right)^{3\tilde{\gamma}},
    \end{align}
    where in the last inequality we used
     $2\leq\lambda_+ < 3$ and $\tilde{\gamma}\geq \frac{\tilde{\mu}}{\beta}\tilde{\gamma} >1$.
Taking the square from the both hand sides of~\eqref{eq:Sig} followed by multiplying the both hand sides with $\frac{[t+\alpha]^2}{\tilde{\mu}\tilde{\gamma}[t_{k-1}+\alpha]}$ gives us:
    \begin{align} \label{eq:sig_sqaured}
        &[{\Sigma}_{+,t}]^2\frac{[t+\alpha]^2}{\tilde{\mu}\tilde{\gamma}[t_{k-1}+\alpha]}
        \leq
        \frac{4\tilde{\gamma}}{\tilde{\mu}}\frac{[t-t_{k-1}][t+\alpha]^2}{[t_{k-1}+\alpha+1][t_{k-1}+\alpha]}
        \left(1+\frac{t-t_{k-1}}{t_{k-1}+\alpha-1}\right)^{6\tilde{\gamma}}
        \nonumber \\&
        \leq
        \frac{4\tilde{\gamma}}{\tilde{\mu}}\frac{[t-t_{k-1}][t+\alpha]^2}{[t_{k-1}+\alpha-1]^2}
        \frac{[t_{k-1}+\alpha-1]^2}{[t_{k-1}+\alpha+1][t_{k-1}+\alpha]}
        \left(1+\frac{\tau-1}{t_{k-1}+\alpha-1}\right)^{-2}
        \left(1+\frac{\tau-1}{t_{k-1}+\alpha-1}\right)^{6\tilde{\gamma}+2}
        \nonumber \\&
        \overset{(a)}{\leq}
        \frac{4\tilde{\gamma}}{\tilde{\mu}}\frac{[\tau-1][t_{k-1}+\tau-1+\alpha]^2}{[t_{k-1}+\alpha-1]^2}
        \left(\frac{t_{k-1}+\alpha+\tau-2}{t_{k-1}+\alpha-1}\right)^{-2}
        \frac{[t_{k-1}+\alpha-1]^2}{[t_{k-1}+\alpha+1][t_{k-1}+\alpha]}
        \left(1+\frac{\tau-1}{t_{k-1}+\alpha-1}\right)^{6\tilde{\gamma}+2}
        \nonumber \\&
        \leq
        \frac{4\tilde{\gamma}}{\tilde{\mu}}(\tau-1)\left(1+\frac{1}{\tau+t_{k-1}+\alpha-2}\right)^2
        \frac{[t_{k-1}+\alpha-1]^2}{[t_{k-1}+\alpha+1][t_{k-1}+\alpha]}
        \left(1+\frac{\tau-1}{t_{k-1}+\alpha-1}\right)^{6\tilde{\gamma}+2}\hspace{-6mm},\hspace{-1mm}
    \end{align}
    where $(a)$ comes from the fact that $t\leq t_{k-1}+\tau_k-1\leq t_{k-1}+\tau-1$. To bound~\eqref{eq:sig_sqaured}, we use the facts that
    \begin{align}
        1+\frac{1}{\tau+t_{k-1}+\alpha-2}
        \leq
        1+\frac{1}{\tau+\alpha-2},\ 1+\frac{\tau-1}{t_{k-1}+\alpha-1}\leq 1+\frac{\tau-1}{\alpha-1},
    \end{align}
    and
    \begin{align}
        \frac{[t_{k-1}+\alpha-1]^2}{[t_{k-1}+\alpha+1][t_{k-1}+\alpha]}\leq 1,
    \end{align}
    which yield
    \begin{align}
        [{\Sigma}_{+,t}]^2\frac{[t+\alpha]^2}{\tilde{\mu}\tilde{\gamma}[t_{k-1}+\alpha]}
        \leq
        \frac{4\tilde{\gamma}}{\tilde{\mu}}(\tau-1)\left(1+\frac{\tau}{\alpha-1}\right)^{2}
        \left(1+\frac{\tau-1}{\alpha-1}\right)^{6\tilde{\gamma}}.
    \end{align}
    Consequently, we have
    \begin{align} \label{eq:Sigma_1st}
        [{\Sigma}_{+,t}]^2
        \leq
        4(\tau-1)\left(1+\frac{\tau}{\alpha-1}\right)^{2}
        \left(1+\frac{\tau-1}{\alpha-1}\right)^{6\tilde{\gamma}}
        \tilde{\eta}_t^2[t_{k-1}+\alpha].
    \end{align}
    
On the other hand, we bound the second instance of $[{\Sigma}_{+,t}]^2$, i.e., $(b)$ in~\eqref{eq:thm1_Sig}, as follows:
\begin{align}
    [t+\alpha] [\Sigma_{+,t}]^2
    &\leq
    4\tilde{\gamma}^2\frac{[t-t_{k-1}][t+\alpha]}{t_{k-1}+\alpha+1}
    \left(1+\frac{t-t_{k-1}}{t_{k-1}+\alpha-1}\right)^{6\tilde{\gamma}}
    \nonumber \\&
    \leq
    4\tilde{\gamma}^2(\tau-1)\left(1+\frac{\tau-2}{t_{k-1}+\alpha+1}\right)
    \left(1+\frac{\tau-1}{t_{k-1}+\alpha-1}\right)^{6\tilde{\gamma}}
    \nonumber \\&
    \leq
    4\tilde{\gamma}^2(\tau-1)\left(1+\frac{\tau-2}{\alpha+1}\right)
    \left(1+\frac{\tau-1}{\alpha-1}\right)^{6\tilde{\gamma}},
\end{align}
which implies
\begin{align} \label{eq:Sigma_2nd}
    [\Sigma_{+,t}]^2
    \leq
    4\tilde{\gamma}(\tau-1)\left(1+\frac{\tau-2}{\alpha+1}\right)
    \left(1+\frac{\tau-1}{\alpha-1}\right)^{6\tilde{\gamma}}\tilde{\eta}_t.
\end{align}
Replacing~\eqref{eq:Sigma_1st} and~\eqref{eq:Sigma_2nd} into \eqref{eq:thm1_Sig}, we get
\begin{align} \label{eq:induce_form}
    \mathbb E[F(\hat{\mathbf w}^{(t+1)})-F(\mathbf w^*)]
    \leq
    \left(1-\tilde{\mu}\tilde{\eta}_t+Z_1\omega^2\tilde{\eta}_t^2 \right)\frac{\nu}{t+\alpha}
    +\tilde{\eta}_t^2Z_2,
\end{align}
where we have defined
\begin{align}
    Z_1\triangleq \frac{32\tilde{\gamma}}{\tilde{\mu}}(\tau-1)\left(1+\frac{\tau}{\alpha-1}\right)^{2}
    \left(1+\frac{\tau-1}{\alpha-1}\right)^{6\tilde{\gamma}},
\end{align}
and
\begin{align}
    Z_2\triangleq
    \frac{1}{2}[\tilde{\sigma}^2+2\tilde{\phi}^2]
    +50\tilde{\gamma}(\tau-1)\left(1+\frac{\tau-2}{\alpha+1}\right)
    \left(1+\frac{\tau-1}{\alpha-1}\right)^{6\tilde{\gamma}}\left[\tilde{\sigma}^2+\tilde{\phi}^2+\tilde{\delta}^2\right].  
\end{align}

Now, from~\eqref{eq:induce_form}, to complete the induction, we aim to show that 
\begin{align}\label{eq:lastBeforeInd}
    \mathbb E[F(\hat{\mathbf w}^{(t+1)})-F(\mathbf w^*)]
    \leq
    \left(1-\tilde{\mu}\tilde{\eta}_t+Z_1\omega^2\tilde{\eta}_t^2 \right)\frac{\nu}{t+\alpha}
    +\tilde{\eta}_t^2Z_2
    \leq
    \frac{\nu}{t+1+\alpha}.
\end{align}
We transform the condition in~\eqref{eq:lastBeforeInd} through the set of following algebraic steps to an inequality condition on a convex function:
\begin{align}\label{eq:longtrasform}
    &
     \left(-\frac{\tilde{\mu}}{\tilde{\eta}_t^2}+\frac{Z_1\omega^2}{\tilde{\eta}_t} \right)\frac{\nu}{t+\alpha}
    +\frac{Z_2}{\tilde{\eta_t}}
    +\frac{\nu}{t+\alpha}\frac{1}{\tilde{\eta_t}^3}
    \leq\frac{\nu}{t+1+\alpha}\frac{1}{\tilde{\eta_t}^3}
    \nonumber \\&
    \Rightarrow
     \left(-\frac{\tilde{\mu}}{\tilde{\eta}_t}+Z_1\omega^2 \right)\frac{\nu}{t+\alpha}\frac{1}{\tilde{\eta}_t}
    +\frac{Z_2}{\tilde{\eta_t}}
    +\left(\frac{\nu}{t+\alpha}-\frac{\nu}{t+1+\alpha}\right)\frac{1}{\tilde{\eta_t}^3}
    \leq0
    \nonumber \\&
    \Rightarrow
     \left(-\frac{\tilde{\mu}}{\tilde{\eta}_t}+Z_1\omega^2 \right)\frac{\nu}{\tilde{\gamma}}
    +\frac{Z_2}{\tilde{\eta_t}}
    +\left(\frac{\nu}{t+\alpha}-\frac{\nu}{t+1+\alpha}\right)\frac{(t+\alpha)^3}{\tilde{\gamma}^3}
    \leq0
    \nonumber \\&
    \Rightarrow
     \left(-\frac{\tilde{\mu}}{\tilde{\eta}_t}+Z_1\omega^2 \right)\frac{\nu}{\tilde{\gamma}}
    +\frac{Z_2}{\tilde{\eta_t}}
    +\frac{\nu}{(t+\alpha)(t+\alpha+1)}\frac{(t+\alpha)^3}{\tilde{\gamma}^3}
    \leq0
    \nonumber \\&
    \Rightarrow
     \left(-\frac{\tilde{\mu}}{\tilde{\eta}_t}+Z_1\omega^2 \right)\frac{\nu}{\tilde{\gamma}}
    +\frac{Z_2}{\tilde{\eta_t}}
    +\frac{\nu}{t+\alpha+1}\frac{(t+\alpha)^2}{\tilde{\gamma}^3}
    \leq0
    \nonumber \\&
    \Rightarrow
      \tilde{\gamma}^2\left(-\frac{\tilde{\mu}}{\tilde{\eta}_t}+Z_1\omega^2 \right)\nu
    +\frac{Z_2}{\tilde{\eta_t}} \tilde{\gamma}^3
    +\frac{(t+\alpha)^2}{t+\alpha+1}\nu
    \leq0
    \nonumber \\&
    \Rightarrow
      \tilde{\gamma}^2\left(-\frac{\tilde{\mu}}{\tilde{\eta}_t}+Z_1\omega^2 \right)\nu
    +\frac{Z_2}{\tilde{\eta_t}} \tilde{\gamma}^3
    +\left(\frac{(t+\alpha+1)(t+\alpha-1)}{t+\alpha+1}\nu+\frac{\nu}{t+\alpha+1}\right)
    \leq 0,
\end{align}
where the last condition in~\eqref{eq:longtrasform} can be written as:
\begin{align}\label{eq:finTransform}
    \tilde{\gamma}^2\left(-\frac{\tilde{\mu}}{\tilde{\eta}_t}+Z_1\omega^2\right)\nu
    +\frac{Z_2}{\tilde{\eta}_t}\tilde{\gamma}^3
    +\nu[t+\alpha-1]
    +\frac{\nu}{t+1+\alpha}
    \leq 0.
\end{align}
Since the above condition needs to be satisfied $\forall t\geq 0$ and the expression on the left hand  side of the inequality is a convex function with respect to $t$  ($1/\eta_t$ is linear in $t$ and $\frac{1}{t+1+\alpha}$ is convex),
 it is sufficient to satisfy this condition for $t\to\infty $ and $t=0$. To obtain these limits, we first express~\eqref{eq:finTransform} as follows:
 \begin{align}\label{eq:fin2}
    \tilde{\gamma}^2\left(-\frac{\tilde{\mu}}{\tilde{\gamma}}(t+\alpha)+Z_1\omega^2 \right)\nu
    +Z_2 \tilde{\gamma}^2(t+\alpha)
    +\nu[t+\alpha-1]
    +\frac{\nu}{t+1+\alpha}
    \leq 0.
 \end{align}
 Upon $t\to\infty$ considering the dominant terms yields
 \begin{align}\label{eq:condinf}
     &-\tilde{\gamma}\tilde{\mu}\nu t
    +Z_2 \tilde{\gamma}^2t
    +\nu t
    \leq0
    \nonumber \\&
    \Rightarrow
    \left[1-\tilde{\gamma}\tilde{\mu}\right]\nu t
    +Z_2 \tilde{\gamma}^2 t
    \leq 0.
 \end{align}
To satisfy~\eqref{eq:condinf}, the necessary condition is given by:
\begin{equation}
    \tilde{\mu}\tilde{\gamma}-1>0,
\end{equation}
 \begin{align} \label{eq:nu1}
     \nu \geq \frac{\tilde{\gamma}^2Z_2}{\tilde{\mu}\tilde{\gamma}-1}.
 \end{align}
 Also, upon $t\rightarrow 0$, from~\eqref{eq:fin2} we have
 \begin{align}
    &\left(-\tilde{\mu}\tilde{\gamma}\alpha+Z_1\omega^2\tilde{\gamma}^2 \right)\nu
    +Z_2 \tilde{\gamma}^2\alpha
    +\nu[\alpha-1]
    +\frac{\nu}{1+\alpha}\leq0
    \nonumber \\&
    \Rightarrow
    \nu\left(\alpha(\tilde{\mu}\tilde{\gamma}-1)+\frac{\alpha}{1+\alpha}-Z_1\omega^2\tilde{\gamma}^2\right)
    \geq
  \tilde{\gamma}^2 Z_2\alpha,
 \end{align}
which implies the following conditions
\begin{align} \label{eq:omega_cond}
    \omega
    <
    \frac{1}{\tilde{\gamma}}\sqrt{\alpha\frac{\tilde{\mu}\tilde{\gamma}-1+\frac{1}{1+\alpha}}{Z_1}},
\end{align}
% \nm{note that we can make this arbitrarily large by making $\alpha$ large for ANY tau.. in fact, $Z_1\to 32\gamma/\mu(\tau-1)$ for $\alpha\to\infty$}
and
\begin{align} \label{eq:nu2}
    \nu\geq\frac{Z_2\alpha}{Z_1\left(\omega_{\max}^2-\omega^2\right)}.
\end{align}
Combining~\eqref{eq:nu1} and~\eqref{eq:nu2}, when $\omega$ satisfies~\eqref{eq:omega_cond}
and
\begin{align}
    \nu \geq Z_2\max\{\frac{\beta^2\gamma^2}{\mu\gamma-1},
    \frac{\alpha}{Z_1\left(\omega_{\max}^2-\omega^2\right)}\},
\end{align}
completes the induction and thus the proof.

\end{proof}

\section{Proof of Lemma \ref{lem:consensus-error}} \label{app:consensus-error}

\begin{lemma} \label{lem:consensus-error}
After performing $\Gamma^{(t)}_{{c}}$ rounds of consensus in cluster $\mathcal{S}_c$ with the consensus matrix $\mathbf{V}_{{c}}$, the consensus error $\mathbf e_i^{(t)}$ satisfies
\begin{equation} 
   \Vert \mathbf e_i^{(t)} \Vert  \hspace{-.5mm}  \leq (\lambda_{{c}})^{\Gamma^{(t)}_{{c}}}
\sqrt{s_c}\underbrace{\max_{j,j'\in\mathcal S_c}\Vert\tilde{\mathbf w}_j^{(t)}-\tilde{\mathbf w}_{j'}^{(t)}\Vert}_{\triangleq \Upsilon^{(t)}_{{c}}},
~ \forall i\in \mathcal{S}_c,
\end{equation}
where {$\lambda_{{c}}=\rho \big(\mathbf{V}_{{c}}-\frac{\textbf{1} \textbf{1}^\top}{s_c} \big)$.}
\end{lemma} 
\begin{proof}
    The evolution of the devices' parameters can be described by~\eqref{eq:consensus} as:
  \begin{equation}
      \mathbf{W}^{(t)}_{{c}}= \left(\mathbf{V}_{{c}}\right)^{\Gamma^{(t)}_{{c}}} \widetilde{\mathbf{W}}^{(t)}_{{c}}, ~t\in\mathcal T_k,
  \end{equation}
  where 
  \begin{align}
      \mathbf{W}^{(t)}_{{c}}=\left[\mathbf{w}^{(t)}_{{c_1}},\mathbf{w}^{(t)}_{{c_2}},\dots,\mathbf{w}^{(t)}_{{s_c}}\right]^\top
  \end{align}
  and 
  \begin{align}
      \widetilde{\mathbf{W}}^{(t)}_{{c}}=\left[\tilde{\mathbf{w}}^{(t)}_{{c_1}},\tilde{\mathbf{w}}^{(t)}_{{c_2}},\dots,\tilde{\mathbf{w}}^{(t)}_{{s_c}}\right]^\top.
  \end{align}
%   with $(\cdot)_q$ denote the $q$-th element of the vector. 
  
  Let matrix $\overline{\mathbf{W}}^{(t)}_{{c}}$ denote be the matrix with rows given by the average model parameters across the cluster, it can be represented as:
  \begin{align}
      \overline{\mathbf{W}}^{(t)}_{{c}}=\frac{\mathbf 1_{s_c} \mathbf 1_{s_c}^\top\widetilde{\mathbf{W}}^{(t)}_{{c}}}{s_c}.
  \end{align}
  We then define $\mathbf E^{(t)}_{{c}}$ as 
  \begin{align}
      \mathbf E^{(t)}_{{c}}= {\mathbf{W}}^{(t)}_{{c}}-\overline{\mathbf{W}}^{(t)}_{{c}}
      =[\left(\mathbf{V}_{{c}}\right)^{\Gamma^{(t)}_{{c}}}-\mathbf 1\mathbf 1^\top/s_c] [\widetilde{\mathbf{W}}^{(t)}_{{c}}-\overline{\mathbf{W}}^{(t)}_{{c}}]
      ,
  \end{align}
  so that $[\mathbf E^{(t)}_{{c}}]_{i,:}=\mathbf e_{i}^{(t)}$, where $[\mathbf E^{(t)}_{{c}}]_{i,:}$ is the $i$th row of $\mathbf E^{(t)}_{{c}}$.
  
   Therefore, using Assumption \ref{assump:cons}, we can bound the consensus error as
  \begin{align} \label{eq:consensus-bound-1}
      &\Vert\mathbf e_{i}^{(t)}\Vert^2\leq
      \mathrm{trace}((\mathbf E^{(t)}_{{c}})^{\top}\mathbf E^{(t)}_{{c}})
      \\&\nonumber
      =
      \mathrm{trace}\Big(
      [\widetilde{\mathbf{W}}^{(t)}_{{c}}{-}\overline{\mathbf{W}}^{(t)}_{{c}}]^\top
      [\left(\mathbf{V}_{{c}}\right)^{\Gamma^{(t)}_{{c}}}{-}\mathbf 1\mathbf 1^\top/s_c]^2 [\widetilde{\mathbf{W}}^{(t)}_{{c}}{-}\overline{\mathbf{W}}^{(t)}_{{c}}]
        \Big)
        \\&
        \leq (\lambda_{{c}})^{2\Gamma^{(t)}_{{c}}}
        \sum_{j=1}^{s_c}\Vert\tilde{\mathbf w}_j^{(t)}-\bar{\mathbf w}_c^{(t)}\Vert^2
        \nonumber
        \\&
        \leq (\lambda_{{c}})^{2\Gamma^{(t)}_{{c}}}
        \frac{1}{s_c}\sum_{j,j'=1}^{s_c}\Vert\tilde{\mathbf w}_j^{(t)}-\tilde{\mathbf w}_{j'}^{(t)}\Vert^2
        \nonumber
        \\&
        \nonumber
        \leq (\lambda_{{c}})^{2\Gamma^{(t)}_{{c}}}
        s_c\max_{j,j'\in\mathcal S_c}\Vert\tilde{\mathbf w}_j^{(t)}-\tilde{\mathbf w}_{j'}^{(t)}\Vert^2.
    \end{align}
  The result of the Lemma directly follows.
\end{proof}
\section{Proof of Lemma \ref{lem1}} \label{app:PL-bound}
 
\begin{lemma} \label{lem1}
Under Assumption \ref{beta}, we have
%\nm{$\eta_t$ not $\eta_k$!}
    \begin{align*}
        &-\frac{\tilde{\eta}_{t}}{\beta}\nabla F(\bar{\mathbf w}^{(t)})^\top \sum\limits_{c=1}^N\varrho_c\frac{1}{s_c}\sum\limits_{j\in\mathcal{S}_c} \nabla F_j(\mathbf w_j^{(t)})
        \leq
        -\tilde{\mu}\tilde{\eta}_{t}(F(\bar{\mathbf w}^{(t)})-F(\mathbf w^*))
        \nonumber \\&
        -\frac{\tilde{\eta}_{t}}{2\beta}\Big\Vert\sum\limits_{c=1}^N\varrho_c\frac{1}{s_c}\sum\limits_{j\in\mathcal{S}_c} \nabla F_j(\mathbf w_j^{(t)})\Big\Vert^2
        +\frac{\tilde{\eta}_{t}\beta}{2}\sum\limits_{c=1}^N\varrho_c \frac{1}{s_c}\sum\limits_{j\in\mathcal{S}_c}\Big\Vert\bar{\mathbf w}^{(t)}-\mathbf w_j^{(t)}\Big\Vert^2.
    \end{align*}
\end{lemma}

\begin{proof}  Since $-2 \mathbf a^\top \mathbf b = -\Vert \mathbf a\Vert^2-\Vert \mathbf b \Vert^2 + \Vert \mathbf a- \mathbf b\Vert^2$ holds for any two vectors $\mathbf a$ and $\mathbf b$ with real elements, we have
    \begin{align} \label{14}
        &-\frac{\tilde{\eta}_{t}}{\beta}\nabla F(\bar{\mathbf w}^{(t)})^\top
        \sum\limits_{c=1}^N\varrho_c\frac{1}{s_c}\sum\limits_{j\in\mathcal{S}_c} \nabla F_j(\mathbf w_j^{(t)})
        \nonumber \\&
        =\frac{\tilde{\eta}_{t}}{2\beta}\bigg[-\Big\Vert\nabla F(\bar{\mathbf w}^{(t)})\Big\Vert^2
        -\Big\Vert\sum\limits_{c=1}^N\varrho_c\frac{1}{s_c}\sum\limits_{j\in\mathcal{S}_c} \nabla F_j(\mathbf w_j^{(t)})\Big\Vert^2
        \nonumber \\&
        +\Big\Vert\nabla F(\bar{\mathbf w}^{(t)})-\sum\limits_{c=1}^N\varrho_c\frac{1}{s_c}\sum\limits_{j\in\mathcal{S}_c} \nabla F_j(\mathbf w_j^{(t)})\Big\Vert^2\bigg].
    \end{align}
    Since $\Vert\cdot\Vert^2$ is a convex function, using Jenson's inequality, we get: $\Vert \sum_{i=1}^j {c_j} \mathbf a_j  \Vert^2 \leq \sum_{i=1}^j {c_j} \Vert  \mathbf a_j  \Vert^2 $, where $\sum_{i=1}^j {c_j} =1$. Using this fact in~\eqref{14} yields
    \begin{align} \label{20}
        &-\frac{\tilde{\eta}_{t}}{\beta}\nabla F(\bar{\mathbf w}^{(t)})^\top \sum\limits_{c=1}^N\varrho_c\frac{1}{s_c}\sum\limits_{j\in\mathcal{S}_c} \nabla F_j(\mathbf w_j^{(t)})
        \nonumber \\&
        \leq\frac{\tilde{\eta}_{t}}{2\beta}\bigg[-\Big\Vert\nabla F(\bar{\mathbf w}^{(t)})\Big\Vert^2
        -\Big\Vert\sum\limits_{c=1}^N\varrho_c\frac{1}{s_c}\sum\limits_{j\in\mathcal{S}_c} \nabla F_j(\mathbf w_j^{(t)})\Big\Vert^2
        \nonumber \\&
        +\sum\limits_{c=1}^N\varrho_c\frac{1}{s_c}\sum\limits_{j\in\mathcal{S}_c}\Big\Vert\nabla F_j(\bar{\mathbf w}^{(t)})-\nabla F_j(\mathbf w_j^{(t)})\Big\Vert^2\bigg].
    \end{align}
    
   Using $\mu$-strong convexity of $F(.)$, we get: $\Big\Vert\nabla F(\bar{\mathbf w}^{(t)})\Big\Vert^2 \geq 2\tilde{\mu}\beta (F(\bar{\mathbf w}^{(t)})-F(\mathbf w^*))$. Also, using $\beta$-smoothness of $F_j(\cdot)$ we get $\Big\Vert\nabla F_j(\bar{\mathbf w}^{(t)})-\nabla F_j(\mathbf w_j^{(t)})\Big\Vert^2 \leq \beta^2 \Vert\bar{\mathbf w}^{(t)}-\mathbf w_j^{(t)}\Big\Vert^2$, $\forall j$. Using these facts in~\eqref{20} yields:
    \begin{align}
        &-\frac{\tilde{\eta}_{t}}{\beta}\nabla F(\bar{\mathbf w}^{(t)})^\top \sum\limits_{c=1}^N\varrho_c\frac{1}{s_c}\sum\limits_{j\in\mathcal{S}_c} \nabla F_j(\mathbf w_j^{(t)})
        \leq
        -\tilde{\mu}\tilde{\eta}_{t}(F(\bar{\mathbf w}^{(t)})-F(\mathbf w^*))
        \nonumber \\&
        -\frac{\tilde{\eta}_{t}}{2\beta}\Big\Vert\sum\limits_{c=1}^N\varrho_c\frac{1}{s_c}\sum\limits_{j\in\mathcal{S}_c} \nabla F_j(\mathbf w_j^{(t)})\Big\Vert^2
        +\frac{\tilde{\eta}_{t}\beta}{2}\sum\limits_{c=1}^N\varrho_c \frac{1}{s_c}\sum\limits_{j\in\mathcal{S}_c}\Big\Vert\bar{\mathbf w}^{(t)}-\mathbf w_j^{(t)}\Big\Vert^2,
    \end{align}
    which concludes the proof.
\end{proof}

\section{Proof of Fact \ref{fact:1}} \label{app:fact1}
\begin{fact}\label{fact:1} For an arbitrary set of $n$ random variables $x_1,\cdots,x_n$, we have: 
 \begin{equation}
     \sqrt{\mathbb E\left[\left(\sum\limits_{i=1}^{n} x_i\right)^2\right]}\leq \sum\limits_{i=1}^{n} \sqrt{\mathbb E[x_i^2]}.
 \end{equation}
\end{fact}
\begin{proof} The proof can be carried out through the following set of algebraic manipulations:
    % Using the following result of Cauchy-Schwarz Inequality
    % \begin{align}
    %     \mathbb E[XY] \leq \sqrt{\mathbb E[X^2]\mathbb E[ Y^2]},
    % \end{align}
    % we obtain
    \begin{align}
        &\sqrt{\mathbb E\left[\left(\sum\limits_{i=1}^{n} x_i\right)^2\right]}
        =
        \sqrt{\sum\limits_{i=1}^{n} \mathbb E[x_i^2] +\sum\limits_{i=1}^{n}\sum\limits_{j=1,j\neq i}^{n}\mathbb E [x_i x_j]}
        \nonumber \\&
        \overset{(a)}{\leq}
        \sqrt{\sum\limits_{i=1}^{n} \mathbb E[x_i^2] +\sum\limits_{i=1}^{n}\sum\limits_{j=1,j\neq i}^{n}\sqrt{\mathbb E [x_i^2] \mathbb E[x_j^2]]}}
        % \nonumber \\&
        =\sqrt{\Big(\sum\limits_{i=1}^{n} \sqrt{\mathbb E[x_i^2]}\Big)^2}
        =
        \sum\limits_{i=1}^{n} \sqrt{\mathbb E[x_i^2]},
    \end{align}
    where $(a)$ is due to the fact that $\mathbb E[XY] \leq \sqrt{\mathbb E[X^2]\mathbb E[ Y^2]}$ resulted from Cauchy-Schwarz inequality.
\end{proof}
% \begin{proof}
%     \begin{align}
%         &\sqrt{\mathbb E[X+Y]^2}=\sqrt{\mathbb E[X^2+2XY+Y^2]}
%         =\sqrt{\mathbb E[X^2]+2\mathbb E[XY]+\mathbb[Y^2]]}
%         \nonumber \\&
%         \leq
%         \sqrt{\mathbb E[X^2]+\mathbb[Y^2]]+2\sqrt{\mathbb E [X^2] \mathbb E[Y^2]]}}
%         =\sqrt{\mathbb E[X^2]}+\sqrt{\mathbb E[Y^2]}.
%     \end{align}
% \end{proof}

\newpage

\section{Additional Experimental Results}
\label{app:experiments}
\subsection{Complimentary Experiments of the Main Text}
This section presents the plots from complimentary experiments mentioned in Sec.~\ref{sec:experiments}. We use an additional dataset, Fashion-MNIST (FMNIST), and fully connected neural networks (FCN) for these additional experiments. FMNIST (\url{https://github.com/zalandoresearch/fashion-mnist}) is a dataset of clothing articles consisting of a training set of 60,000 samples and a test set of 10,000 samples. Each sample is a 28x28 grayscale image, associated with a label from 10 classes.

In the following, we explain the relationship between the figures presented in this appendix and the results presented in the main text. Overall, we find that the results are qualitatively similar to what was observed for the SVM and MNIST cases:
\begin{itemize}
    \item Fig.~\ref{fig:mnist_poc_1_all} from main text is repeated in Fig.~\ref{fig:fmnist_poc_1_all} for FMNIST dataset using SVM, Fig.~\ref{fig:mnist_nn_poc_1_all} for MNIST dataset using FCN, and Fig.~\ref{fig:fmnist_nn_poc_1_all} for FMNIST dataset using FCN.
    
    \item Fig.~\ref{fig:mnist_poc_2_all} from main text is repeated in Fig.~\ref{fig:fmnist_poc_2_all} for FMNIST dataset using SVM, Fig.~\ref{fig:mnist_nn_poc_2_all} for MNIST dataset using FCN, and Fig.~\ref{fig:fmnist_nn_poc_2_all} for FMNIST dataset using FCN.
    
    \item Fig.~\ref{fig:poc_3_iid_1_gamma_1_lut_50} from main text is repeated in Fig.~\ref{fig:fmnist_poc_3_iid_1_gamma_1_lut_50} for FMNIST dataset using SVM, Fig.~\ref{fig:mnist_nn_poc_3_iid_1_gamma_1_lut_50} for MNIST dataset using FCN, and Fig.~\ref{fig:fmnist_nn_poc_3_iid_1_gamma_1_lut_50} for FMNIST dataset using FCN.
    
    \item Fig.~\ref{fig:resource_bar_0} from main text is repeated in Fig.~\ref{fig:fmnist_resource_bar_0} for FMNIST dataset using SVM, Fig.~\ref{fig:mnist_nn_resource_bar_0} for MNIST dataset using FCN, and Fig.~\ref{fig:fmnist_nn_resource_bar_0} for FMNIST dataset using FCN.
    
    \item Fig.~\ref{fig:behaviour_of_tau} from main text is repeated in Fig.~\ref{fig:fmnist_behaviour_of_tau} for FMNIST dataset using SVM.
\end{itemize}
Since FCN has a non-convex loss function, Algorithm~\ref{GT} is not directly applicable for the experiments in Fig.~\ref{fig:mnist_nn_resource_bar_0}\&\ref{fig:fmnist_nn_resource_bar_0}. As a result, in these cases, we skip the control steps in line 24-25. We instead use a fixed step size, using a constant $\phi$ value to calculate the $\Gamma$'s using \eqref{eq:consensusNum}. We are still able to obtain comparable reductions in total cost compared with the federated learning baselines.

\subsection{Extension to Other Federated Learning Methods}\label{fedProx_app}
Although we develop our algorithm based on federated learning with vanilla SGD local optimizer, our method can benefit other counterparts in literature. In particular, we perform some numerical experiments on FedProx~\cite{sahu2018convergence} to demonstrate the superiority of our semi-decentralized architecture. The performance improvement is due to the fact that, intuitively, conducting D2D communications via the method proposed by us reduces the local bias of the nodes' models to their local datasets. This benefits the convergence of federated learning methods via counteracting the effect of data heterogeneity across the nodes. The simulation results are provided in Fig.~\ref{fig:prox1_app} and~\ref{fig:prox2_app}.

\begin{figure}[t]
\includegraphics[width=1.0\columnwidth]{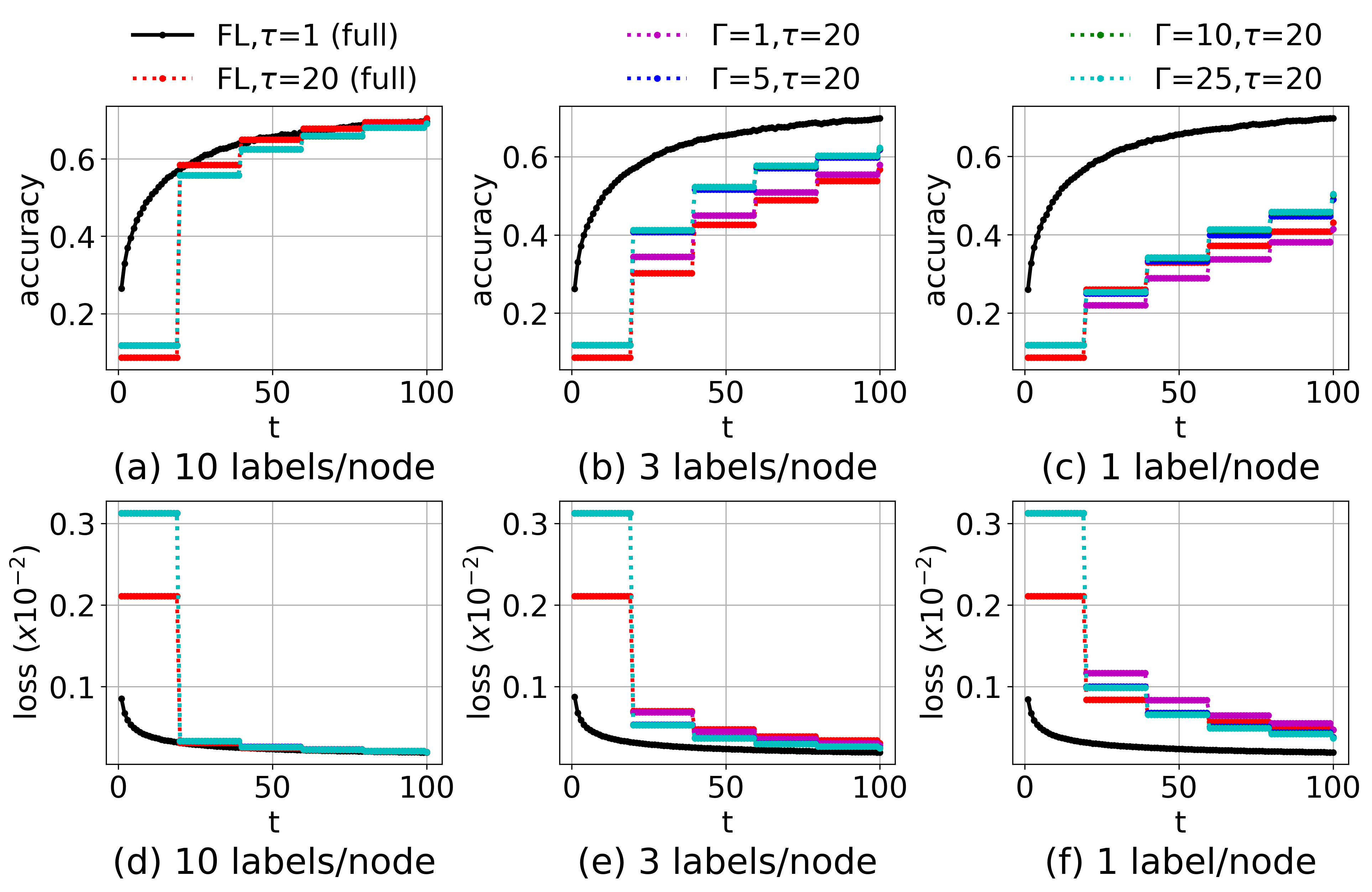}
\centering
\caption{Performance comparison between {\tt TT-HF} and baseline methods when varying the number of D2D consensus rounds ($\Gamma$). Under the same period of local model training ($\tau$), increasing $\Gamma$ results in a considerable improvement in the model accuracy/loss over time as compared to the current art~\cite{wang2019adaptive,Li} when data is non-i.i.d. (FMNIST, SVM)}
\label{fig:fmnist_poc_1_all}
\vspace{-5mm}
\end{figure} 

\begin{figure}[t]
\includegraphics[width=1.0\columnwidth]{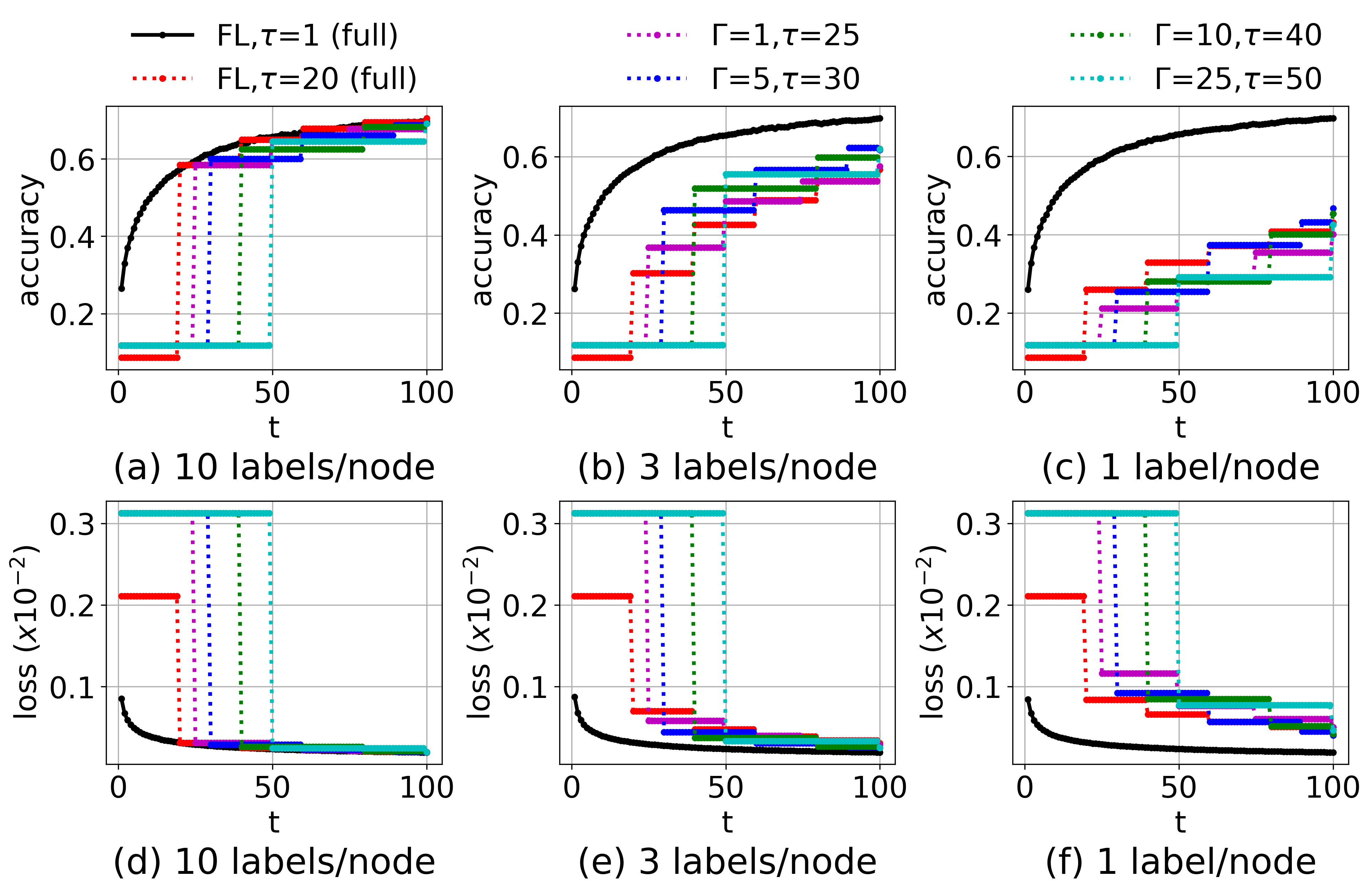}
\centering
\caption{Performance comparison between {\tt TT-HF} and baseline methods when varying the local model training interval ($\tau$) and the number of D2D consensus rounds ($\Gamma$). With a larger $\tau$, {\tt TT-HF} can still outperform the baseline federated learning~\cite{wang2019adaptive,Li} if $\Gamma$ is increased, i.e., local D2D communications can be used to offset the frequency of global aggregations. (FMNIST, SVM)}
\label{fig:fmnist_poc_2_all}
\vspace{-5mm}
\end{figure}

\begin{figure}[t]
\includegraphics[width=1.0\columnwidth]{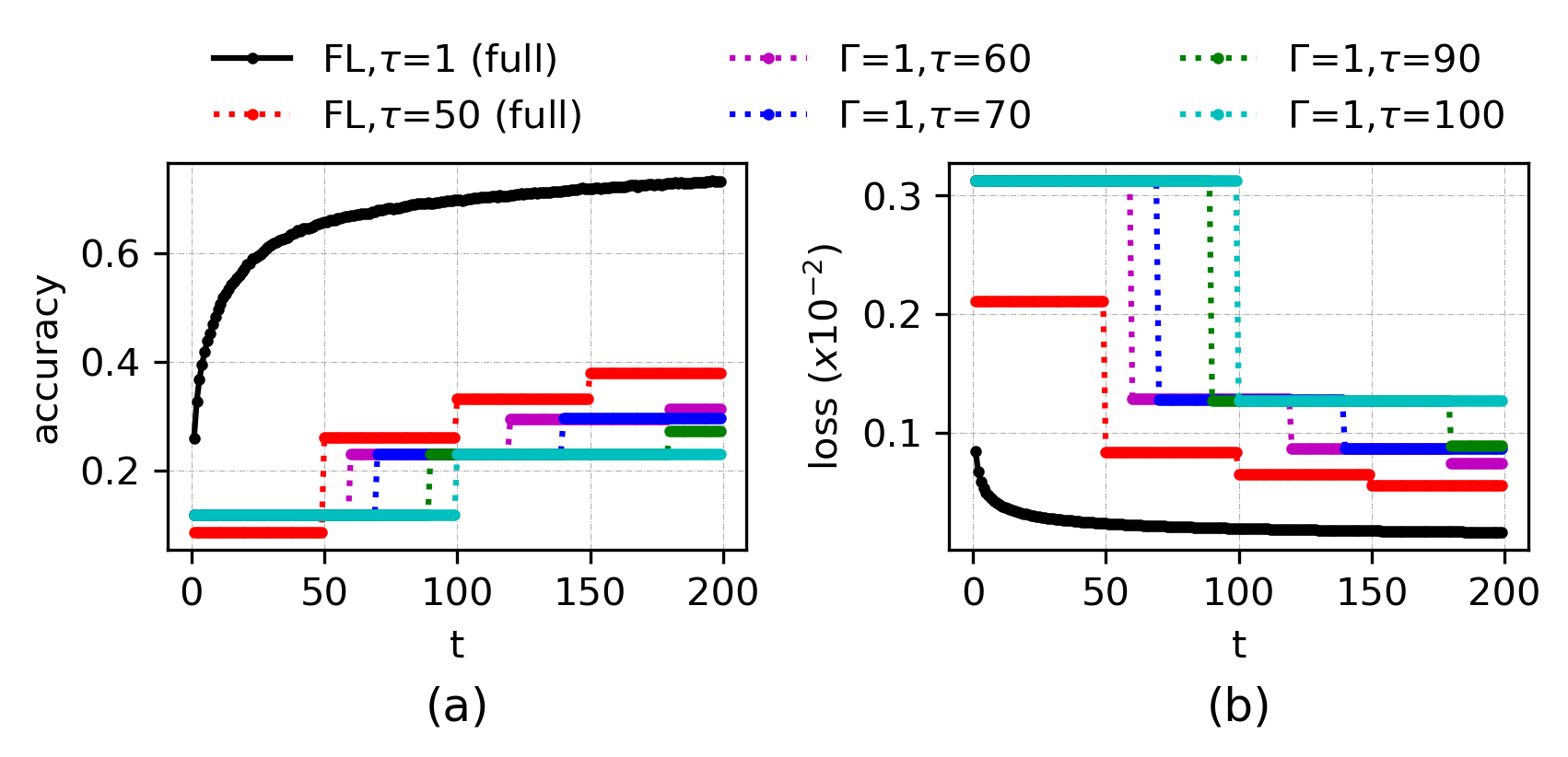}
\centering
\caption{Performance of {\tt TT-HF} in the extreme non-i.i.d. case for the setting in Fig.~\ref{fig:mnist_poc_2_all} when $\Gamma$ is small and the local model training interval length is increased substantially. {\tt TT-HF} exhibits poor convergence behavior when $\tau$ exceeds a certain value, due to model dispersion. (FMNIST, SVM)}
\label{fig:fmnist_poc_3_iid_1_gamma_1_lut_50}
\vspace{-5mm}
\end{figure}

\begin{figure}[t]
\hspace{-3mm}
\includegraphics[width=0.99\columnwidth]{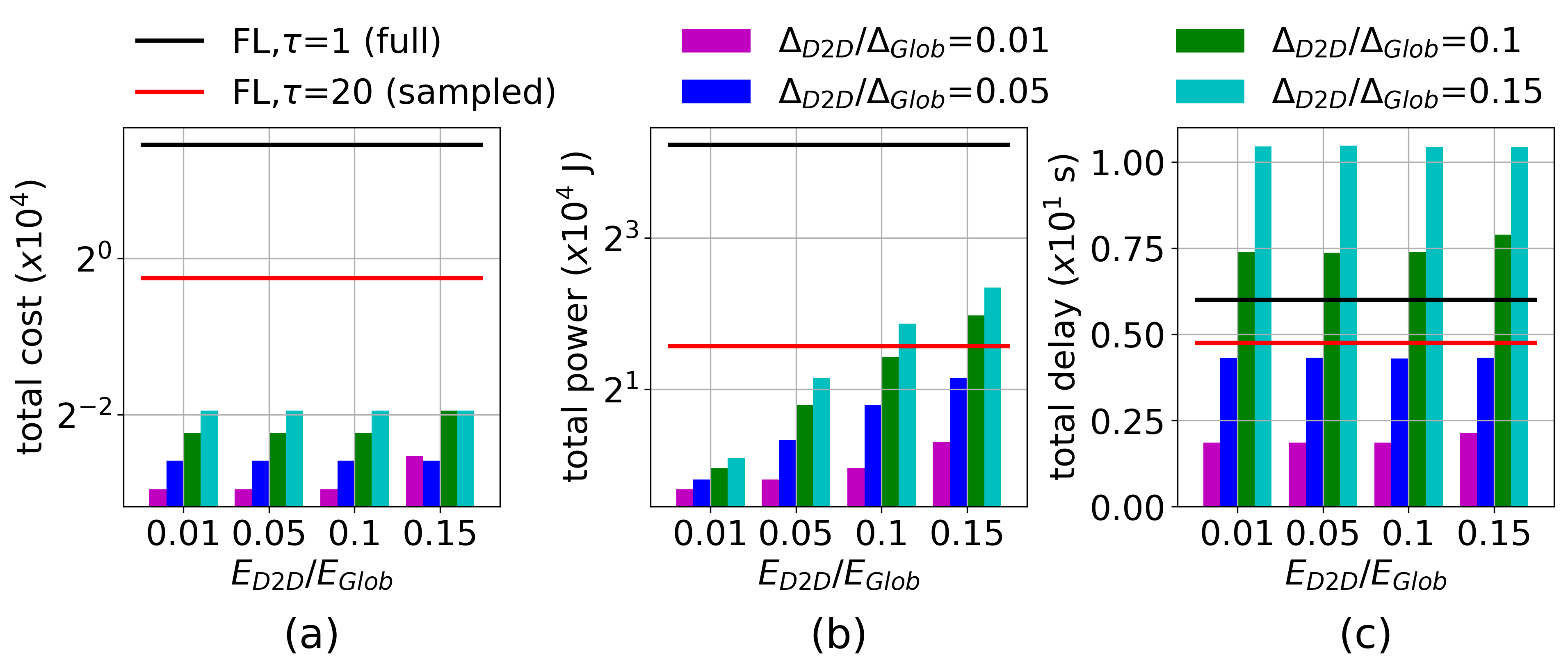}
\centering
\caption{Comparing total (a) cost, (b) power, and (c) delay metrics from the optimization objective in $(\bm{\mathcal{P}})$ achieved by {\tt TT-HF} versus baselines upon reaching $75\%$ of peak accuracy, for different configurations of delay and energy consumption. {\tt TT-HF} obtains a significantly lower total cost in (a). (b) and (c) demonstrate the region under which {\tt TT-HF} attains energy savings and delay gains. (FMNIST, SVM)}
\label{fig:fmnist_resource_bar_0}
\vspace{-5mm}
\end{figure}

\begin{figure}[t]
\vspace{4mm}
\hspace{-2.5mm}
\includegraphics[width=1\columnwidth]{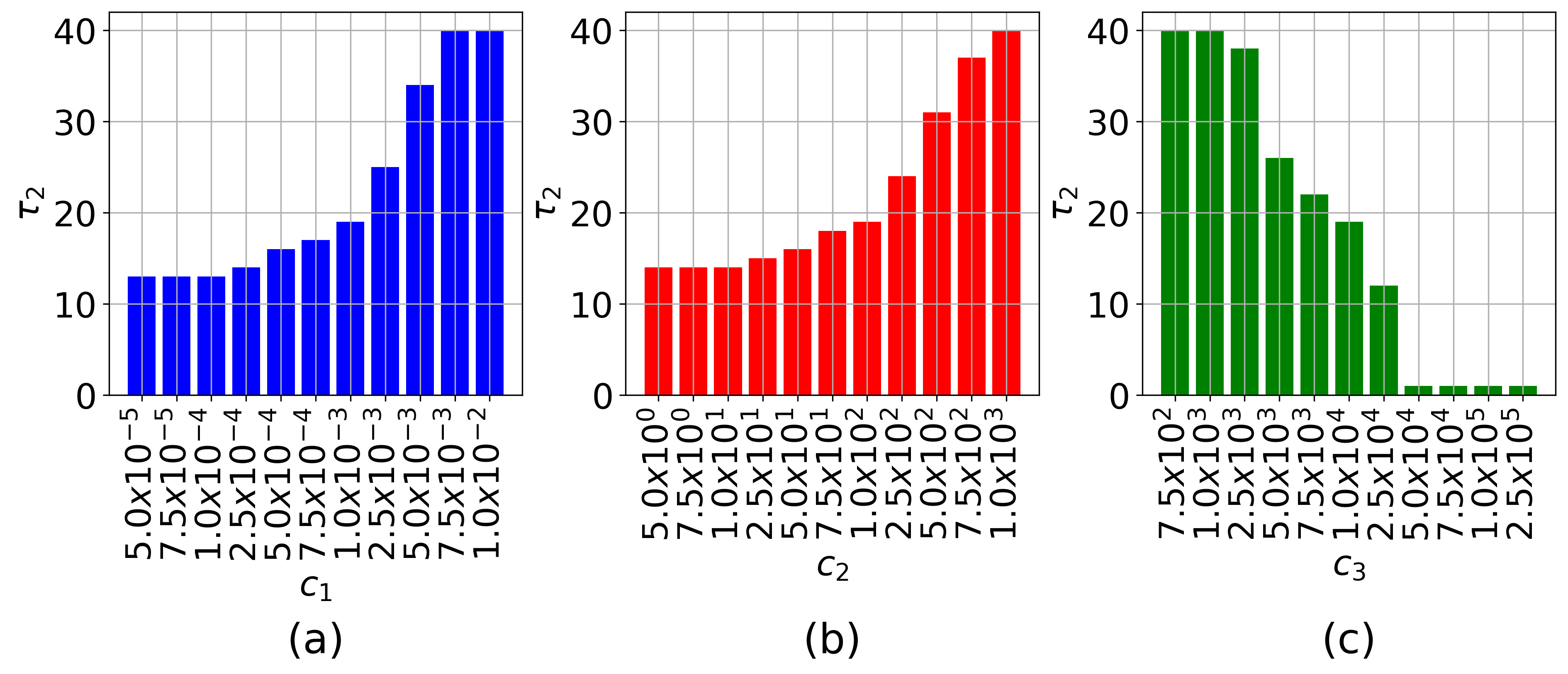}
\centering
\caption{Value of the second local model training interval obtained through $(\bm{\mathcal{P}})$ for different configurations of weighing coefficients $c_1, c_2, c_3$ (default $c_1 = 10^{-3}, c_2=10^2, c_3=10^4$). Higher weight on energy and delay (larger $c_1$ and $c_2$) prolongs the local training period, while higher weight on the global model loss (larger $c_3$) decreases the length, resulting in more rapid global aggregations. (FMNIST, SVM)
}
\label{fig:fmnist_behaviour_of_tau}
\vspace{-5mm}
\end{figure}

\begin{figure}[t]
\includegraphics[width=1.0\columnwidth]{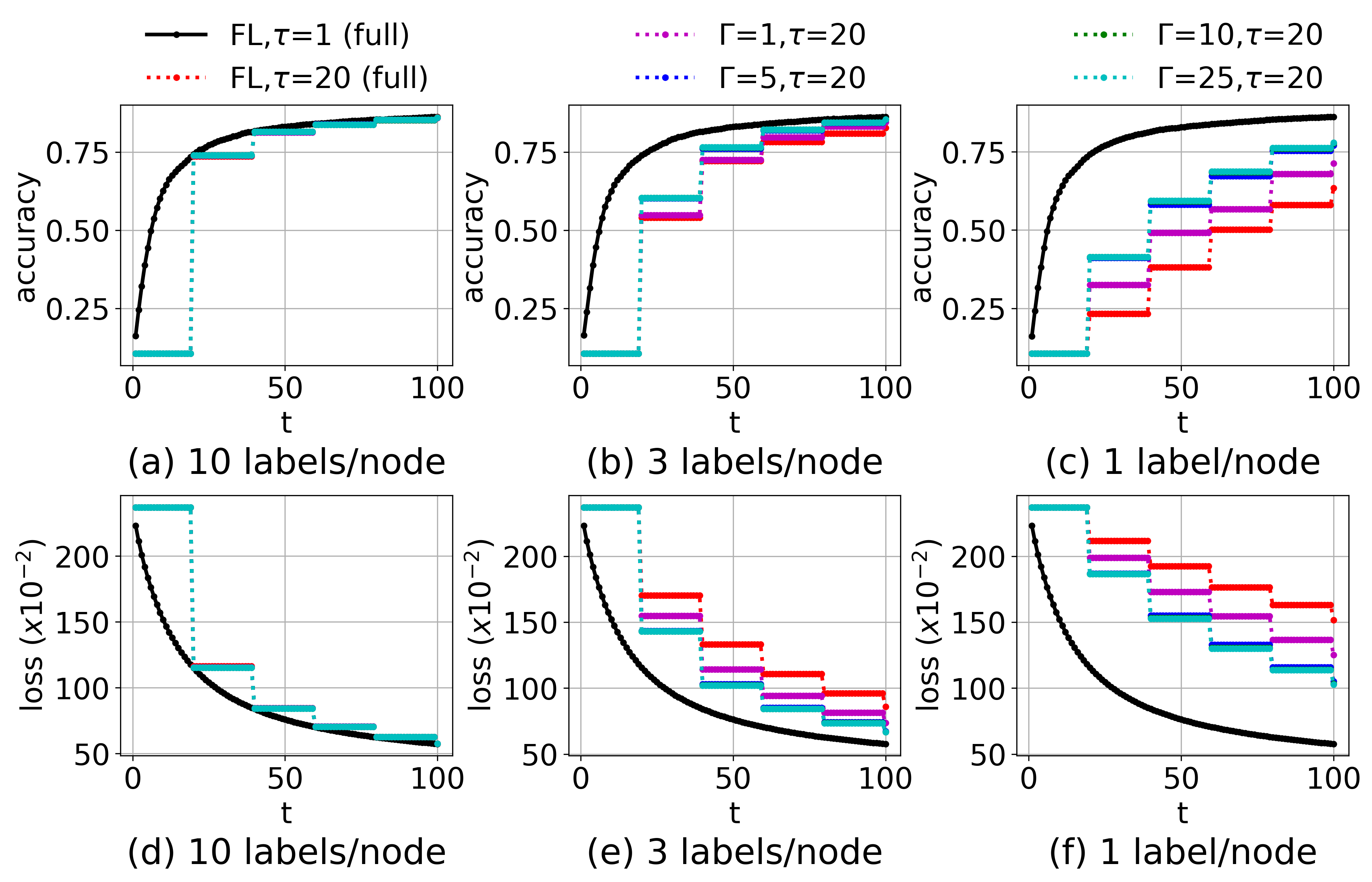}
\centering
\caption{Performance comparison between {\tt TT-HF} and baseline methods when varying the number of D2D consensus rounds ($\Gamma$). Under the same period of local model training ($\tau$), increasing $\Gamma$ results in a considerable improvement in the model accuracy/loss over time as compared to the current art~\cite{wang2019adaptive,Li} when data is non-i.i.d. (MNIST, Neural Network)}
\label{fig:mnist_nn_poc_1_all}
\vspace{-5mm}
\end{figure} 

\begin{figure}[t]
\includegraphics[width=1.0\columnwidth]{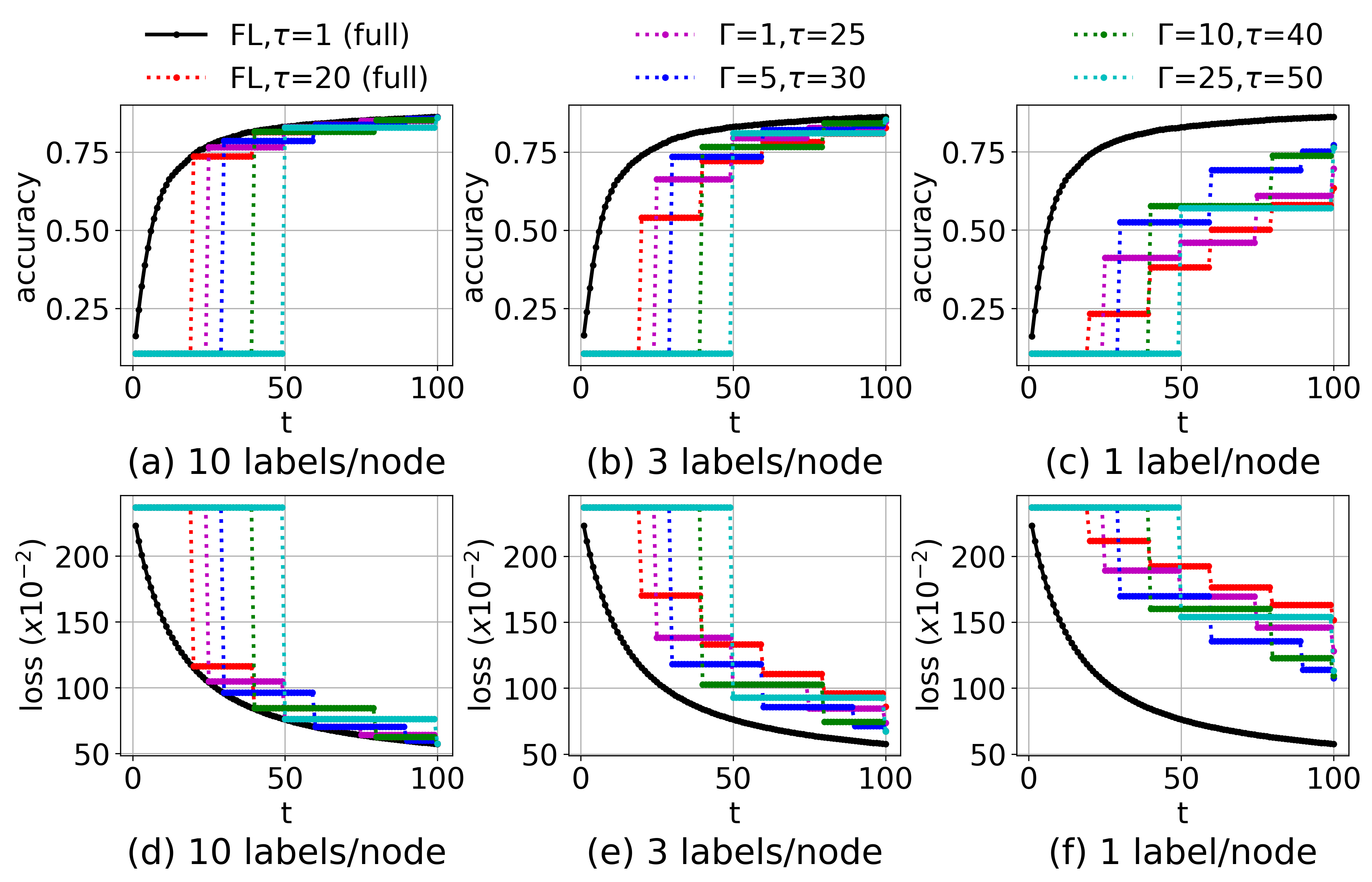}
\centering
\caption{Performance comparison between {\tt TT-HF} and baseline methods when varying the local model training interval ($\tau$) and the number of D2D consensus rounds ($\Gamma$). With a larger $\tau$, {\tt TT-HF} can still outperform the baseline federated learning~\cite{wang2019adaptive,Li} if $\Gamma$ is increased, i.e., local D2D communications can be used to offset the frequency of global aggregations. (MNIST, Neural Network)}
\label{fig:mnist_nn_poc_2_all}
\vspace{-5mm}
\end{figure}

\begin{figure}[t]
\includegraphics[width=1.0\columnwidth]{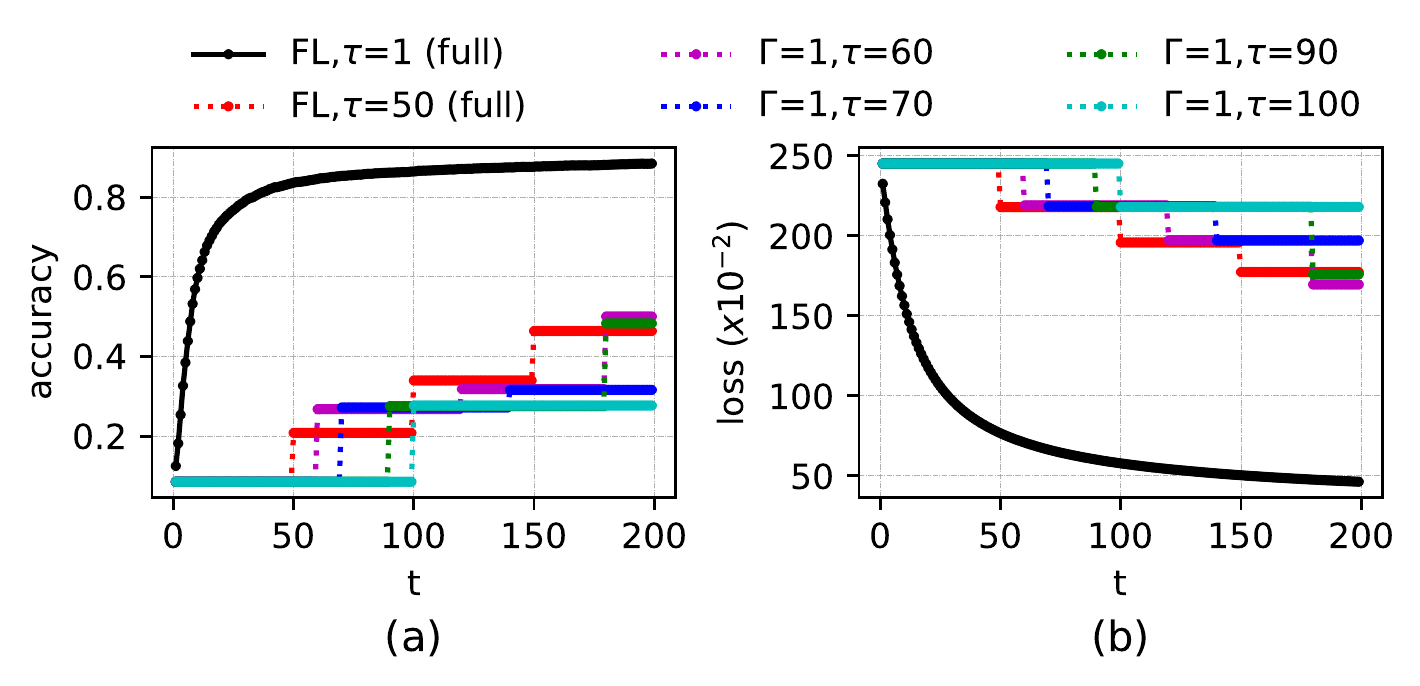}
\centering
\caption{Performance of {\tt TT-HF} in the extreme non-i.i.d. case for the setting in Fig.~\ref{fig:mnist_poc_2_all} when $\Gamma$ is small and the local model training interval length is increased substantially. {\tt TT-HF} exhibits poor convergence behavior when $\tau$ exceeds a certain value, due to model dispersion. (MNIST, Neural Network)}
\label{fig:mnist_nn_poc_3_iid_1_gamma_1_lut_50}
\vspace{-5mm}
\end{figure}

\begin{figure}[t]
\hspace{-3mm}
\includegraphics[width=0.99\columnwidth]{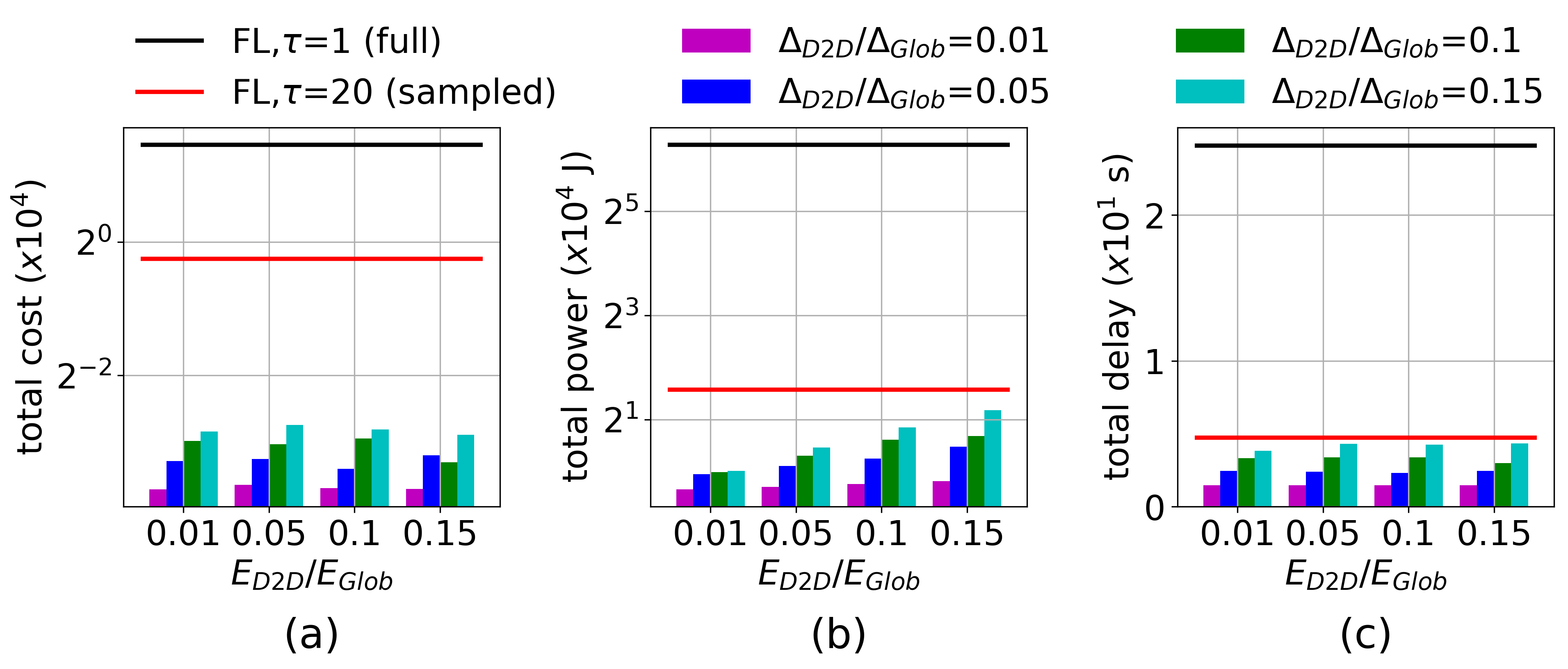}
\centering
\caption{Comparing total (a) cost, (b) power, and (c) delay metrics from the optimization objective in $(\bm{\mathcal{P}})$ achieved by {\tt TT-HF} versus baselines upon reaching $75\%$ of peak accuracy, for different configurations of delay and energy consumption. {\tt TT-HF} obtains a significantly lower total cost in (a). (b) and (c) demonstrate the region under which {\tt TT-HF} attains energy savings and delay gains. (FMNIST, SVM)}
\label{fig:mnist_nn_resource_bar_0}
\vspace{-5mm}
\end{figure}

\begin{figure}[t]
\includegraphics[width=1.0\columnwidth]{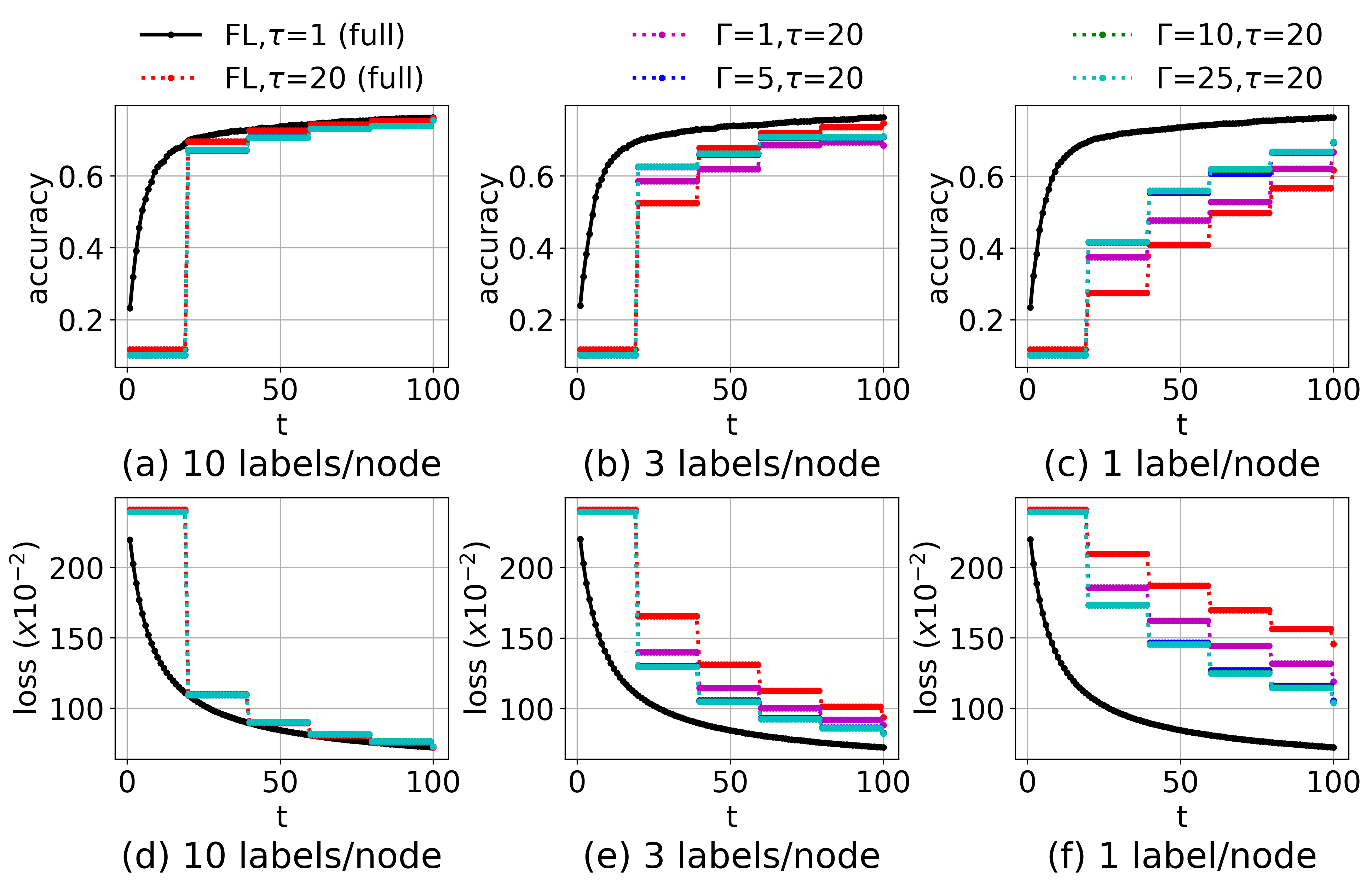}
\centering
\caption{Performance comparison between {\tt TT-HF} and baseline methods when varying the number of D2D consensus rounds ($\Gamma$). Under the same period of local model training ($\tau$), increasing $\Gamma$ results in a considerable improvement in the model accuracy/loss over time as compared to the current art~\cite{wang2019adaptive,Li} when data is non-i.i.d. (FMNIST, Neural Network)}
\label{fig:fmnist_nn_poc_1_all}
\vspace{-5mm}
\end{figure} 

\begin{figure}[t]
\includegraphics[width=1.0\columnwidth]{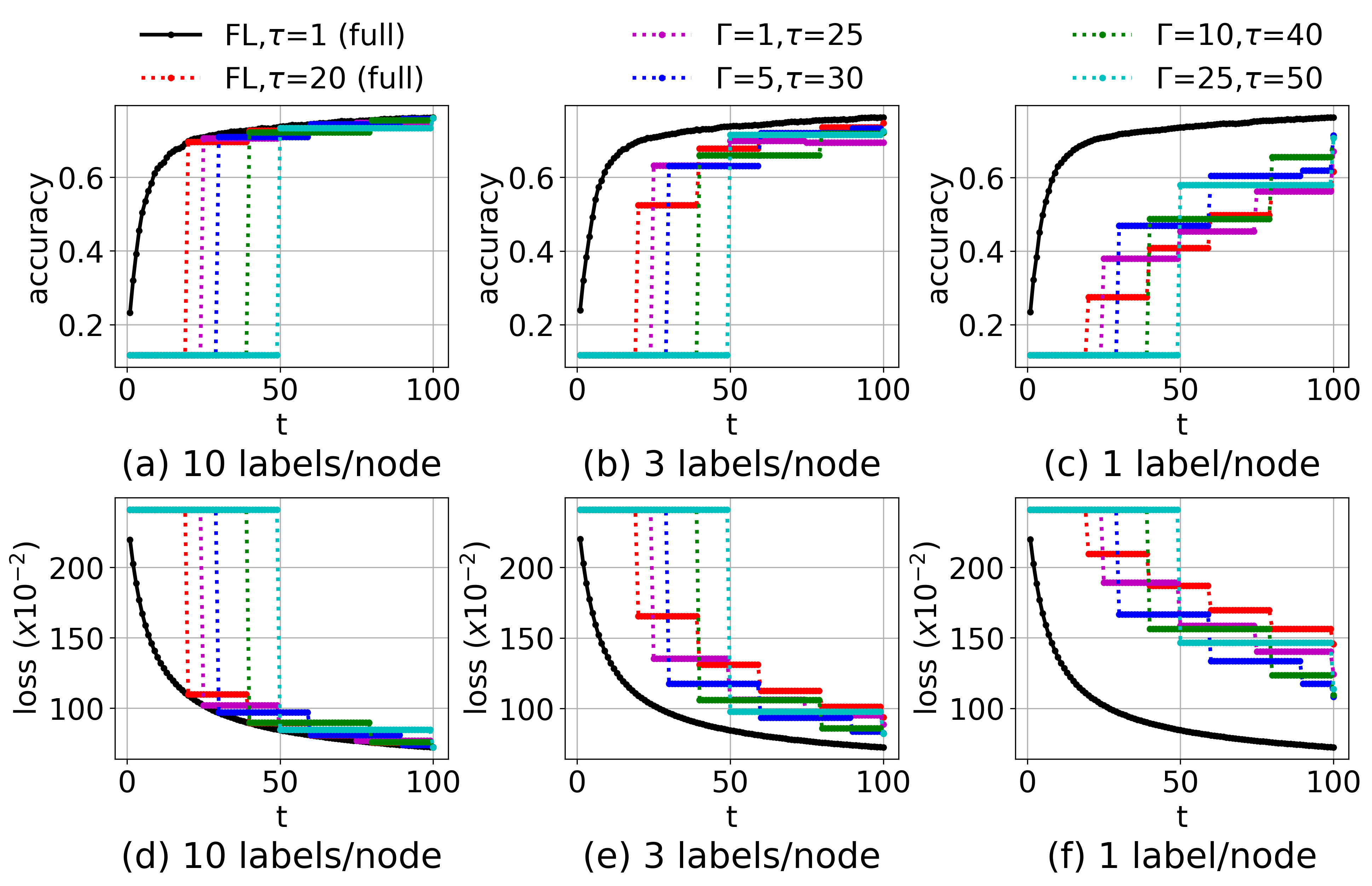}
\centering
\caption{Performance comparison between {\tt TT-HF} and baseline methods when varying the local model training interval ($\tau$) and the number of D2D consensus rounds ($\Gamma$). With a larger $\tau$, {\tt TT-HF} can still outperform the baseline federated learning~\cite{wang2019adaptive,Li} if $\Gamma$ is increased, i.e., local D2D communications can be used to offset the frequency of global aggregations. (FMNIST, Neural Network)}
\label{fig:fmnist_nn_poc_2_all}
\vspace{-5mm}
\end{figure}

\begin{figure}[t]
\includegraphics[width=1.0\columnwidth]{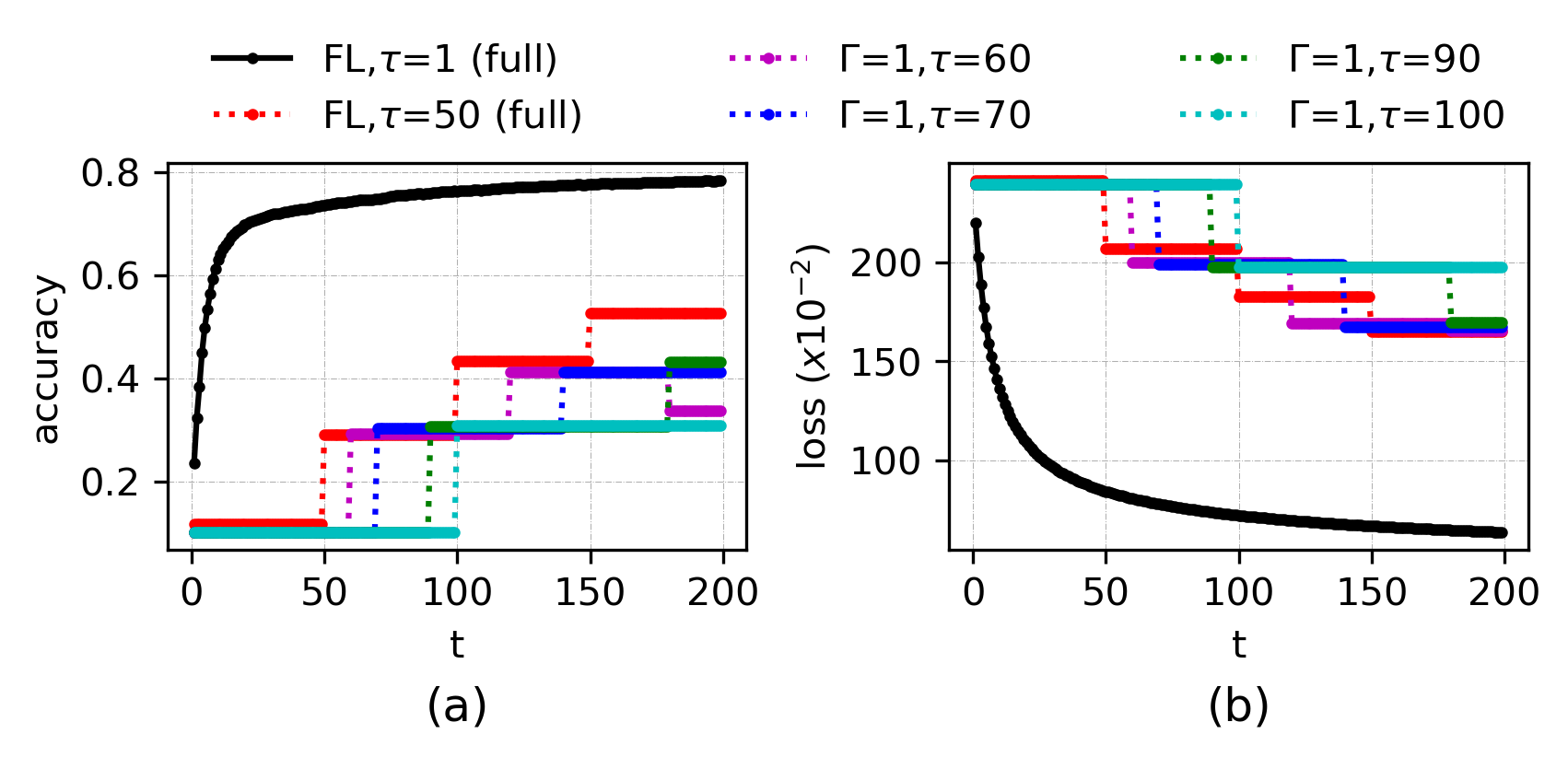}
\centering
\caption{Performance of {\tt TT-HF} in the extreme non-i.i.d. case for the setting in Fig.~\ref{fig:mnist_poc_2_all} when $\Gamma$ is small and the local model training interval length is increased substantially. {\tt TT-HF} exhibits poor convergence behavior when $\tau$ exceeds a certain value, due to model dispersion. (FMNIST, Neural Network)}
\label{fig:fmnist_nn_poc_3_iid_1_gamma_1_lut_50}
\vspace{-5mm}
\end{figure}

\begin{figure}[t]
\hspace{-3mm}
\includegraphics[width=0.99\columnwidth]{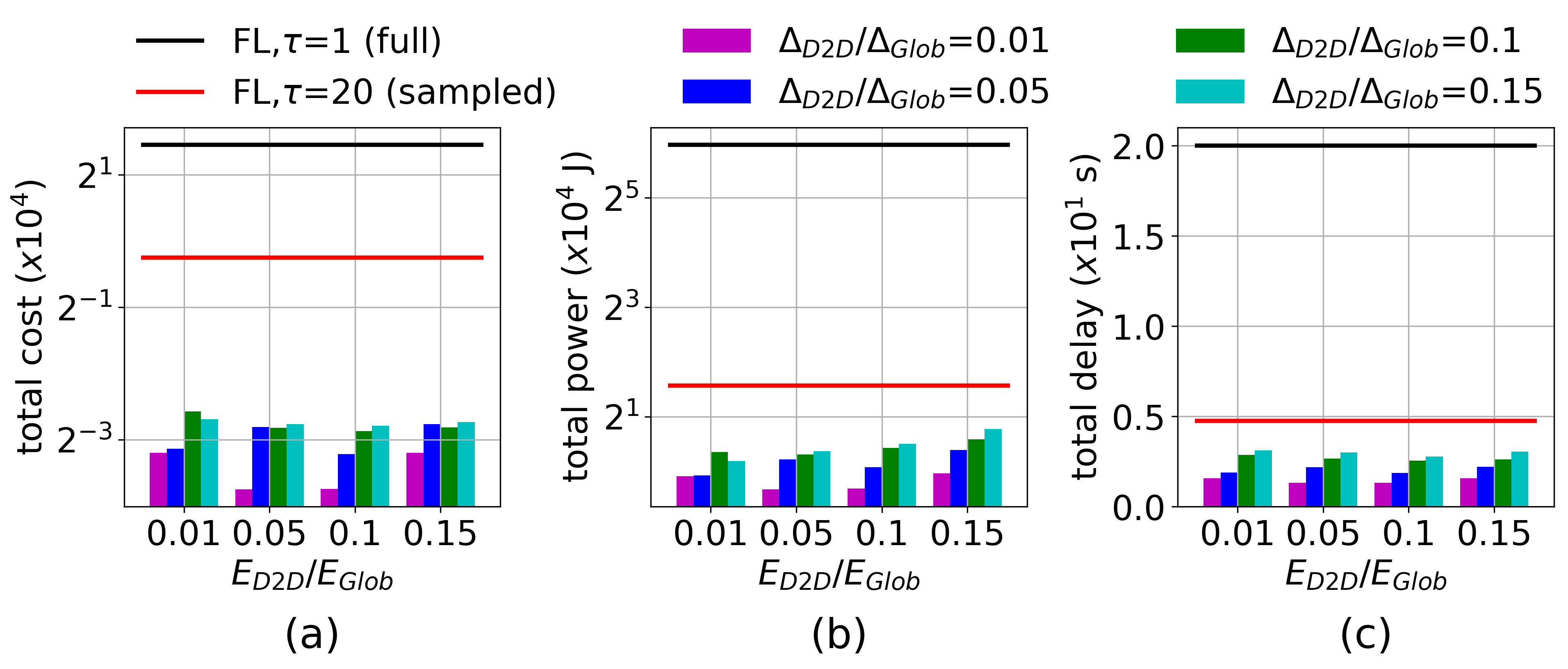}
\centering
\caption{Comparing total (a) cost, (b) power, and (c) delay metrics from the optimization objective in $(\bm{\mathcal{P}})$ achieved by {\tt TT-HF} versus baselines upon reaching $75\%$ of peak accuracy, for different configurations of delay and energy consumption. {\tt TT-HF} obtains a significantly lower total cost in (a). (b) and (c) demonstrate the region under which {\tt TT-HF} attains energy savings and delay gains. (FMNIST, Neural Network)}
\label{fig:fmnist_nn_resource_bar_0}
\vspace{-5mm}
\end{figure}

\begin{figure}[t]
\includegraphics[width=0.99\columnwidth]{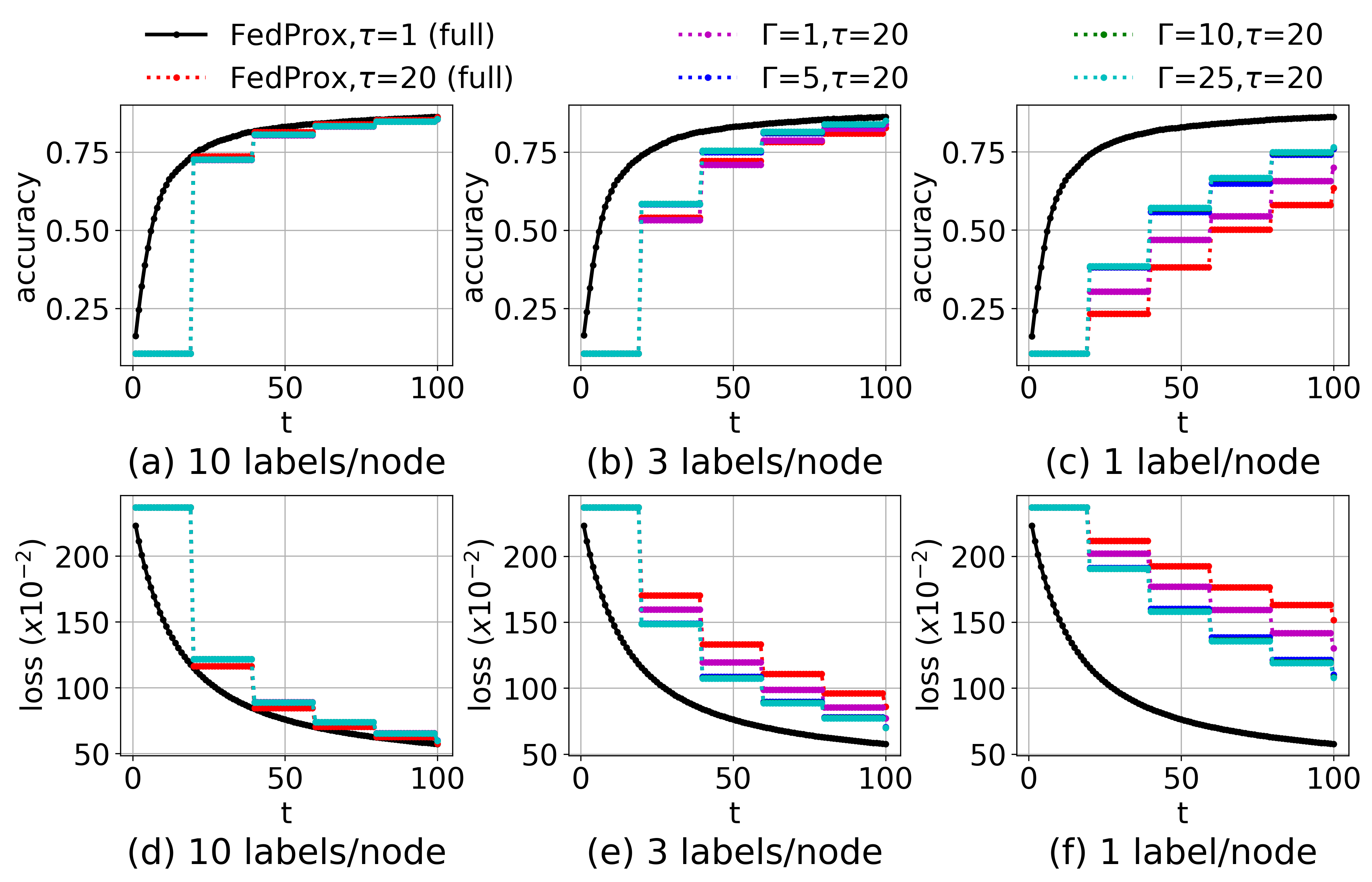}
\centering
\caption{Performance comparison between {\tt TT-HF} and FedProx~\cite{sahu2018convergence} when varying the number of D2D consensus rounds ($\Gamma$). Under the same period of local model training ($\tau$), increasing $\Gamma$ results in a considerable improvement in the model accuracy/loss over time as compared to the baseline when data is non-i.i.d. (MNIST, Neural Network)}
\label{fig:prox1_app}
\vspace{-5mm}
\end{figure}

\begin{figure}[t]
\includegraphics[width=0.99\columnwidth]{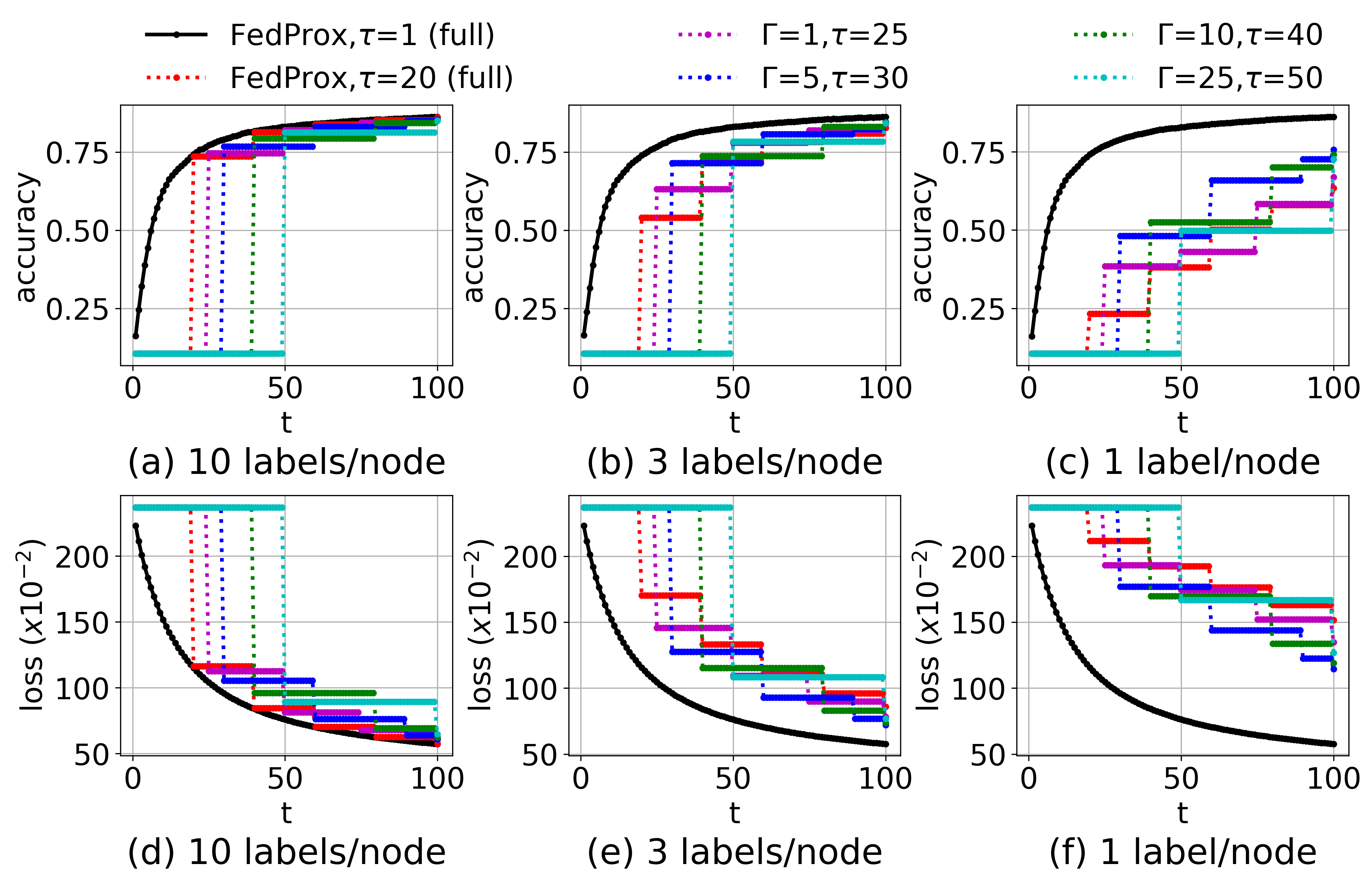}
\centering
\caption{Performance comparison between {\tt TT-HF} and FedProx~\cite{sahu2018convergence} when varying the local model training interval ($\tau$) and the number of D2D consensus rounds ($\Gamma$). With a larger $\tau$, {\tt TT-HF} can still outperform the baseline method if $\Gamma$ is increased, i.e., local D2D communications can be used to offset the frequency of global aggregations. (MNIST, Neural Network)}
\label{fig:prox2_app}
\vspace{-5mm}
\end{figure}

\endgroup

\end{document}